\documentclass[12pt]{article}
\usepackage[utf8]{inputenc}
\usepackage[round,compress]{natbib}
\usepackage{soul}
\usepackage[nodisplayskipstretch]{setspace}

\usepackage{mathtools, amsthm}
\usepackage{enumerate}
\usepackage{array}
\usepackage{times}
\usepackage{amsmath}
\usepackage{latexsym}
\usepackage{amssymb}
\usepackage{graphics}
\usepackage{graphicx}
\usepackage{stmaryrd}
\usepackage{accents}
\usepackage{color}
\usepackage{mathrsfs}
\usepackage[normalem]{ulem}
\usepackage{url}
\usepackage{wrapfig}
\usepackage[font=scriptsize,labelfont=bf]{caption}
\usepackage{multirow}
\usepackage{algorithm} 
\usepackage{algpseudocode} 
\usepackage{natbib}
\usepackage{subfig}
\usepackage{setspace} 
\usepackage{xr-hyper}
\usepackage{booktabs}
\usepackage{bm}
\PassOptionsToPackage{unicode}{hyperref}
\PassOptionsToPackage{hyphens}{url}
\usepackage{amsbsy}

\usepackage{stmaryrd}
\SetSymbolFont{stmry}{bold}{U}{stmry}{m}{n}
\usepackage{bookmark}

\IfFileExists{xurl.sty}{\usepackage{xurl}}{} 
\hypersetup{
  pdftitle={Title},
  pdfauthor={Author 1; Author 2},
  pdfkeywords={3 to 6 keywords, that do not appear in the title},
  colorlinks=true,
  linkcolor={blue},
  filecolor={purple},
  citecolor={blue},
  urlcolor={blue},
  pdfcreator={LaTeX via pandoc}}

\usepackage[mathscr]{eucal}

\definecolor{red}{rgb}{1,0,0}
\definecolor{green}{rgb}{0,1,0}
\definecolor{blue}{rgb}{0,0,1}

\definecolor{purple}{RGB}{180,0,180}

\newcommand{\Real}{\mathbb{R}} 
\newcommand{\E}{\mathbb{E}} 
\newcommand{\Tra}{^{\sf T}} 
\def\vec{\mathop{\rm vec}\nolimits}
\newcommand{\tr}{\mathop{\rm trace}\nolimits}
\newcommand{\Var}{\mathrm{Var}}

\newcommand{\V}[1]{{\bm{\mathbf{\MakeLowercase{#1}}}}} 

\newcommand{\M}[1]{{\bm{\mathbf{\MakeUppercase{#1}}}}} 

\newcommand{\T}[1]{\boldsymbol{\mathscr{\MakeUppercase{#1}}}} 

\newcommand{\Kron}{\otimes} 

\usepackage{comment}
\usepackage{xspace}
\definecolor{blue}{rgb}{0.2,0.5,0.7}
\definecolor{green}{rgb}{0.3,0.68,0.29}
\definecolor{purple}{rgb}{0.6,0.31,0.64}

\definecolor{sy}{RGB}{229,114,0}

\newtheorem{theorem}{Theorem}[section]
\newtheorem{lemma}{Lemma}[section]
\newtheorem{proposition}{Proposition}[section]
\newtheorem{corollary}{Corollary}[section]

\theoremstyle{remark}
\newtheorem{remark}{Remark}[section]

\theoremstyle{definition} 
\newtheorem{assumption}{Assumption}[section]

\newcommand{\blind}{1}

 \makeatletter
\newcommand*{\addFileDependency}[1]{
  \typeout{(#1)}
  \@addtofilelist{#1}
  \IfFileExists{#1}{}{\typeout{No file #1.}}
}
\makeatother

\begin{document}
\def\spacingset#1{\renewcommand{\baselinestretch}%
{#1}\small\normalsize} \spacingset{1}

\if1\blind
{ 
  \title{\bf Asymptotic Optimism for Tensor Regression Models with Applications to Neural Network Compression}
  \author{Haoming Shi\\
    Department of Statistics, Rice University\\
    Eric C. Chi \\
    School of Statistics, University of Minnesota \\
    Hengrui Luo \\
    Department of Statistics, Rice University}
  \date{}
  \maketitle
} \fi

\if0\blind
{
\title{\bf Asymptotic Optimism for Tensor Regression Models with Applications to Neural Network Compression}
  \bigskip
  \bigskip
  \bigskip
  \medskip
  \date{}
  \maketitle
} \fi

\bigskip
\begin{abstract}
We study rank selection for low-rank tensor regression under random covariates design. Under a Gaussian random-design model and some mild conditions, we derive population expressions for the expected training--testing discrepancy (optimism) for both CP and Tucker decomposition. We further demonstrate that the optimism is minimized at the true tensor rank for both CP and Tucker regression. This yields a prediction-oriented rank-selection rule that aligns with cross-validation and extends naturally to tensor-model averaging. We also discuss conditions under which under- or over-ranked models may appear preferable, thereby clarifying the scope of the method. Finally, we showcase its practical utility on a real-world image regression task and extend its application to tensor-based compression of neural network, highlighting its potential for model selection in deep learning.

\end{abstract}

\noindent%
\vfill

\newpage

\spacingset{1}

\section{Introduction}
In the realm of statistical learning, regression analysis fundamentally seeks to construct models capable of accurate prediction on unseen data. This process critically depends on model selection, a procedure of choosing from a collection of candidates that optimally balances its goodness-of-fit with its complexity. Classical methods for model selection, such as Mallows's $C_p$ \citep{mallows2000some}, AIC \citep{akaike1998information}, and BIC \citep{schwarz1978bic}, approach this task by penalizing the in-sample training error with a measure of model complexity. However, this traditional framework, by focusing on in-sample error, can inadequately capture a model's predictive power, and is challenged in high-dimensional settings where models often operate in an over-parameterized regime. 
 
Recent research has revisited the concept of optimism \citep{efron1986biased, efron_estimation_2004} to better differentiate the predictive performance of models.
Optimism is the difference between the testing and training errors. It quantifies how much the in-sample error underestimates the prediction error, providing a more accurate measure for a model's generalizability. While this optimism framework effectively generalizes the traditional model-selection criteria mentioned above, its capacity to fully capture a model's predictive powers is limited by an underlying assumption that the covariate values in the test data are deterministic and identical to those in the training data, a scenario referred to as the ``Fixed-$\M{X}$" setting \citep{rosset2019fixed}. This ``Fixed-$\M{X}$" assumption, while reasonable in controlled experimental designs, is often unrealistic in predictive applications where new, unseen feature values are the norm. \citet{rosset2019fixed} argues that it is more appropriate to assess the model's predictive performance in the``Random-$\M{X}$" setting, where the covariates $\M{X}$ are random and the test features are different (independently regenerated) from those in the training set. 

Following this paradigm, \citet{hastie2022surprises} studied the out-of-sample prediction risk for the minimum $\ell_2$-norm interpolator in the high-dimensional least squares regime. More directly, \citet{luan2021predictive} designed a new predictive model degrees of freedom based on optimism for linear regression under the ``Random-$\M{X}$" setting. \citet{luo2025optimism} further extended this line of work by deriving a closed-form expression for the expected optimism in random design linear regression and generalizing the result to kernel ridge regression (KRR). Collectively, these works establish a rigorous foundation for optimism as a unifying, data-dependent complexity measure for linear smoothers under the ``Random-$\M{X}$" framework.

Parallel to the developments in understanding optimism for vector linear models, the regression analysis itself has evolved to address covariates with far richer structures. While classical regression assumes vector-valued covariates (i.e., $\V{x} \in \Real^p$), modern applications in image analysis \citep{zhou2013tensor,luo2024frontal}, spatiotemporal denoising  \citep{bahadori2014fast,luo2026wedgesamplingefficienttensor}, and multi-task learning \citep{yu2018tensor} frequently involve data that are naturally represented as a multi-dimensional array, also called tensor ($\T{X} \in \Real^{I_1 \times I_2 \times \cdots \times I_M}$). Tensor regression (i.e., $y = \langle\!\langle \T{X}, \T{B}\rangle\!\rangle + \epsilon$) therefore offers a powerful advantage over traditional methods by analyzing these tensor covariates in their native form, thereby preserving latent structural information, such as spatial correlations across tensor modes, which is lost through vectorization. Given this advantage, tensor covariate regression has gained significant attention, leading to numerous extensions from the   scalar-on-tensor setting \citep{zhou2013tensor, li2018tucker} to more complex models like tensor-on-tensor regression \citep{lock2018tensor}, tensor Gaussian Process (TensorGP \cite{yu2018tensor}), and efficient tensor-based decision trees \citep{luo2024efficient,luo2026triangulated}. In this work, we focus on the   scalar-on-tensor regression problem, aiming to bridge the two research streams discussed by extending the rigorous optimism analysis from the Random-X framework to the high-dimensional, structured setting of tensor regression.

A central challenge in tensor regression is the high dimensionality of the coefficient tensor $\T{B}$,  which can lead to a prohibitive number of parameters and computational burden. To address this, a common and effective strategy is to impose a low-rank structure on $\T{B}$ through tensor decompositions, such as the Canonical/Polyadic (CP) and Tucker decompositions \citep{kolda2009tensor}. Unlike the vector covariate case, tensor low-rank constraint naturally introduces a more complex model selection problem, as the definition of low-rankness is not unique \citep{kolda2009tensor}, and one needs to determine an appropriate rank to balance model fit and computational complexity. 
Traditional approaches \citep{zhou2013tensor,li2018tucker,si2022efficient,luo2024efficient, bu2025improving}, often rely on information criteria, where the number of effective parameters is typically defined as a sum of parameters in each low-rank component. 
However, such definition of complexity may not fully capture the nuanced interactions inherent in the tensor rank structure. In the ``Fixed-$\M{X}$", as \citet{shao1993linear,shao1997asymptotic}'s asymptotic analysis for linear models clarifies, classical criteria like AIC is loss-efficient only in pure approximation regimes and otherwise over-selects; BIC can be directionally consistent when a finite true rank exists, but its ``Fixed-$\M{X}$", parameter-count penalty fundamentally mismeasures the ``Random-$\M{X}$'' predictive complexity induced by the tensor's multilinear structure, where the low-rankness has non-unique definition in tensor covariate regressions. These limitations motivate us to study 
the optimism of tensor regression under a ``Random-$\T{X}$" design \footnote{we replace $\M{X}$ by $\T{X}$ to reflect tensor covariates}, aiming to theoretically establish a direct relationship between the coefficient tensor's rank and the model's predictive capability.

The rest of this paper is organized as follows. In Section~\ref{sec:preliminaries}, we provide important background on optimism theory and low-rank tensor regression, describing how we approach our optimism analysis via a kernel regression perspective \citep{yu2018tensor}. We then present our main results for low-rank CP regression in Section~\ref{sec:CP_regression} and generalize them to the Tucker case in Section~\ref{sec:Tucker_regression}. Section~\ref{sec:CP_ensemble} studies the optimism behavior in the context of ensemble CP regression. Numerical experiments and real-data studies are provided in Section~\ref{sec:simulation} and Section~\ref{sec:Real_Data}, respectively. We conclude our work in Section~\ref{sec:conclusion}. Proofs and additional experimental results are deferred to the Supplementary Materials.

\section{Preliminaries}
\label{sec:preliminaries}

\subsection{Notation}
We begin by defining notations used throughout the paper. A boldface Euler script letter ($\T{X} \in \Real^{I_1 \times \cdots I_M}$) denotes a tensor of mode $M\geq3$. A boldface uppercase letter denotes a matrix ($\M{X} \in \Real^{I_1 \times I_2}$), a boldface lowercase letter denotes a vector ($\V{x} \in \Real^n$), and a plain letter denotes a scalar ($x\in \Real$).
The inner product between two vectors $\V{x},\V{y}\in\Real^{n}$ is denoted by $\langle\V{x}, \V{y} \rangle = \sum_{i=1}^{n}x_{i}y_{i}$. For two $M$-mode tensors $\T{X},\T{Y}\in\Real^{I_{1}\times\cdots\times I_{M}}$, their inner product is denoted by:
\begin{equation*}
\langle\!\langle\T{X},\T{Y}\rangle\!\rangle=\sum_{i_{1}=1}^{I_{1}}\sum_{i_{2}=1}^{I_{2}}\cdots\sum_{i_{M}=1}^{I_{M}}x_{i_{1}i_{2}\cdots i_{M}}y_{i_{1}i_{2}\cdots i_{M}}=\langle\vec(\T{X}),\vec(\T{Y})\rangle
\end{equation*}
where $x_{i_{1}i_{2}\cdots i_{M}}$ is the $(i_{1},i_{2},\cdots,i_{M})$-th
element of $\T{X}$ and $\vec(\T{X})\in\Real^{\prod_{m}I_{m}\times1}$
is the vectorization of $\T{X}$. The outer product of $M$ vectors $\V{x}_1, \dots, \V{x}_M$ of size $I_1, \dots, I_M$ is an $M$-mode tensor $\V{x}_1 \circ \cdots \circ\V{x}_M \in \Real^{I_1 \times \cdots \times I_M}$. Given a tensor $\T{X} \in \Real^{I_1 \times \cdots \times I_M}$, $\T{X}_{(m)}$ is its m-mode matricization that maps $\T{X}$ to a $I_m \times \prod_{m' \neq m} I_{m'}$ matrix \citep{kolda2009tensor}. The Frobenius norm of a tensor $\T{X}$ is defined as $\lVert\T{X}\rVert = \sqrt{
\langle\!\langle \T{X}, \T{X}
\rangle\!\rangle }$, which is analogous to the matrix Frobenius norm $\lVert\M{X}\rVert$. And we use $\lVert\V{x}\rVert_2$ and $\lVert\M{X}\rVert_2$ to denote the $\ell_2$ norm of a vector and operator norm of a matrix, respectively. Finally, given two matrices $\M{A}\in\Real^{p\times m}$ and $\M{B}\in\Real^{q\times n}$, their Kronecker product is 
$\M{A}\Kron\M{B}=[\V{a}_{1}\Kron\M{B},\V{a}_{2}\Kron\M{B},\dots,\V{a}_{n}\Kron\M{B}]\in\Real^{pq\times mn}.
$
If they also have the same number of columns (m=n), their Khatri-Rao product (the column-wise Kronecker product) is denoted by:
$\M{A}\odot\M{B}=[\V{a}_{1}\Kron\mathbf{b}_{1},\V{a}_{2}\Kron\mathbf{b}_{2},\dots,\V{a}_{n}\Kron\mathbf{b}_{n}]\in\Real^{pq\times n}.
$

\subsection{Optimism Theory}
We now review some recent developments in the optimism theory pioneered by \citet{ye1998measuring} and \citet{efron_estimation_2004}.  Consider a regression model with training data $(\V{x}_i, y_i)\ \in \Real^d \times \Real$ for $i = 1,\dots, n$ given by 
$
    y_i = f(\V{x}_i) + \epsilon_i, 
$
where $f(\V{x}) = \E(y|\V{x})$ and $\epsilon_i$ are i.i.d. random noise terms independent of $\V{x}_i$ with $\E(\epsilon_i) = 0$ and $\Var(\epsilon_i) = \sigma_\epsilon$. Denote $\M{X} = (\V{x}_1, \dots, \V{x}_n)\Tra$, $\V{y} = (y_1, \dots, y_n)\Tra$, and  $\hat{\V{f}} = (\hat{f}(\V{x}_1), \dots, \hat{f}(\V{x}_n))\Tra$ as the vector of fitted values for $\V{y}$, the training error is then defined and minimized as
\begin{equation}
    \text{ErrT}_\M{X} = \frac{1}{n} \E_\V{y} \lVert\V{y} - \hat{\V{f}}  \rVert_2^2
    \label{eq:train_mse}
\end{equation}
where the subscript on $\E$ denotes the random variable over which the expectation is taken. In the classical ``Fixed-$\M{X}$" setting, $\M{X}$ is non-random and identical in training and testing data. Let $\tilde{\V{y}}$ be an independent copy of $\V{y}$ given $\M{X}$, the in-sample prediction error is given as:
\begin{equation}
    \text{ErrF}_\M{X} = \frac{1}{n} \E_{\V{y}, \tilde{\V{y}}} \lVert \tilde{\V{y}} - \hat{\V{f}} \rVert_2^2
    \label{eq:test_mse_fix}
\end{equation}
\citet{efron_estimation_2004} defines \uline{\emph{fixed-design optimism}} as the difference between these errors and shows that it can be expressed as the sum of covariance between $y_i$ and $\hat{f}(\V{x}_i)$:
\begin{equation}
    \mathrm{OptF}_\M{X} = \eqref{eq:test_mse_fix} - \eqref{eq:train_mse} = \frac{2}{n} \sum_{i=1}^n \mathrm{Cov}\left(y_i, \hat{f}(\V{x}_i)\right)
    \label{eq:optimism_fixed_design}
\end{equation}
For ordinary least squares where $\hat{\V{f}} = \M{H}\V{y} = \M{X}\Tra (\M{X}\Tra \M{X})^{-1} \M{X}\Tra \V{y}$, one can show that \eqref{eq:optimism_fixed_design} will become $\mathrm{OptF}_\M{X} = 2\sigma_\epsilon^2 \tr(\M{H})/n = 2\sigma_\epsilon^2 d/n$, which coincides with the classical degrees of freedom for linear regression models \citep{tibshirani1987local, hastie2009elements}.

Subsequent work~\citep{tibshirani2009bias,rosset2019fixed, tibshirani2019excess} shifts the focus to the ``Random-$\M{X}$" setting. Consider $\M{X}$ now random and let $(\V{x}_*, \epsilon_*)$ be an independent copy of $(\V{x}_i, \epsilon_i)$ and $y_* = f(\V{x}_*) + \epsilon_*$. The out-of-sample prediction error is defined as:
\begin{equation}
    \text{ErrR}_\M{X} = \E_{\V{y}, y_*, \M{X}, \V{x}_*} (y_* - \hat{f}(\V{x}_*))^2
    \label{eq:test_mse_random}
\end{equation}
where the expectation is over all sources of randomness. Under a linear smoother (i.e., $\hat{\V{f}} = \M{H}\V{y}$ and $\hat{f}(\V{x}_*) = \V{h}_*\Tra \V{y}$), \citet{luan2021predictive} shows that the \uline{\emph{expected optimism}} of \eqref{eq:optimism_fixed_design} in this random design is given as:
\begin{equation}
    \mathrm{OptR}_{\M{X}} = \eqref{eq:test_mse_random} - \eqref{eq:train_mse} = \Delta B_{\M{X}} + 
    \frac{2}{n} \sigma_\epsilon^2 \left[\tr(\M{H}) + \frac{n}{2} \left( \E_{\V{x}_*} \lVert \V{h}_* \rVert_2^2 - \frac{1}{n} \tr(\M{H}\Tra \M{H}) \right) \right]
    \label{eq:optimism_random_design}
\end{equation}
where $\Delta B_{\M{X}} = \E_{\V{x}_*}(f(\V{x}_*) - \V{h}_*\Tra \V{f} )^2 - \lVert \V{f} - 
\M{H}\V{f} \rVert_2^2/n$ is the excess bias from out-of-sample prediction. Building on this, \citet{luo2025optimism} considers \eqref{eq:optimism_random_design} in the least-square setting with Gaussian noise and derives a closed-form of the expected optimism (their Theorem 3). More importantly, they generalize this analysis to KRR setting and show that (their Theorem 12)
\begin{equation}
    \mathrm{OptR}_{\M{X}}\! = \frac{2}{n}\E_{\M{X}}\left[\left\Vert \left(\M{\Sigma}_{\phi}^{1/2}\M{\Sigma}_{\phi,\lambda}^{-1}\right)\!\left[\phi(\V{x}_{*})y_{*}-\left(\phi(\V{x}_{*})\phi(\V{x}_{*})\Tra+\lambda\M{I}\right)\M{\Sigma}_{\phi,\lambda}^{-1}\V{\eta}_{\phi}\right]\right\Vert _{2}^{2}\right] + \mathcal{O}_{p}\left(\frac{1}{n^{3/2}}\right)
    \label{eq:optimism_KRR_random_design}
\end{equation}
where $\V{\eta}_{\phi} = \E_{\V{x}_*}(\phi(\V{x}_{*})y_{*})$, $\M{\Sigma}_{\phi} = \E_{\V{x}_*}(\phi(\V{x}_{*})\phi(\V{x}_{*})\Tra+\lambda\M{I})$, and $\phi$ is the feature map associated with the kernel function. Equation \eqref{eq:optimism_KRR_random_design} demonstrates that any model admitting a representer theorem with an explicit feature map inherits an analytically tractable complexity penalty. This result provides the foundation for our work. As we will see next, a low-rank tensor regression model can be viewed as a kernel regression model with a specific multi-linear kernel \citep{yu2018tensor}. This equivalence enables us to leverage \eqref{eq:optimism_KRR_random_design} to conduct our optimism analysis for tensor regression.

\subsection{Tensor Regression}
\label{sec:pre_tensor_regrssion}
Tensor regression extends the standard linear model to accommodate tensor-structured covariates. Given a scalar response $y_i \in \Real$ and a tensor covariate $\T{X}_i \in \Real^{I_1 \times \cdots \times I_M}$ for $i = 1, \dots, n$, the tensor model is formulated as:
\begin{equation}
    y_i = \langle\!\langle\T{X}_i,\T{B}\rangle\!\rangle + \epsilon_i
    \label{eq:tensor_regression_model}
\end{equation}
where $\T{B} \in \Real^{I_1 \times \cdots \times I_M}$ is the tensor coefficient and $\epsilon_i$ are i.i.d.\@ additive mean-zero Gaussian noises. This formulation is a direct analog of vector linear regression, with the tensor inner product $\langle\!\langle \cdot,\cdot\rangle\!\rangle$ replacing the vector dot product $\langle \cdot, \cdot \rangle$. Notice that \eqref{eq:tensor_regression_model} is equivalent to a linear regression on $\vec(\T{X}_i)$ since $\langle\!\langle\T{X}_i,\T{B}\rangle\!\rangle = \langle\vec(\T{X}_i),\vec(\T{B})\rangle$. A predominant challenge in fitting \eqref{eq:tensor_regression_model} is the ultrahigh dimensionality of $\T{B}$, where the number of parameters (i.e., $D = \prod_m I_m$) can be computationally prohibitive. To resolve this issue, a common solution is to impose a low-rank structure on $\T{B}$ in \eqref{eq:tensor_regression_model} via tensor decomposition.

A popular approach is the CP decomposition \citep{guo2011tensor, zhou2013tensor, lock2018tensor}, which represents $\T{B}$ as a sum of rank-one tensors:
\begin{equation}
    \T{B}=\sum_{r=1}^{R}\V{\beta}_{1}^{(r)}\circ\cdots\circ\V{\beta}_{M}^{(r)} = \llbracket\M{B}_{1},\dots,\M{B}_{M}\rrbracket \label{eq:CP_decomposition}
\end{equation}
where $R$ is the CP rank of $\T{B}$ and $\M{B}_{m} =[\V{\beta}_{m}^{(1)},\dots,\V{\beta}_{m}^{(R)}]\in\Real^{I_{m}\times R}$ with $\V{\beta}_{m}^{(r)}\in\Real^{I_{m}}$ \citep{kolda2009tensor}. Then the low-rank CP regression model from \eqref{eq:tensor_regression_model} can be defined as:
\begin{equation}
    y_i = \langle\!\langle\T{X}_i,\T{B}\rangle\!\rangle + \epsilon_i \quad \text{s.t. } \T{B}=\sum_{r=1}^{R}\V{\beta}_{1}^{(r)}\circ\cdots\circ\V{\beta}_{M}^{(r)} \label{eq:tensor_regression_model_CP}
\end{equation}
which reduces the number of parameters from $\prod_m I_m$ to $R \sum_m I_m$, substantially lowering the computational cost.

An alternative and more general approach is the Tucker decomposition \citep{li2018tucker, zhang2022tucker}, which is a higher-order analog of the matrix singular value decomposition (SVD) \citep{kolda2009tensor}:
\begin{equation}
\T{B} = \T{G}\times_{1}\M{U}_{1}\times_{2}\M{U}_{2}\;\cdots\times_{M}\M{U}_{M}
\label{eq:tucker_decomposition}
\end{equation}
where $\T{G} \in \Real^{R_1\times R_2\times \dots\times R_M}$ is the core tensor, $\M{U}_m \in \Real^{I_m \times R_m}$ are the factor matrices with full column rank, and $(R_1, \dots, R_M)$ are the Tucker ranks of $\T{B}$. 
Similar to model \eqref{eq:tensor_regression_model_CP}, the low-rank Tucker regression model based on \eqref{eq:tensor_regression_model} is given as:
\begin{equation}
    y_i = \langle\!\langle\T{X}_i,\T{B}\rangle\!\rangle + \epsilon_i \quad \text{s.t. } \T{B} = \T{G}\times_{1}\M{U}_{1}\times_{2}\M{U}_{2}\;\cdots\times_{M}\M{U}_{M} \label{eq:tensor_regression_model_tucker}
\end{equation}
where the dimensionality now reduces to $\prod_m R_m + \sum_m I_m R_m$, 
offering more modeling flexibility by allowing different ranks for each mode.

Complementing this line of the research, \citet{yu2018tensor} approaches the problem from a different direction by studying it under a kernel regression setting. Crucially, they demonstrate that fitting a low-rank tensor regression (i.e., \eqref{eq:tensor_regression_model_CP} and \eqref{eq:tensor_regression_model_tucker}) is equivalent to performing kernel-based learning with a specific multi-linear kernel structure $K(\T{X}_{i},\T{X}_{j})$. For a Tucker model \eqref{eq:tensor_regression_model_tucker}:
\begin{equation}
K(\T{X}_{i},\T{X}_{j}) =\vec(\T{X}_{i})\Tra(\Kron_{m=1}^{M}\M{U}_{M+1-m})(\Kron_{m=1}^{M}\M{U}_{M+1-m})\Tra\vec(\T{X}_{j})\in\Real 
\label{eq:Tucker_KRR_Kernel}
\end{equation}
where $\M{U}_m$ are the parameters aiming to learn the factor matrices in \eqref{eq:tucker_decomposition}. And since CP decomposition is a special case of Tucker, it becomes:
\begin{equation}
K(\T{X}_{i},\T{X}_{j}) =\vec(\T{X}_{i})\Tra(\odot_{m=1}^{M}\M{B}_{M+1-m})(\odot_{m=1}^{M}\M{B}_{M+1-m})\Tra\vec(\T{X}_{j})\in\Real
\label{eq:CP_KRR_Kernel}
\end{equation}
for a CP model \eqref{eq:tensor_regression_model_CP} where $\M{B}_m$ are the learning parameters for the matrices in \eqref{eq:CP_decomposition}. This equivalence is foundational, as it enables us to analyze the optimism of low-rank tensor regression by applying the general result for KRR in \eqref{eq:optimism_KRR_random_design}.

\section{CP Regression}
\label{sec:CP_regression}
This section presents our main theoretical results concerning the expected optimism of low-rank CP regression under the ``Random-$\T{X}$" design. Consider a low-rank CP regression model \eqref{eq:tensor_regression_model_CP} with a target rank  $R_t$. The feature map $\phi(\T{X}_i): \Real^{I_1 \times \cdots \times I_M} \rightarrow \Real^{R_t}$ from the kernel in \eqref{eq:CP_KRR_Kernel} is given by
\begin{equation}
    \phi(\T{X}_{i})= (\odot_{m=1}^{M}\M{B}_{M+1-m})\Tra\vec(\T{X}_{j}) = \left(\phi_1(\T{X}_{i}), \dots, \phi_{R_t}(\T{X}_{i})\right)\Tra,
    \label{eq:CP_feature_map}
\end{equation}
where each component $\phi_{r}(\T{X}_{i})=\langle\V{\beta}_{M}^{(r)}\Kron\cdots\Kron\V{\beta}_{1}^{(r)},\vec(\T{X}_{i})\rangle\in\Real$ for $r = 1, \dots, R_t$. Given $n$ training samples, the corresponding feature matrix $\M{\Phi} \in \Real^{n \times R_t}$ is constructed as:
\begin{equation}
    \M{\Phi} = \left(\phi(\T{X}_{1}),\phi(\T{X}_{2}),\dots,\phi(\T{X}_{n})\right)\Tra.
    \label{eq:CP_feature_matrix}
\end{equation}
The vectorized rank-one components that define this feature map \eqref{eq:CP_feature_map} of vectorized CP components are linearly independent. Hence the feature space is non-degenerate, as demonstrated in the following lemmas.
\begin{lemma}
(Linear Independence of Vectorized CP Components) \label{lemma:lemma31}
Let $\T{B}\in\Real^{I_{1} \times \cdots \times I_{M}}$ be a tensor with a rank-$R$ CP decomposition \eqref{eq:CP_decomposition} and $\{\V{v}_{r}\}_{r=1}^{R}$
denote its vectorized rank-1 components with $\V{v}_{r} = \V{\beta}_{M}^{(r)}\Kron\cdots\Kron\V{\beta}_{1}^{(r)}$. Then $\{\mathbf{v}_{r}\}_{r=1}^{R}$ are linearly independent. 
\end{lemma}

\begin{lemma}
(Generic nondegeneracy of CP features) \label{lemma:lemma32}
Let $\{ \T{X}_i \}_{i=1}^n \in \Real^{I_1 \times \cdots \times I_M}$ be i.i.d.\ tensor covariates with $\vec(\T{X}_i) \sim \mathrm{N}(0, \M{I}_{\prod_m I_m})$. Suppose that the tensor coefficient $\T{B}$ admits a rank-R decomposition in \eqref{eq:CP_decomposition} with feature map $\phi(\T{X})$ in \eqref{eq:CP_feature_map}. Let $\M{\Sigma}_\phi = \E_{\T{X}}\bigl[ \phi(\mathcal{X}) \phi(\mathcal{X})\Tra \bigr]$ be the population feature covariance. Then $\M{\Sigma}_\phi$ is positive definite.

\end{lemma}
\(
\)
\(
\)
\(
\)

These lemmas ensure that the feature space is well-defined and that the population covariance matrix of the features is positive definite, a key condition for our subsequent analysis.

Our analysis then proceeds in two stages. We first examine an oracle case where, for a given target rank $R_t$, kernel \eqref{eq:CP_KRR_Kernel} is assumed to be constructed from the true and known CP components of the underlying coefficient tensor $\T{B}$. This idealized setting enables us to directly investigate how the expected optimism varies with the CP rank. We then extend this analysis to the more general case by accounting for estimation errors.
Our results are based on the following assumptions, which we adapt from assumptions in \citet{luo2025optimism} to the tensor KRR setting, along with some additional assumptions on $\T{X}$.

\begin{assumption}
\label{Assumption:assumption_cp}
Let $\M{\Phi}$ be the feature matrix from \eqref{eq:CP_feature_matrix}. Define the empirical quantities $\hat{\V{\eta}}_{\phi}=\frac{1}{n}\M{\Phi}\Tra\V{y}\in\Real^{R_t}$ and $\hat{\M{\Sigma}}_{\phi,\lambda}=\frac{1}{n}(\M{\Phi}\Tra\M{\Phi}+\lambda\M{I}_{R_t})\in\Real^{R_t\times R_t}$ for a fixed positive $\lambda$. We assume that
\begin{equation*}
    \lVert\hat{\V{\eta}}_{\phi}-\V{\eta}_{\phi}\rVert_{2}=\mathcal{O}_{p}(\frac{1}{\sqrt{n}}),\quad\lVert\hat{\M{\Sigma}}_{\phi,\lambda}-\M{\Sigma}_{\phi,\lambda}\rVert_{2}=\mathcal{O}_{p}(\frac{1}{\sqrt{n}})
\end{equation*}
where $\V{\eta}_{\phi}=\E_{\T{X}_{*}}[\phi(\T{X}_{*})y_*]$ and
$\M{\Sigma}_{\phi,\lambda}=\M{\Sigma}_{\phi} + \lambda\M{I}_{R_t} = \E_{\T{X}_{*}}[\phi(\T{X}_{*})\phi(\T{X}_{*})\Tra]+\lambda\M{I}_{R_t}$. Let $\M{\Sigma}_{\phi}=\M{U}\M{\Lambda}\M{U}\Tra$ be the eigendecomposition of $\M{\Sigma}_{\phi}$ with eigenvalues $\M{\Lambda}=\mathrm{diag}(v_{1},\dots,v_{R_t})$
where $v_{1}\geq v_{2}\geq\cdots\geq v_{R_t}>0$ (by Lemma \ref{lemma:lemma31} and \ref{lemma:lemma32}). Then
\begin{equation*}
\mathbf{\Sigma}_{\phi,\lambda}=\sum_{i=1}^{R_t}(v_{i}+\lambda)u_{i}u_{i}\Tra,\quad\mathbf{\Sigma}_{\phi,\lambda}^{-1}=\sum_{i=1}^{R_t}\frac{1}{v_{i}+\lambda}u_{i}u_{i}\Tra
\end{equation*}
Finally, let $\V{\epsilon} = (\epsilon_1, \dots, \epsilon_n)\Tra$ and assume that $\vec(\T{X}_{*})\sim \mathrm{N}(\mathbf{0},\M{I}_{\prod_{m}I_{m}})$ and 
$\V{\epsilon}\sim \mathrm{N}(\mathbf{0},\sigma^{2}\M{I}_{n})$ are independent.
\end{assumption}
Under these assumptions, we derive a closed-form expression for the expected optimism up to $\mathcal{O}_{p}(n^{-\frac{3}{2}})$ when the target rank correctly specifies at the true underlying rank (i.e., $R_t = R$).

\begin{theorem}
\label{thm:thm_CP_true_rank}
(Expected Optimism of CP Regression Under True Rank in ``Random-$\T{X}$" Design) Under Assumption~\ref{Assumption:assumption_cp}, the expected optimism $\mathrm{OptR}_{\T{X}}^{(\mathrm{true})}$ for the CP regression \eqref{eq:tensor_regression_model_CP} at the true rank $R_{t}=R$ is: 
\begin{equation}
\mathrm{OptR}_{\T{X}}^{(\mathrm{true})}=\frac{2\left(\sigma^{2}+\frac{\lambda^2 v_1}{(v_1+\lambda)^2}\right)}{n}\sum_{r=1}^{R}\frac{v_{r}^2}{(v_{r}+\lambda)^{2}}+\mathcal{O}_{p}(n^{-\frac{3}{2}})\label{eq:tensorkrr_true_rank_CP}
\end{equation}
\end{theorem}
This reveals a direct relationship between the expected optimism and the spectrum of the kernel \eqref{eq:CP_KRR_Kernel}, a quantity governed by the target rank $R_t$.

We now analyze the behavior when the model rank is misspecified. We first consider the over-specified rank case, where the the target CP rank $R_{t}$ is larger than the true rank $R$. Denote the rank $R_t$ CP decomposition as $\T{B}=\sum_{r=1}^{R_{t}}\tilde{\V{\beta}}_{1}^{(r)}\circ\cdots\circ\tilde{\V{\beta}}_{M}^{(r)}$. Let $\phi_{R_{t}}(\T{X}_{*})\in\Real^{R_{t}}$ and $\M{\Phi}_{R_{t}}\in\Real^{n\times R_{t}}$ be constructed the same way as \eqref{eq:CP_feature_map} and \eqref{eq:CP_feature_matrix}. And denote $\M{\Sigma}_{\phi_{R_{t}}}=\E_{\T{X}_{*}}[\phi_{R_{t}}(\T{X}_{*})\phi_{R_{t}}(\T{X}_{*})\Tra]\in\Real^{R_{t}\times R_{t}}$ and $\V{\eta}_{\phi_{R_{t}}}=\E_{\T{X}_{*}}[\phi_{R_{t}}(\T{X}_{*})y_*]\in\Real^{R_t}$. In this scenario, the expected optimism follows a similar structure as \eqref{eq:tensorkrr_true_rank_CP}:
\begin{theorem}
\label{thm:thm_CP_over_rank}
(Expected Optimism of CP Regression for Over-Specified Rank in ``Random-$\T{X}$" Design) Under Assumption~\ref{Assumption:assumption_cp}, the expected optimism $\mathrm{OptR}_{\T{X}}^{(\mathrm{over})}$ for the CP regression \eqref{eq:tensor_regression_model_CP} at a target rank $R_{t} > R$ is: 
\begin{equation}
\mathrm{OptR}_{\T{X}}^{(\mathrm{over})}=\frac{2\left(\sigma^{2}+\frac{\lambda^2\tilde{v}_1}{(\tilde{v}_1+\lambda)^2}\right)}{n}\sum_{r=1}^{R_{t}}\frac{\tilde{v}_{r}^2}{(\tilde{v}_{r}+\lambda)^{2}}+\mathcal{O}_{p}(n^{-\frac{3}{2}})\label{eq:tensorkrr_over_rank_CP}
\end{equation}
where $\{\tilde{v}_{r}\}_{r=1}^{R_{t}}$ and $\tilde{\mathbf{U}}$
are from the eigen-decomposition of $\mathbf{\Sigma}_{\phi_{R_{t}}}=\tilde{\mathbf{U}}\text{diag}(\tilde{v}_{1},\dots,\tilde{v}_{R_{t}})\tilde{\mathbf{U}}\Tra$
with $\tilde{v}_{1}\geq\tilde{v}_{2}\cdots\geq\tilde{v}_{R_{t}}\geq0$. 
\end{theorem}

Compared to \eqref{eq:tensorkrr_true_rank_CP}, the over-specified optimism \eqref{eq:tensorkrr_over_rank_CP} differs by the summation over a larger set of eigenvalues introduced by fitting a higher target rank $R_t > R$. For the under-specified case ($R_t < R$), the analysis is more nuanced, as the best rank-$R_t$ CP approximation of a rank-$R$ tensor can be ill-posed (see \cite{kolda2009tensor} section 3.3 for details). In this case, We proceed by assuming a well-defined best rank-$R_t$ approximation, denoted $\T{B}_{R_{t}} = \sum_{r=1}^{R_{t}}\tilde{\V{\beta}}_{1}^{(r)}\circ\cdots\circ\tilde{\V{\beta}}_{M}^{(r)}$, exists as the solution to:
\begin{equation}
\min_{\{\V{\beta}_{m}^{(r)}\}}\lVert\T{B}-\sum_{r=1}^{R_{t}}{\V{\beta}}_{1}^{(r)}\circ\cdots\circ{\V{\beta}}_{M}^{(r)}\rVert^{2}\label{eq:cp_low_rank_approximation_obj}
\end{equation}
Let the vectorized approximation residual be $\V{\Delta}=\vec(\T{B})-\sum_{r=1}^{R_{t}}\tilde{\V{\beta}}_{M}^{(r)}\Kron\cdots\Kron\tilde{\V{\beta}}_{1}^{(r)}\in\Real^{\prod_{m}I_{m}}$. The expected optimism for the under-specified case now includes an additional term accounting for this approximation error.

\begin{theorem}
\label{thm:thm_CP_under_rank}
(Expected Optimism of CP Regression for Under-Specified Rank in ``Random-$\T{X}$" Design) Under Assumption~\ref{Assumption:assumption_cp}, the expected optimism $\mathrm{OptR}_{\T{X}}^{(\mathrm{under})}$ for the CP regression \eqref{eq:tensor_regression_model_CP} at a target rank $R_{t} < R$ is: 
\begin{equation}
\mathrm{OptR}_{\T{X}}^{(\mathrm{under})}=\frac{2\left(\sigma^{2}+\frac{\lambda^2\tilde{v}_1}{(\tilde{v}_1+\lambda)^2}\right)}{n}\sum_{r=1}^{R_{t}}\frac{\tilde{v}_{r}^2}{(\tilde{v}_{r}+\lambda)^{2}}
+\frac{2 \lVert\V{\Delta}\rVert^{2} }{n} \sum_{r=1}^{R_{t}}\frac{\tilde{v}_{r}^2}{(\tilde{v}_{r}+\lambda)^{2}}
+\mathcal{O}_{p}(n^{-\frac{3}{2}})\label{eq:tensorkrr_under_rank_CP}
\end{equation}
where $\{\tilde{v}_{r}\}_{r=1}^{R_{t}}$ and $\tilde{\mathbf{U}}$
are from the eigen-decomposition of $\mathbf{\Sigma}_{\phi_{R_{t}}}=\tilde{\mathbf{U}}\text{diag}(\tilde{v}_{1},\dots,\tilde{v}_{R_{t}})\tilde{\mathbf{U}}\Tra$
with $\tilde{v}_{1}\geq\tilde{v}_{2}\cdots\geq\tilde{v}_{R_{t}}\geq0$. 
\end{theorem}
Unlike the previous case, the optimism \eqref{eq:tensorkrr_under_rank_CP} is inflated by an approximation error term that grows with the magnitude of the residual $\V{\Delta}$. Theorem
\ref{thm:thm_CP_over_rank} and Theorem
\ref{thm:thm_CP_under_rank} together give  our key finding regarding the optimality of the true rank; the expected optimism $\mathrm{OptR}_{\T{X}}$ is non-decreasing when $R_t \neq R$.

\begin{proposition}
\label{prop:CP_proposition}
Under Assumption~\ref{Assumption:assumption_cp}, let the true tensor coefficient admit a rank-$R$ CP decomposition and assume regularization $\lambda$ remains small compare to all the eigenvalues (i.e., $\lambda \ll v_R\,\wedge \tilde{v}_{R_t}$). Then for a target rank $R_t > R$, $\mathrm{OptR}_{\T{X}}^{(\mathrm{over})} \geq \mathrm{OptR}_{\T{X}}^{(\mathrm{true})}$; and under $R_t < R$, if $\lVert\V{\V{\Delta}}\rVert^{2} \geq \sigma^2 \frac{R - R_t}{R_t}$,
then $\mathrm{OptR}_{\T{X}}^{(\mathrm{under})} \geq \mathrm{OptR}_{\T{X}}^{(\mathrm{true})}$.
\end{proposition}

This proposition demonstrates that under some mild conditions (see Remark~\ref{remark:remark_CP_noise} for a discussion on when these assumptions hold), the expected optimism at the true CP rank is minimized given the non-decreasing properties, which showcases the validity of using optimism as a criterion for rank selection in model \eqref{eq:tensor_regression_model_CP}.

\begin{remark}
\label{remark:remark_CP_noise}
Regarding the condition on the regularization parameter $\lambda$, notice that it is an auxiliary variable originating from the underlying KRR results \eqref{eq:KRR_optimism} and can be chosen arbitrarily small. In the context of standard tensor regression (where $\lambda \to 0$), this condition is automatically satisfied. Regarding inequality $\lVert\V{\V{\Delta}}\rVert^{2} \geq \sigma^2 \frac{R - R_t}{R_t}$, violations can arise if $\sigma^2$ is too large or if $R_t$ is too small relative to the true rank $R$. But since the left-hand side quantifies the residual (approximation) error due to rank under-specification, which naturally scales with the rank deficit (i.e., $\V{\Delta} =\mathcal{O}(R-R_t)$), the condition essentially mandates a good control over the noise level $\sigma^2$. Specifically, the inequality holds in low-to-moderate noise regimes, which ensures that optimism is minimized at the true rank $R$. However, if the noise $\sigma^2$ is sufficiently large, the reduction in model complexity (variance) will dominate the approximation error, causing the expected optimism at $R_t < R$ to be smaller than at the true rank and essentially become monotonic in $R_t$. This theoretical prediction is consistent with the empirical behavior observed in the left panel of Figure~\ref{fig:CP_krr_vary_noise_and_sample} (last column).

\end{remark}

The preceding analysis assumes oracle knowledge of the true CP components. We now extend these results to the more realistic setting where components are estimated from data. Let $\hat{\T{B}}=\sum_{r=1}^{R}\hat{\V{\beta}}_{1}^{(r)}\circ\cdots\circ\hat{\V{\beta}}_{M}^{(r)}$ be an estimate of the true rank-$R$ coefficient tensor $\T{B}$, with vectorized estimation error
\begin{equation*}
\V{e} =\vec(\hat{\T{B}})-\vec(\T{B})
 =\sum_{r=1}^{R}(\hat{\V{\beta}}_{M}^{(r)}\Kron\cdots\Kron\hat{\V{\beta}}_{1}^{(r)})-({\V{\beta}}_{M}^{(r)}\Kron\cdots\Kron{\V{\beta}}_{1}^{(r)}) \in \Real^{\prod_m I_m}
\end{equation*}
Let $\hat{\phi}_{R}(\T{X}_{*})\in\Real^{R}$ and $\hat{\M{\Phi}}_{R}\in\Real^{n\times R}$ be constructed in the same way as \eqref{eq:CP_feature_map} and \eqref{eq:CP_feature_matrix} but with the CP components of $\hat{\T{B}}$. Let  $\hat{\M{\Sigma}}_{\phi_{R}}=\E[\hat{\phi}_{R}(\T{X}_{*})\hat{\phi}_{R}(\T{X}_{*})\Tra]\in\Real^{R\times R}$ denote the population
variance. We can obtain a plug-in estimate of the optimism at true rank $R$ by using  estimated CP components. The proof follows  identical arguments of the proof of Theorem \ref{thm:thm_CP_under_rank} by replacing $\V{\Delta}$ with $\V{e}$
and $\phi_{R_{t}}$ with $\hat{\phi}_{R}$. 

\begin{theorem}
\label{thm:thm_CP_true_rank_general}
(Expected Optimism of CP Regression Under True Rank with Estimated Components in ``Random-$\T{X}$" Design.)
Under Assumption~\ref{Assumption:assumption_cp}, the expected optimism $\widehat{\mathrm{OptR}}_{\T{X}}^{(\mathrm{true})}$
for the CP regression \eqref{eq:tensor_regression_model_CP} at true rank $R_{t}=R$ using the estimated components is 
\begin{equation}
\widehat{\mathrm{OptR}}_{\T{X}}^{(\mathrm{true})}=\frac{2\left(\sigma^{2}+\frac{\lambda^2 \hat{v}_1}{(\hat{v}_1+\lambda)^2}\right)}{n}\sum_{r=1}^{R}\frac{\hat{v}_{r}^2}{(\hat{v}_{r}+\lambda)^{2}}
+\frac{2 \lVert\mathbf{e}\rVert^{2} }{n} \sum_{r=1}^{R}\frac{\hat{v}_{r}^2}{(\hat{v}_{r}+\lambda)^{2}}
+\mathcal{O}_{p}(n^{-\frac{3}{2}})\label{Thm:tensorkrr_true_rank_general}
\end{equation}
where $\{\hat{v}_{r}\}_{r=1}^{R}$ and $\hat{\mathbf{U}}$ are from
the eigen-decomposition of $\hat{\mathbf{\Sigma}}_{\phi_{R}}=\hat{\mathbf{U}}\text{diag}(\hat{v}_{1},\dots,\hat{v}_{R})\hat{\mathbf{U}}\Tra$
with $\hat{v}_{1}\geq\hat{v}_{2}\cdots\geq\hat{v}_{R}\geq0$.
\end{theorem}
Theorem~\ref{thm:thm_CP_true_rank_general} tells us that, by applying analogous arguments to those for the misspecified rank cases in Theorem \ref{thm:thm_CP_over_rank}
and \ref{thm:thm_CP_under_rank}, the expected optimism is also minimized at the true rank $R$ in this general setting. When the estimate $\hat{\T{B}}$
is asymptotically normal with an order of $\mathcal{O}_{p}(n^{-\frac{1}{2}})$ \citep{zhou2013tensor}, the error term will become $\mathcal{O}_{p}(n^{-2})$ and is absorbed into $\mathcal{O}_{p}(n^{-\frac{3}{2}})$.

We conclude this section by noting that the identifiability of the individual CP components $\V{\beta}_m^{(r)}$ does not affect the validity of our results. Conditions for the identifiability of each $\V{\beta}_m^{(r)}$ are equivalent to conditions for the uniqueness of the CP decomposition \citep{kolda2009tensor}. Establishing them requires additional constraints to resolve scaling and permutation indeterminacy \citep{zhou2013tensor, lock2018tensor}. However, our optimism analysis is based on the prediction error of $\hat{{y}_i} = \langle\!\langle\T{X}_i,\hat{\T{B}}\rangle\!\rangle$. Therefore, our conclusions remain valid as long as the full coefficient tensor $\T{B}$ itself is identifiable, which is naturally satisfied for any given CP rank \citep{lock2018tensor}.

\section{Tucker Regression}
\label{sec:Tucker_regression}
We now extend our optimism analysis from the CP case to the Tucker decomposition, another commonly used low-rank structure in tensor regression \citep{li2018tucker}. For a tensor coefficient $\T{B} \in \Real^{I_1\times \dots\times I_M}$ that admits a rank $(R_1, \dots, R_M)$ Tucker decomposition \eqref{eq:tucker_decomposition}, these ranks correspond to its \textit{m-ranks} \citep{kolda2009tensor}, defined as $R_m = \mathrm{rank}_m(\T{B}) = \mathrm{colrank}(\M{B}_{(m)})$. Because the CP decomposition is a special case of the Tucker decomposition, the results derived for the CP model follow from the more general theorems presented in this section. 

For a low-rank Tucker regression model \eqref{eq:tensor_regression_model_tucker} with a target rank  $(R_{1_t}, \dots, R_{M_t})$, the feature map $\varphi(\T{X}_i): \Real^{I_1 \times \cdots \times I_M} \rightarrow \Real^{R_t}$, where  $R_t = \prod_{m=1}^M R_{m_t}$, derived from the kernel in \eqref{eq:Tucker_KRR_Kernel} is given by
\begin{equation}
    \varphi(\T{X}_{i})= (\Kron_{m=1}^{M}\M{U}_{M+1-m})\Tra\vec(\T{X}_{j}) = \M{P}\Tra \vec(\T{X}_{j})
    \label{eq:tucker_feature_map}
\end{equation}
where $\M{P}=\M{U}_{M}\Kron\cdots\Kron \M{U}_{1}\;\in\;\Real^{D\times R_t}$ and $D=\prod_{m=1}^{M}I_{m}$. Moreover, define the vectorized core tensor $\T{G}$ as $\V{g} = \vec(\T{G}) \in \Real^{R_t}$. Given $n$ training samples, the corresponding feature matrix $\M{\Phi} \in \Real^{n \times R_t}$ is constructed as:
\begin{equation}
    \M{\Phi} = \left(\varphi(\T{X}_{1}),\varphi(\T{X}_{2}),\dots,\varphi(\T{X}_{n})\right)\Tra
    \label{eq:tucker_feature_matrix}
\end{equation}
Our analysis proceeds in a similar fashion to the CP case, starting with an oracle setting where the Tucker components are known, followed by a general analysis that accounts for estimation error. We first adapt  Assumption~\ref{Assumption:assumption_cp} for the Tucker framework.

\begin{assumption}
\label{Assumption:assumption_tucker}
Let $\M{\Phi}$ be the feature matrix from \eqref{eq:tucker_feature_matrix}. Define the empirical quantities $\hat{\V{\eta}}_{\varphi}=\frac{1}{n}\M{\Phi}\Tra\V{y}\in\Real^{R_t}$ and $\hat{\M{\Sigma}}_{\varphi,\lambda}=\frac{1}{n}(\M{\Phi}\Tra\M{\Phi}+\lambda\M{I}_{R_t})\in\Real^{R_t\times R_t}$ for a fixed positive $\lambda$. We assume that
\begin{equation*}
    \lVert\hat{\V{\eta}}_{\varphi}-\V{\eta}_{\varphi}\rVert_{2}=\mathcal{O}_{p}(\frac{1}{\sqrt{n}}),\quad\lVert\hat{\M{\Sigma}}_{\varphi,\lambda}-\M{\Sigma}_{\varphi,\lambda}\rVert_{2}=\mathcal{O}_{p}(\frac{1}{\sqrt{n}})
\end{equation*}
where $\V{\eta}_{\varphi}=\E_{\T{X}_{*}}[\varphi(\T{X}_{*})y_*]$ and
$\M{\Sigma}_{\varphi,\lambda}=\M{\Sigma}_{\varphi} + \lambda\M{I}_{R_t} = \E_{\T{X}_{*}}[\varphi(\T{X}_{*})\varphi(\T{X}_{*})\Tra]+\lambda\M{I}_{R_t}$. Let $\M{\Sigma}_{\varphi}=\M{U}\M{\Lambda}\M{U}\Tra$ be the eigendecomposition of $\M{\Sigma}_{\varphi}$ with eigenvalues $\M{\Lambda}=\mathrm{diag}(v_{1},\dots,v_{R_t})$
where $v_{1}\geq v_{2}\geq\cdots\geq v_{R_t}>0$ (see Remark~\ref{remark:remark_tucker_eigenvalue}). Then
\begin{equation*}
\mathbf{\Sigma}_{\varphi,\lambda}=\sum_{i=1}^{R_t}(v_{i}+\lambda)u_{i}u_{i}\Tra,\quad\mathbf{\Sigma}_{\varphi,\lambda}^{-1}=\sum_{i=1}^{R_t}\frac{1}{v_{i}+\lambda}u_{i}u_{i}\Tra
\end{equation*}
Finally, denote $\V{\epsilon} = (\epsilon_1, \dots, \epsilon_n)\Tra$ and assume that $\vec(\T{X}_{*})\sim \mathrm{N}(\mathbf{0},\M{I}_{\prod_{m}I_{m}})$ and 
$\V{\epsilon}\sim \mathrm{N}(\mathbf{0},\sigma^{2}\M{I}_{n})$ are independent.
\end{assumption}

\begin{remark}
\label{remark:remark_tucker_eigenvalue}
The spectrum of $\M{\Sigma}_{\varphi}$ can be expressed in terms of the eigenvalues of each factor matrix $\M{U}_m$. By assumption $\vec(\T{X}_{*})\sim \mathrm{N}(\mathbf{0},\M{I}_{\prod_{m}I_{m}})$, using \eqref{eq:tucker_feature_map} gives
\begin{align*}
\M{\Sigma}_{\varphi} & = \E_{\T{X}_{*}}[\M{P}\Tra \vec(\T{X}_{*})\vec(\T{X}_{*})\Tra\M{P}] \\
& = \M{P}\Tra \E_{\T{X}_{*}}[\vec(\T{X}_{*})\vec(\T{X}_{*})\Tra]\M{P} \\
& = (\M{U}_{M}\Kron\cdots\Kron \M{U}_{1})\Tra (\M{U}_{M}\Kron\cdots\Kron \M{U}_{1}) \\
& = \M{U}_{M}\Tra \M{U}_{M}\Kron\cdots\Kron \M{U}_{1}\Tra \M{U}_{1}
\end{align*} 
where the last equality is by the mixed-product property of the Kronecker product (Lemma 4.2.10 in \citet{horn1994topics}). Since $\M{U}_{m}$ has full column rank $R_{m_t}$, $\M{U}_{m}\Tra \M{U}_{m} \in \Real^{R_{m_t} \times R_{m_t}}$ has a rank $R_{m_t}$ eigen-decomposition  $\M{U}_{m}\Tra \M{U}_{m} \!=\! \M{U}_{(m)}\M{\Lambda}_{(m)}\M{U}_{(m)}\Tra$ where $\M{\Lambda}_{(m)} \!=\! \text{diag}(v_{1}^{(m)}, \dots, v_{R_{m_t}}^{(m)})$ and $v_{1}^{(m)} \geq v_{2}^{(m)} \geq \cdots \geq v_{R_{m_t}}^{(m)} > 0$. Then the eigen-decomposition of $\M{\Sigma}_\varphi$ can be expressed as (Theorem 4.2.12 in \citet{horn1994topics}):
\begin{align*}
    \M{\Sigma}_\varphi & = (\M{U}_{(M)} \Kron \cdots \Kron \M{U}_{(1)}) (\M{\Lambda}_{(M)} \Kron \cdots \Kron \M{\Lambda}_{(1)}) (\M{U}_{(M)}\Tra \Kron \cdots \Kron \M{U}_{(1)}\Tra),
\end{align*}
and the spectrum of $\M{\Sigma}_\varphi$ (i.e., $\{v_1, \dots, v_{R_t}\}$) will be the cross-product of each $\M{U}_m$'s eigenvalues:
\begin{equation}
    \sigma(\M{\Sigma}_\varphi) = \{\prod_m v_{r_m}^{(m)} : v_{r_m}^{(m)} \in \sigma(\M{U}_{m}\Tra \M{U}_{m}),\, r_m = 1, \dots, R_{m_t},\, m = 1, \dots, M\}. \label{eq:tucker_KRR_eigenvalues}
\end{equation}
Consequently, the largest and smallest eigenvalues of $\M{\Sigma}_\varphi$ can be expressed as $v_1 = \prod_{m=1}^M v_{1}^{(m)} \geq v_R = \prod_{m=1}^M v_{R_m}^{(m)} > 0$.

\end{remark}
We present the expected optimism for the Tucker regression model \eqref{eq:tensor_regression_model_tucker} 
under Assumption~\ref{Assumption:assumption_tucker}, 
when the target ranks match the true underlying ranks $\{R_1, \dots, R_M\}$.

\begin{theorem}
\label{thm:thm_tucker_true_rank}
(Expected Optimism of Tucker Regression Under True Rank in ``Random-$\T{X}$" Design)
Under Assumption~\ref{Assumption:assumption_tucker}, the expected optimism $\mathrm{OptR}_{\T{X}}^{\mathrm{(true)}}$ for the Tucker regression \eqref{eq:tensor_regression_model_tucker} at true rank
$R_{1_t} = R_1, \dots, R_{M_t} = R_M$ is:
\begin{equation}
\mathrm{OptR}_{\T{X}}^{\mathrm{(true)}}=\frac{2\left(\sigma^{2}+\frac{\lambda^2 v_1}{(v_1+\lambda)^2}\right)}{n}\sum_{r=1}^{R}\frac{v_{r}^2}{(v_{r}+\lambda)^{2}}+\mathcal{O}_{p}(n^{-\frac{3}{2}})\label{eq:tensorkrr_true_rank_tucker}
\end{equation}
where $R = \prod_m R_m = \prod_{m} R_{m_t}$.
\end{theorem}

\begin{remark}
\label{remark:remakr_cp_tucker_connection}
This result naturally contains the CP case. When the Tucker decomposition of $\T{B}$ has a common rank $R_0$ across all modes and a superdiagonal core, we have the following equality \citep{kolda2009tensor}:
\begin{equation*}
\vec(\T{B})=\bigl(\M{U}_{M}\Kron\cdots\Kron \M{U}_{1}\bigr)\vec(\T{G})=\bigl(\M{U}_{M}\odot\cdots\odot \M{U}_{1}\bigr)
\end{equation*}
which implies that $\T{B}$ reduces to a rank-$R_0$ CP decomposition as $\T{B} =\llbracket\M{U}_{1},\dots,\M{U}_{M}\rrbracket$. Therefore, the Tucker eigenvalue decomposition \eqref{eq:tucker_KRR_eigenvalues} in Remark~\ref{remark:remark_tucker_eigenvalue} will reduce to 
\begin{equation}
\bigl\{\prod_{m=1}^M v_{r_m}^{(m)} : r_1 = r_2 = \cdots = r_M = r,\; r = 1,\dots,R_0\bigr\}
= \bigl\{\prod_{m=1}^M v_{r}^{(m)} : r = 1,\dots,R_0\bigr\}
\label{eq:cp_special_case}
\end{equation}
and the expected optimism \eqref{eq:tensorkrr_true_rank_tucker} will reduce to the CP case \eqref{eq:tensorkrr_true_rank_CP}.
\end{remark}

Next, we discuss the over-specified case where the target Tucker ranks $R_{m_t} > R_m$ for at least one $m = 1, \dots, M$. Denote its target ranks $(R_{1_t}, \dots, R_{M_t})$ Tucker decomposition as $\T{B} = \tilde{\T{G}}\times_{1} \tilde{\M{U}}_{1}\;\cdots\;\times_{m}\tilde{\M{U}}_{m}\;\cdots\;\times_{M}\tilde{\M{U}}_{M}$ where $\tilde{\T{G}} \in \Real^{R_{1_t} \times \cdots \times R_{M_t}}$ and $\tilde{\M{U}}_{m} \in \Real^{I_m \times R_{m_t}}$. Let $\varphi_{R_{t}}(\T{X}_{*})\in\Real^{R_{t}}$ and $\M{\Phi}_{R_{t}}\in\Real^{n\times R_{t}}$ be constructed the same way as \eqref{eq:tucker_feature_map} and \eqref{eq:tucker_feature_matrix} where $R_t = \prod_m R_{m_t}$. And denote $\M{\Sigma}_{\varphi_{R_{t}}}=\E_{\T{X}_{*}}[\varphi_{R_{t}}(\T{X}_{*})\varphi_{R_{t}}(\T{X}_{*})\Tra]\in\Real^{R_{t}\times R_{t}}$ and $\V{\eta}_{\varphi_{R_{t}}}=\E_{\T{X}_{*}}[\varphi_{R_{t}}(\T{X}_{*})y_*]\in\Real^{R_t}$. In this scenario, the expected optimism is expressed as:
\begin{theorem}
\label{thm:thm_tucker_over_rank}
(Expected Optimism of Tucker Regression for Over-Specified Rank in ``Random-$\T{X}$" Design)
Under Assumption~\ref{Assumption:assumption_tucker}, the expected optimism $\mathrm{OptR}_{\T{X}}^{\mathrm{(over)}}$ for the Tucker regression \eqref{eq:tensor_regression_model_tucker} at target ranks
$R_{m_t} > R_m$ for at least one $m = 1. \dots, M$ is:
\begin{equation}
\mathrm{OptR}_{\T{X}}^{\mathrm{(over)}}=\frac{2\left(\sigma^{2}+\frac{\lambda^2\tilde{v}_1}{(\tilde{v}_1+\lambda)^2}\right)}{n}\sum_{r=1}^{R_{t}}\frac{\tilde{v}_{r}^2}{(\tilde{v}_{r}+\lambda)^{2}}+\mathcal{O}_{p}(n^{-\frac{3}{2}})\label{eq:tensorkrr_over_rank_tucker}
\end{equation}
where $\{\tilde{v}_{r}\}_{r=1}^{R_{t}}$ and $\tilde{\mathbf{U}}$
are the eigen-decomposition of $\mathbf{\Sigma}_{\phi_{R_{t}}}=\tilde{\mathbf{U}}\text{diag}(\tilde{v}_{1},\dots,\tilde{v}_{R_{t}})\tilde{\mathbf{U}}\Tra$
with $\tilde{v}_{1}\geq\cdots\geq\tilde{v}_{R_{t}}\geq0$.  
\end{theorem}

When the target Tucker rank is under-specified with $R_{m_t} < R_m$ for at least one $m = 1, \dots, M$, we will have a truncated Tucker decomposition \citep{kolda2009tensor}. Let it be given as $\T{B}_{R_t} = \tilde{\T{G}}\times_{1} \tilde{\M{U}}_{1}\;\cdots\;\times_{m}\tilde{\M{U}}_{m}\;\cdots\;\times_{M}\tilde{\M{U}}_{M}$ where $\tilde{\M{U}}_{m} = {\M{U}}^{(R_{m_t})}_{m} \in \Real^{I_m \times R_{m_t}}$ is the best rank $R_{m_t}$ approximation of each true-rank factor matrix $\M{U}_m$. Denote the approximation residual as $\V{\Delta}=\vec(\T{B})-(\tilde{\M{U}}_{M}\Kron\cdots\Kron\tilde{\M{U}}_{1})\vec{(\tilde{\T{G}})}\in\Real^{\prod_{m}I_{m}}$. Then the expected optimism has an additional term accounting for this low-rank approximation error.
\begin{theorem}
\label{thm:thm_tucker_under_rank}
(Expected Optimism of Tucker Regression for Under-Specified Rank in ``Random-$\T{X}$" Design) Under Assumption~\ref{Assumption:assumption_tucker}, the expected optimism $\mathrm{OptR}_{\T{X}}^{\mathrm{(under)}}$ for the Tucker regression \eqref{eq:tensor_regression_model_tucker} at target ranks
$R_{m_t} < R_m$ for at least one $m = 1. \dots, M$ is:
\begin{equation}
\mathrm{OptR}_{\T{X}}^{\mathrm{(under)}}=\frac{2\left(\sigma^{2}+\frac{\lambda^2\tilde{v}_1}{(\tilde{v}_1+\lambda)^2}\right)}{n}\sum_{r=1}^{R_{t}}\frac{\tilde{v}_{r}^2}{(\tilde{v}_{r}+\lambda)^{2}}
+ \frac{2 \lVert\V{\Delta}\rVert^{2} }{n} \sum_{r=1}^{R_{t}}\frac{\tilde{v}_{r}^2}{(\tilde{v}_{r}+\lambda)^{2}}
+\mathcal{O}_{p}(n^{-\frac{3}{2}})
\label{eq:tensorkrr_under_rank_tucker}
\end{equation}
where $\{\tilde{v}_{r}\}_{r=1}^{R_{t}}$ and $\tilde{\mathbf{U}}$
are from the eigen-decomposition of $\mathbf{\Sigma}_{\varphi_{R_{t}}}=\tilde{\mathbf{U}}\text{diag}(\tilde{v}_{1},\dots,\tilde{v}_{R_{t}})\tilde{\mathbf{U}}\Tra$ with $\tilde{v}_{1}\geq\cdots\geq\tilde{v}_{R_{t}}>0$.  
\end{theorem}

The above results enable us to make a similar statement as Proposition~\ref{prop:CP_proposition} in Tucker case, which shows that the expected optimism $\mathrm{OptR}_{\T{X}}$ is also non-decreasing whenever $R_{n_t} \neq R_m$ for at least one $m = 1, \dots, M$.

\begin{proposition}
\label{prop:tucker_proposition}
Under Assumption~\ref{Assumption:assumption_tucker}, let the true tensor coefficient admit a Tucker decomposition of rank $(R_1, \dots, R_M)$ and denote the target Tucker rank as $(R_{1_t}, \dots, R_{M_t})$, with total ranks defined as $R = \prod_m R_m$ and $R_t = \prod_m R_{m_t}$, respectively. Assuming the regularization $\lambda$ remains sufficiently small (i.e., $\lambda \ll v_R \wedge \tilde{v}_{R_t}$), if $R_{m_t} \geq R_m\,$ for at least one $m = 1, \dots, M$, then $\mathrm{OptR}_{\T{X}}^{(\mathrm{over})} \geq \mathrm{OptR}_{\T{X}}^{(\mathrm{true})}$; and if $R_{m_t} \leq R_m\,$ for at least one $m = 1, \dots, M$ and $\lVert\V{\V{\Delta}}\rVert^{2} \geq \sigma^2 \frac{R - R_t}{R_t}$, then $\mathrm{OptR}_{\T{X}}^{(\mathrm{under})} \geq \mathrm{OptR}_{\T{X}}^{(\mathrm{true})}$.
\end{proposition}

Extending the oracle case to the general case follows the same spirit as in Section~\ref{sec:CP_regression}.  Let $\hat{\T{B}}=\hat{\T{G}}\times_{1} \hat{\M{U}_{1}}\times_{2}\hat{\M{U}_{2}}\;\cdots\times_{M}\hat{\M{U}_{M}}$ be an estimate of the true ranks-$(R_1, \dots, R_M)$ coefficient tensor $\T{B}$, with vectorized estimation error $\V{e} =\vec(\hat{\T{B}})-\vec(\T{B}) \in \Real^{\prod_m I_m}$.
Let $\hat{\varphi}_{R}(\T{X}_{*})\in\Real^{R}$ and $\hat{\M{\Phi}}_{R}\in\Real^{n\times R}$ be constructed in the same way as \eqref{eq:tucker_feature_map} and \eqref{eq:tucker_feature_matrix} but with Tucker components of $\hat{\T{B}}$ where $R = \prod_m R_m$. Let  $\hat{\M{\Sigma}}_{\varphi_{R}}=\E[\hat{\varphi}_{R}(\T{X}_{*})\hat{\varphi}_{R}(\T{X}_{*})\Tra]\in\Real^{R\times R}$ denote the population
variance. We can obtain a plug-in estimate of the optimism at true rank $(R_1, \dots, R_M)$ by using those estimated Tucker components.

\begin{theorem}
\label{thm:thm_Tucker_true_rank_general}
(Expected Optimism of Tucker Regression Under True Rank with Estimated Components in ``Random-$\T{X}$" Design.)
Under Assumption~\ref{Assumption:assumption_cp}, the expected optimism $\widehat{\mathrm{OptR}}_{\T{X}}^{(\mathrm{true})}$
for the Tucker regression \eqref{eq:tensor_regression_model_tucker} at true rank
$R_{1_t} = R_1, \dots, R_{M_t} = R_M$ using the estimated components is 
\begin{equation}
\widehat{\mathrm{OptR}}_{\T{X}}^{(\mathrm{true})}=\frac{2\left(\sigma^{2}+\frac{\lambda^2 \hat{v}_1}{(\hat{v}_1+\lambda)^2}\right)}{n}\sum_{r=1}^{R}\frac{\hat{v}_{r}^2}{(\hat{v}_{r}+\lambda)^{2}}
+\frac{2 \lVert\mathbf{e}\rVert^{2} }{n} \sum_{r=1}^{R}\frac{\hat{v}_{r}^2}{(\hat{v}_{r}+\lambda)^{2}}
+\mathcal{O}_{p}(n^{-\frac{3}{2}})\label{eq:tensorkrr_true_rank_general}
\end{equation}
where $\{\hat{v}_{r}\}_{r=1}^{R}$ and $\hat{\mathbf{U}}$ are from
the eigen-decomposition of $\hat{\mathbf{\Sigma}}_{\varphi_{R}}=\hat{\mathbf{U}}\text{diag}(\hat{v}_{1},\dots,\hat{v}_{R})\hat{\mathbf{U}}\Tra$
with $\hat{v}_{1}\geq\hat{v}_{2}\cdots\geq\hat{v}_{R}\geq0$.
\end{theorem}
The proof follows the proof of Theorem \ref{thm:thm_tucker_under_rank} with $\V{\Delta}$ replacing $\V{e}$ and $\varphi_{R_{t}}$ replacing $\hat{\varphi}_{R}$. Using analogous arguments to those for the misspecified rank cases in Theorem \ref{thm:thm_tucker_over_rank}
and \ref{thm:thm_tucker_under_rank}, one can show that the expected optimism at the true rank $(R_1, \dots, R_M)$ is also minimized in the general setting. This result showcases the optimality of using (a plug-in estimate of) optimism as a rank selection tool for tensor regression.

\section{CP Ensemble Regression}
\label{sec:CP_ensemble}
Our analysis thus far has focused on the Random-$\T{X}$ optimism of a single low-rank CP and Tucker regression model. This framework can be extended to an ensemble setting. From this perspective, averaging multiple low-rank linear smoothers reduces variance. The optimism of the resulting ensemble is (up to small covariance terms) no larger than the mean optimism of its individual members, yielding more stable  predictions and rank choices. This connects to previous work on tensor ensembles, such as boosted or shared models, which use CP or Tucker regressions as weak learners to adapt to data heterogeneity that a single global model fails to capture \citep{luo2023sharded, luo2024efficient, bu2025improving}. In this section, we formalize this by focusing on CP ensemble regression, where each ensemble member is a low-rank CP tensor model. We demonstrate that the combined ensemble estimator is equivalent to a single low-rank CP tensor regression on the full dataset and its optimism is upper-bounded by the sample-average optimism of the individual ensembles. Some of the results are built on an extension of the KRR optimism analysis to additive feature mappings, and we leave those technical details to the Supplementary Materials.

We begin by introducing the CP ensemble regression setup. Let $\mathcal{D} = \{\T{X}_i, y_i\}_{i=1}^n$ denote the full training dataset, which consists of independent samples of the tensor $\T{X}_i \in \Real^{I_1\times \cdots \times I_M}$ and the scalar $y_i \in \Real$. The ensembles are generated by subsampling (slicing) along the (row) index $i$. Assume there are $K$ subsets. For $k = 1,2,\dots, K$, we denote the $k$th ensemble training subset by
$
\mathcal{D}_k = \{\T{X}_i^{(k)}, y_i^{(k)}\}_{i=1}^{n_k} \subset \mathcal{D}, 
$
where $1 \leq n_k \leq n$. Let the tensor coefficient fitted on $\mathcal{D}_k$ with CP rank $R_k$ be 
$
\T{B}^{(k)} = \sum_{r=1}^{R_{k}}{\V{\beta}}_{1}^{(k,r)}\circ\cdots\circ{\V{\beta}}_{M}^{(k,r)}. 
$
Given each learner's estimate ${\T{B}}^{(k)}$, the ensemble-averaged tensor coefficient is defined as 
$
    \bar{\T{B}}=\frac{1}{K}\sum_{k=1}^{K}\T{B}^{(k)}.
$
Then the following lemma shows that $\bar{\T{B}}$ can itself be viewed as a CP decomposition of rank $R_{\rm ens}$ with $\max_k R_k \leq R_{\rm ens} \leq \sum_k R_k.$. In other words, the ensemble-averaged estimator is equivalent to performing a rank-$R_{\rm ens}$ CP regression on $\mathcal{D}$.
\begin{lemma}
(CP Representation of Ensemble-Averaged Estimator) \label{lemma:lemma_ensemble_CP}
The ensemble-averaged coefficient tensor $\bar{\T{B}}$ has a rank-$R_{\rm ens}$ CP decomposition:
\begin{equation}
    \bar{\T{B}} = \frac{1}{K}\sum_{r=1}^{R_{\rm ens}}\bar{\V{\beta}}_{1}^{(r)}\circ\cdots\circ\bar{\V{\beta}}_{M}^{(r)}
    \label{eq:ensemble_CP_decomp}
\end{equation}
where $\max_k R_k \leq R_{\rm ens} \leq \sum_k R_k$.
\end{lemma}

Following the KRR framework used in Section~\ref{sec:CP_regression}, we formalize the relationship between the optimism of the ensemble estimator $\bar{\T{B}}$ and the optimisms of the individual learners $\T{B}^{(k)}$ under the following assumptions.
\begin{assumption}
\label{Assumption:assumption_CP_ensemble} Denote the ensemble feature map $\phi^{(k)}(\T{X}) = \left(\phi_1^{(k)}(\T{X}), \dots, \phi_{R_k}^{(k)}(\T{X}) \right)\Tra \in \Real^{R_k}$, where $\phi_r^{(k)}(\T{X}) = \langle \V{v}^{(k,r)}, \, \vec{(\mathcal{X})} \rangle$ and $\V{v}^{(k,r)} = {\V{\beta}}_{1}^{(k,r)}\otimes\cdots\otimes{\V{\beta}}_{M}^{(k,r)} \in \Real^{\prod_m I_m}$ is the vectorized rank-1 tensor component. Assuming a fixed positive $\lambda$ that is small enough, for each $k=1, \dots, K$, let the empirical quantities be $\hat{\V{\eta}}_{\phi^{(k)}}=\frac{1}{n_k}\M{\Phi}_{(k)}{\Tra}\V{y}^{(k)}\in\Real^{R_k}$ and $\hat{\M{\Sigma}}_{\phi^{(k)},\lambda}=\frac{1}{n}(\M{\Phi}_{(k)}\Tra\M{\Phi}_{(k)}+\lambda\M{I}_{R_k})\in\Real^{R_k\times R_k}$, where $\M{\Phi}_{(k)}$ is formulated as \eqref{eq:CP_feature_matrix} using $\phi^{(k)}$ and $\V{y}^{(k)} = (y_1^{(k)}, \dots, y_{n_k}^{(k)})\Tra$. We assume that:
\begin{equation*}
    \lVert\hat{\V{\eta}}_{\phi^{(k)}}-\V{\eta}_{\phi^{(k)}}\rVert_{2}=\mathcal{O}_{p}(\frac{1}{\sqrt{n_k}}),\quad\lVert\hat{\M{\Sigma}}_{\phi^{(k)},\lambda}-\M{\Sigma}_{\phi^{(k)},\lambda}\rVert_{2}=\mathcal{O}_{p}(\frac{1}{\sqrt{n_k}})
\end{equation*}
Similarly, from Lemma \ref{lemma:lemma_ensemble_CP}, define the feature map $\bar{\phi}(\T{X})  = \frac{1}{K}\left(\bar{\phi}_1(\T{X}), \dots, \bar{\phi}_{R_{\rm ens}}(\T{X}) \right)\Tra \in \Real^{R_{\rm ens}}$ for $\bar{\T{B}}$, where $\bar{\phi}_r(\T{X}) = \langle \V{u}_r, \, \vec{(\mathcal{X})} \rangle$ and $\V{u}_r = \bar{\V{\beta}}_{1}^{(r)}\otimes\cdots\otimes\bar{\V{\beta}}_{M}^{(r)} \in \Real^{\prod_m I_m}$ for $r=1,\dots, R_{\rm ens}$. Using the same $\lambda$, let its empirical quantities on the full dataset be $\hat{\V{\eta}}_{\bar\phi}=\frac{1}{n}\bar{\M{\Phi}}{\Tra}\V{y}\in\Real^{R_{\rm ens}}$ and $\hat{\M{\Sigma}}_{\bar\phi,\lambda}=\frac{1}{n}(\bar{\M{\Phi}}\Tra\bar{\M{\Phi}}+\lambda\M{I}_{R_{\rm ens}})\in\Real^{R_{\rm ens}\times R_{\rm ens}}$ with assumptions that
\begin{equation*}
    \lVert\hat{\V{\eta}}_{\bar\phi}-\V{\eta}_{\bar\phi}\rVert_{2}=\mathcal{O}_{p}(\frac{1}{\sqrt{n}}),\quad\lVert\hat{\M{\Sigma}}_{\bar\phi,\lambda}-\M{\Sigma}_{\bar\phi,\lambda}\rVert_{2}=\mathcal{O}_{p}(\frac{1}{\sqrt{n}})
\end{equation*}
Finally, we assume that $\vec(\T{X}_{*})\sim N(\V{0},\M{I}_{\prod_{m}I_{m}})$,
$\V{\epsilon}\sim N(\V{0},\sigma^{2}\M{I}_{n})$, and
$\V{\epsilon}\perp\!\!\!\perp \vec(\T{X}_{*})$. And for each training subset $k$, we have $0<\frac{n_k}{n} <1$ as $n, n_k \rightarrow \infty$.
\end{assumption}

Let $\mathrm{OptR}_{\T{X}}^{\mathrm{(k)}}$ be the expected optimism for the $k$th learner fitted on training subset $\mathcal{D}_k$.   Let $\mathrm{OptR}_{\T{X}}^{\mathrm{(ens)}}$ be the expected optimism of the ensemble-averaged estimator $\bar{\T{B}}$. Then the following theorem shows that $\mathrm{OptR}_{\T{X}}^{\mathrm{(ens)}}$ is upper bounded by a sample-average of $\mathrm{OptR}_{\T{X}}^{\mathrm{(k)}}$.

\begin{theorem}
(Upper Bound for Expected Random Optimism of Ensemble-Averaged CP Regression)
\label{thm:thm_CP_ensemble_bound} Under Assumption~\ref{Assumption:assumption_CP_ensemble}, the expected optimism $\mathrm{OptR}_{\T{X}}^{\mathrm{(ens)}}$ of the ensemble-averaged estimator $\bar{\T{B}}$ is upper bounded by a sample-average of individual learner's expected optimism as:
\begin{equation}
    \mathrm{OptR}_{\T{X}}^{\mathrm{(ens)}} \leq  \sum_{k=1}^{K} \frac{n_k}{n} \mathrm{OptR}_{\T{X}}^{\mathrm{(k)}}
    \label{eq:ensemble_optimism_up_bound}
\end{equation}
with equality when all component vectors across learners are jointly linearly independent.
\end{theorem}

\begin{remark}
\label{remark:remark_CP_ensemble}
Theorem \ref{thm:thm_CP_ensemble_bound} tells us that the expected random optimism of ensemble-averaged estimator does not exceed a sample-fraction-weighted average of the individual learners' optimism. When the size of the training subsets are the same, i.e., $n_k = \frac{n}{K} \rightarrow \frac{n_k}{n} = \frac{1}{K}$), inequality \eqref{eq:ensemble_optimism_up_bound} becomes:
\begin{equation*}
    \mathrm{OptR}_{\T{X}}^{\mathrm{(ens)}}  \leq  \frac{1}{K}\sum_{k=1}^{K}\mathrm{OptR}_{\T{X}}^{\mathrm{(k)}}.
\end{equation*}
So naively averaging $K$ learners' estimates can never perform worse (and usually perform better) than the arithmetic mean of their individual optimism levels. Additionally, note that under the degenerate case where the concatenated CP components are linearly independent and $\sum_k R_k \geq \prod_m I_m$, where $\prod_m I_m$ is the dimension of $\vec(\T{X})$, the CP ensemble no longer reduces the dimensionality. Then the same arguments as in the proof of Theorem \ref{thm:thm_CP_ensemble_bound} can show that the right-hand side is algebraically equivalent to an ordinary least-squares regression on $\vec{(\T{X})}$, and the optimism of the ensemble-averaged estimator  coincides with that of the full linear model.
\end{remark}

This line of ensemble analysis connects to the recent work on tensor regression model averaging (TRMA) \citep{bu2025improving}, which combines tensor regressors of different fixed ranks using data-driven weights, typically chosen by cross-validation. Our framework produces, for each candidate rank, an optimism-corrected estimate of its ``Random-$\T{X}$" prediction error. It can be shown that, under the same assumptions used for our main results, these optimism-corrected risks and the $K$-fold cross-validation criteria employed in TRMA converge to the \emph{same} population prediction error. TRMA can thus be viewed as a smoothed, averaged version of the same underlying selection principle (see Section~\ref{sec:trma-optimism} in the Supplementary Materials).

\section{Numerical Studies}
\label{sec:simulation}

In this section, we present two numerical studies to validate the theoretical results developed in Sections \ref{sec:CP_regression} and \ref{sec:Tucker_regression}. Our experimental design mirrors the progression of our theoretical analysis, proceeding from an oracle setting to a general estimation setting. We first investigate the expected optimism within the KRR framework, assuming the decomposition components of the true coefficient tensor $\T{B}$ are known. Subsequently, we evaluate it in the more practical case through estimated CP and Tucker regression models. Results in both scenarios confirm that the expected optimism is indeed minimized at the true underlying rank, demonstrating its validity and utility as a rank selection criterion for tensor regression. We also conducted studies to evaluate our results for ensemble CP regression discussed in Section~\ref{sec:CP_ensemble}. We defer these results to the Supplementary Materials (Section~\ref{sec:CP_ensemble_results}).


\begin{figure}[H]
    \centering
    \subfloat[\centering Varying Noise Level]{{\includegraphics[width=0.494\textwidth]{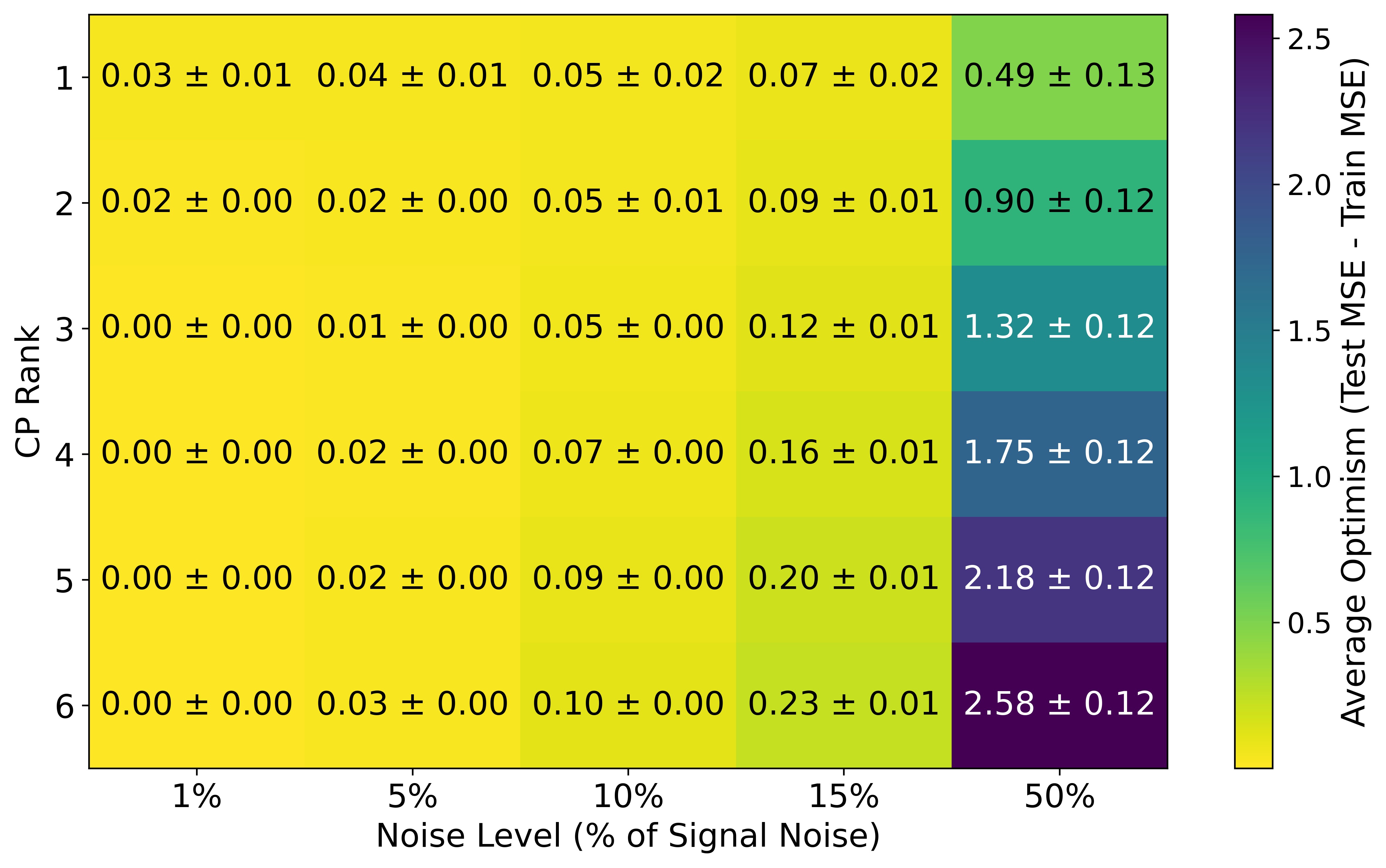} }}%
    \subfloat[\centering Varying Sample Size]{{\includegraphics[width=0.494\textwidth]{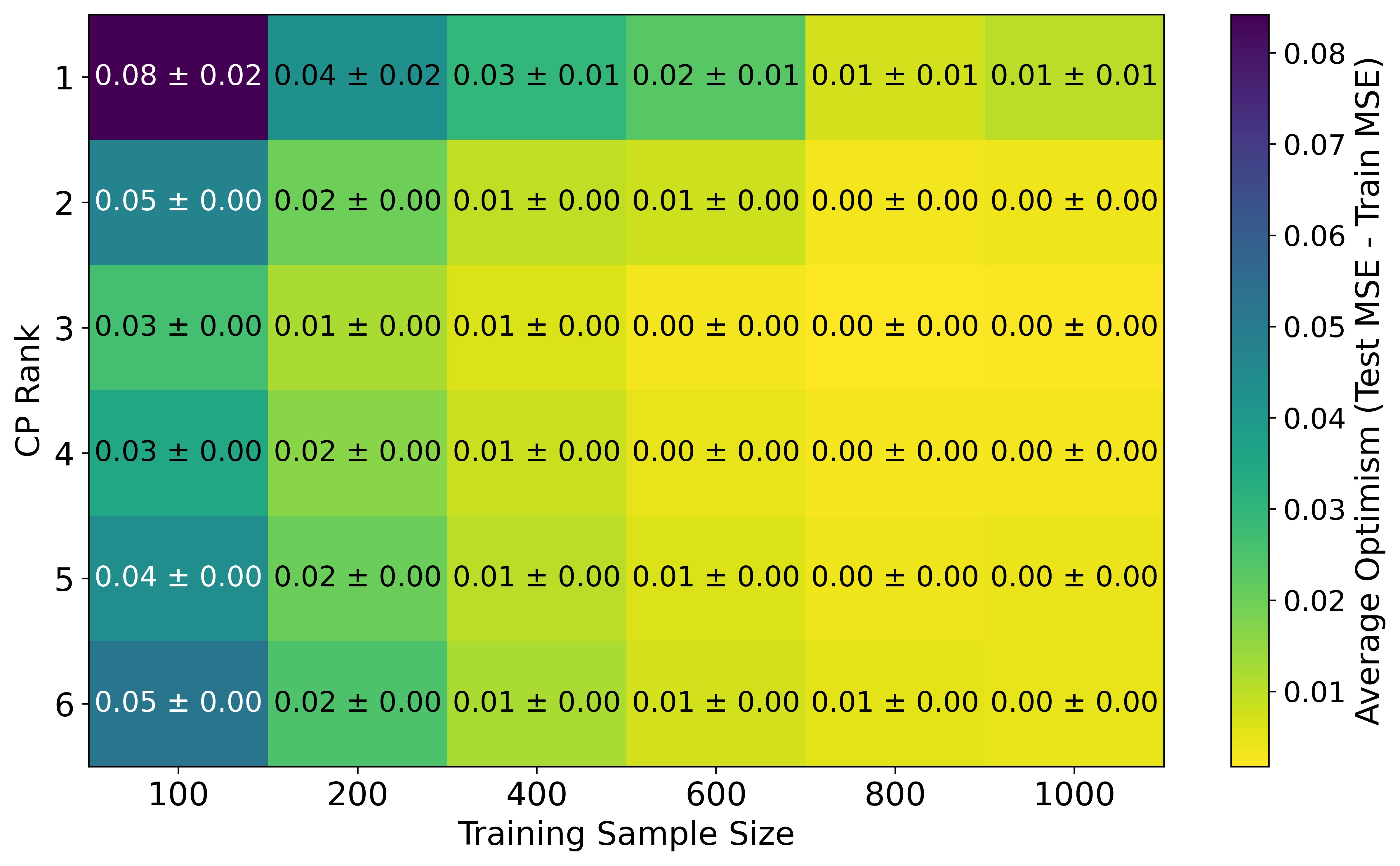} }}%
    \caption{\label{fig:CP_krr_vary_noise_and_sample}Average optimism of the tensor KRR model over $10000$ Monte Carlo replicates. The left panel varies the target CP rank and noise level, with the training sample size held constant at $n_{\text{train}} = 200$. The right panel varies the target CP rank and the training sample size, with the noise level held constant at $5\%$ of the signal standard deviation. In all cases, the regularization parameter is $\lambda = 1$. Results are shown for the oracle case, where the CP kernel is constructed from the true tensor coefficient $\T{B}$, which has a default rank of $3$.}%
\end{figure}

In all experiments, the covariate $\T{X}_i$ is a 3-mode tensor of size $\T{X}_i \!\in \! \mathbb{R}^{10 \times 8 \times 12}$, with entries randomly generated from $\vec(\T{X}_i) \sim \mathrm{N}(\V{0}, \M{I})$. The scalar response $y_i \in \mathbb{R}$ is then simulated using the model \eqref{eq:tensor_regression_model} with noise $\epsilon_i \sim \mathrm{N}(0, \sigma^2)$. The noise level $\sigma$ is set as a proportion of the signal variation (i.e., $\sigma = s \sigma_s$ with $s \in \{0.01, 0.05, 0.1, 0.15, 0.5\}$) to control the signal-to-noise ratio. For CP regression \eqref{eq:tensor_regression_model_CP}, we construct the true coefficient tensor $\T{B}$ with a rank-3 CP decomposition (so the true rank is $R=3$). For Tucker regression \eqref{eq:tensor_regression_model_tucker}, $\T{B}$ is constructed with a Tucker decomposition of rank $(3,3,3)$ (so the true rank is $R=(3,3,3)$). The test sample size is fixed at $n_{\mathrm{test}} = 100$, and we vary the training sample size $n_{\mathrm{train}}$ from $100$ to $1000$. The expected optimism is evaluated under the mean-squared-error (MSE) loss (i.e., finite sample estimate of \eqref{eq:test_mse_random} and \eqref{eq:train_mse}) and computed by averaging $10000$ Monte Carlo (MC) replicates as an approximation to expectation (see 
Algorithm 1 from \citet{luo2025optimism} for details).




\begin{figure}[H]
\centering
\includegraphics[width=1\textwidth]{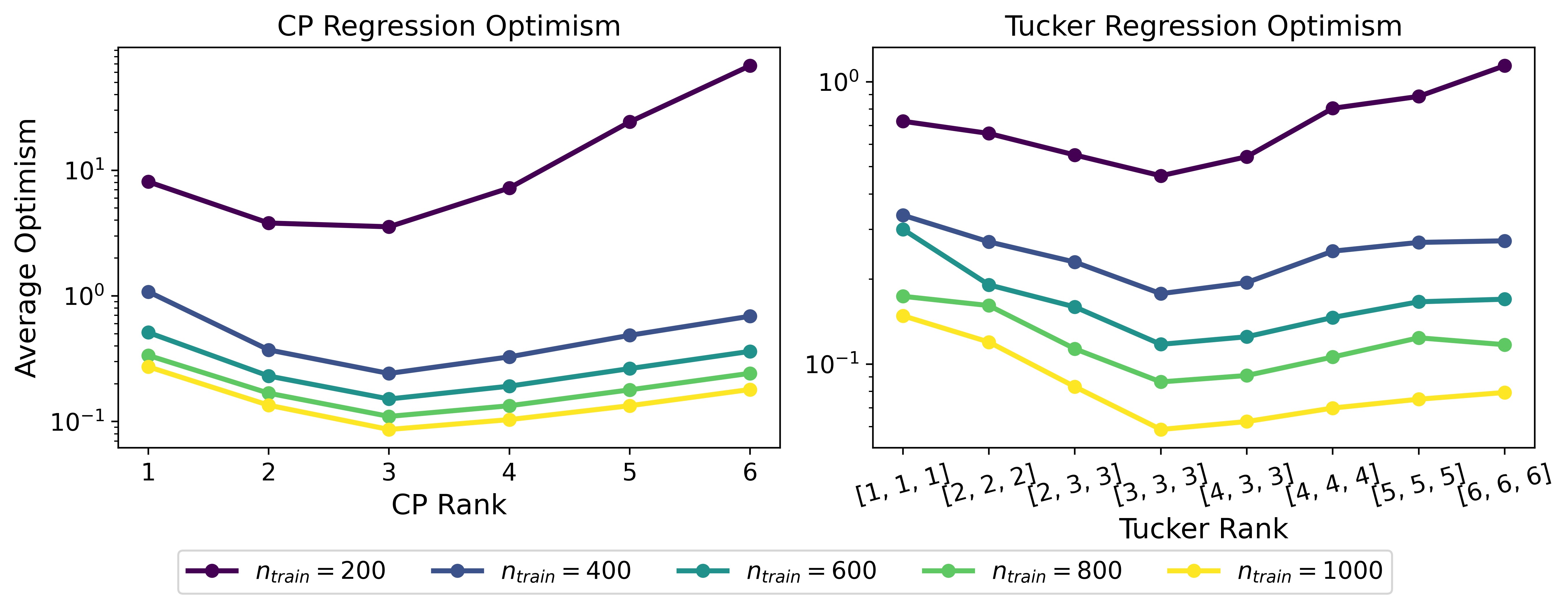}
\caption{\label{fig:CP_tucker_est_vary_sample}Average optimism for low-rank CP and Tucker regression with varying ranks and sample size. The left panel shows the results for CP regression over $10000$ MC replicates. The noise level is fixed at $5\%$ of the signal standard deviation, and the true CP rank is $R = 3$. The right panel shows the results for Tucker regression (using TensorGP from \citet{yu2018tensor}). The noise level is fixed at $1\%$ of the signal standard deviation to maintain a similar noise magnitude as the CP case, and the  true Tucker rank is $R = (3,3,3)$. Here we reduce the MC replicates to $100$ due to its high computational complexity and our fixed computational power. In both plots, the model rank varies along the x-axis, while different colors correspond to varying training sample sizes.}
\end{figure}

Figures \ref{fig:CP_krr_vary_noise_and_sample} presents two heatmaps of the expected optimism in the oracle case, where we fit a KRR model using the CP-component kernel \eqref{eq:CP_KRR_Kernel} for varying target ranks, noise levels, and training sample sizes. Both plots clearly illustrate that the expected optimism is minimized at the true rank ($R=3$) across all cases, which empirically validates Proposition \ref{prop:CP_proposition}. Furthermore, the observed asymptotic trends corroborate our derived expressions for the expected optimism. Specifically, holding other parameters constant, $\mathrm{OptR}_{\T{X}}$ scales approximately as $\mathcal{O}(n^{-1})$ with sample size and $\mathcal{O}(\sigma^2)$ with noise variance, further verifying our theoretical results.

Next, we evaluate the behavior of optimism in the practical setting under the same simulation design. We fit the CP regression model using the Python \texttt{tensorly} library \citep{tensorly} and the Tucker regression model using the TensorGP approach from \citet{yu2018tensor}. Figure \ref{fig:CP_tucker_est_vary_sample} presents the expected optimism for both CP and Tucker regression with varying ranks and training sample sizes. Both plots indicate a clear minimization of optimism at the respective true ranks, further demonstrating the capability of optimism as a rank selection method. For comparison, Figure \ref{fig:aic_bic_analysis_tensor_regression} in the Supplementary Materials shows the AIC and BIC measures for both models under the same settings, and we see that these traditional criteria can be unreliable. For CP regression, AIC fails to identify the true rank at any sample size, whereas BIC succeeds. For Tucker regression, however, AIC correctly identifies the true rank in all cases, while BIC fails to do so when the training sample size is small. These results reinforce the superiority of our optimism framework as a more reliable approach for tensor rank selection.


\section{Real Data Study}
\label{sec:Real_Data}
We now evaluate the performance of our optimism-based rank selection criterion on two real-world application examples. The first example applies optimism to a CP regression model for an age prediction task using facial images, demonstrating its utility in selecting an optimal CP rank with better predictive performance. In the second example, we extend this framework to the tensor-based compression of neural networks. In this context, we illustrate the potential of using the expected optimism to guide the choice of compression structure, effectively balancing the network complexity and prediction accuracy.

\subsection{CP Regression on FGNET}
\label{sec:cp_regresion_FGNET}
We perform CP regression using the FGNET dataset from \citet{fu2015robust}. The dataset contains 1002 images of 82 individuals with ages ranging from 0 to 69 years, and each image is resized to 90 by 90 pixels \citep{lock2018tensor}. To formulate the tensor structure, we consider the color intensity (i.e., red, green, and blue) as an additional mode along with the height and width of the image, resulting in a 3-mode tensor covariate $\T{X}_i \in \Real^{90 \times 90 \times 3}$. The CP regression model is trained on 80\% of the data and tested on the remaining 20\% for various CP ranks. The expected optimism is evaluated using MSE loss and computed over 100 MC replicates.

The resulting optimism values, alongside AIC and BIC, are shown in Figure \ref{fig:CP_regression_opt_aic_bic}. The accompanying table on the right indicates the optimal rank chosen by each selection method and its corresponding test MSE. Although the global minimum for optimism occurs at rank 1, the plot reveals this rank-1 model has a very high training MSE, suggesting that it fails to capture meaningful data structure. We therefore select the rank corresponding to the second-lowest optimism, which represents a more stable model. Based on the results, the ranks selected by AIC and BIC yield much higher test MSE compared to the rank selected by optimism. The plot also reveals that at these selected ranks, the test MSE is still decreasing, indicating these models are still learning to generalize and have not yet achieved optimal predictive performance. By focusing only on in-sample error, AIC and BIC tend to make suboptimal model selections that are fitting to training patterns rather than generalizing. This reinforces the failure of traditional criteria to account for predictive power and highlights the benefit of the optimism framework in balancing model fit with true generalizability.

\begin{figure}[H]
    \begin{minipage}[c]{0.7\textwidth}
        \includegraphics[width=\linewidth]{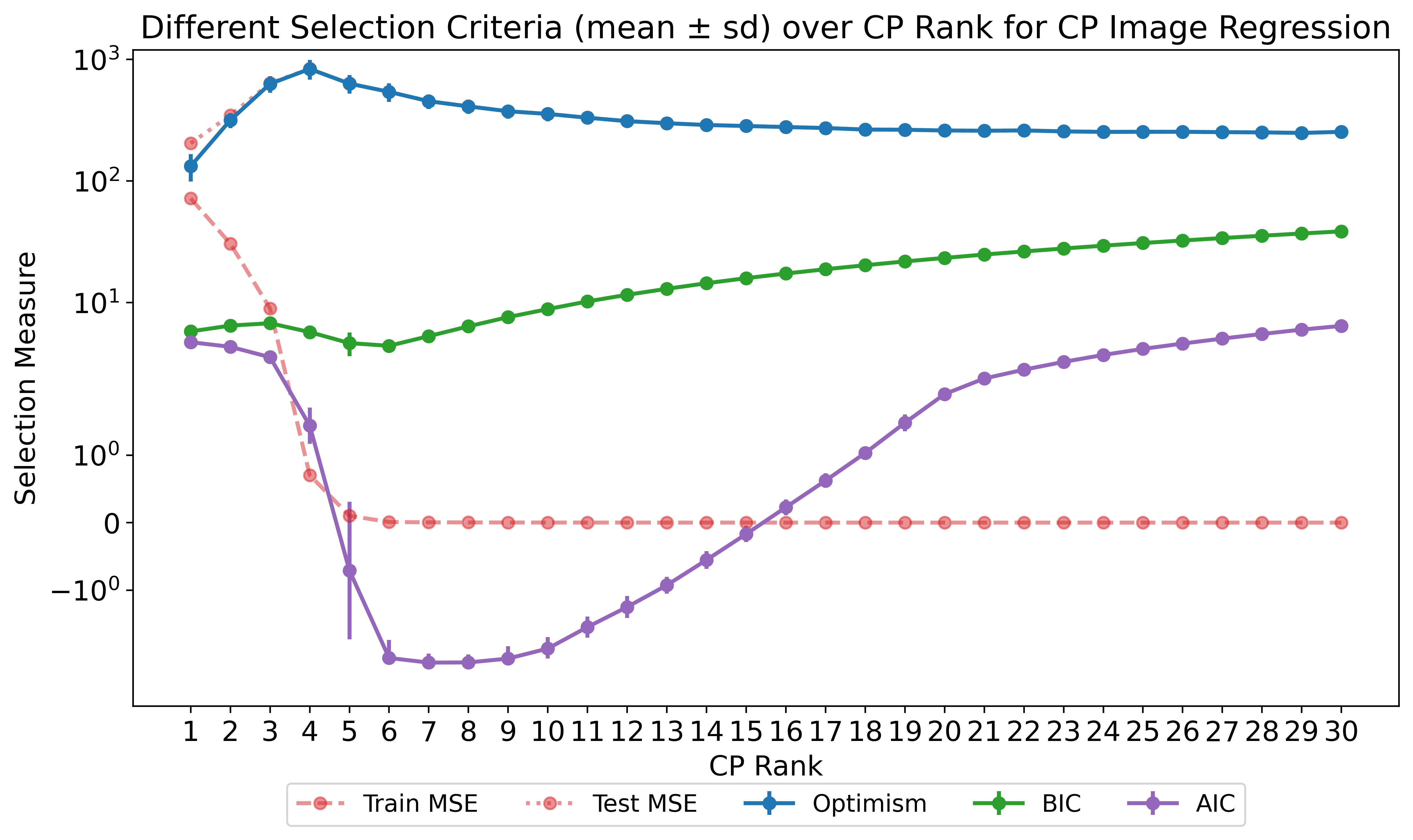}
    \end{minipage}%
     \hspace{0.02\textwidth} 
    \begin{minipage}[c]{0.25\textwidth}
        \centering
        \captionof{table}{Selected Ranks}
        \label{tab:selected_ranks_tensor_cp}
        
        \resizebox{\linewidth}{!}{%
            \begin{tabular}{lcc}
                \toprule
                \textbf{Criterion} & \textbf{Rank} & \textbf{Test MSE} \\
                \midrule
                Optimism* & 29 & 249 \\
                AIC       & 7  & 542 \\
                BIC       & 6  & 454 \\
                \bottomrule
            \end{tabular}%
        } 
    \end{minipage}

    \caption{\label{fig:CP_regression_opt_aic_bic}Different Selection Criteria for a low-rank CP regression on the FGNET dataset \citet{fu2015robust} across different target CP ranks (x-axis). The model is evaluated using an 80\% training and 20\% testing split. Optimism is calculated via the hold-out algorithm ($100$ MC replicates) in \citet{luo2025optimism}. Error bars show one standard deviation. The table on the right summarizes the optimal rank selected by each criterion (the minimum value on its curve) and the corresponding Test MSE at that rank. For optimism, the global minimum at rank 1. But we report 29 as its optimal rank, which has the second-lowest value to avoid the poor model fit (high training MSE) at rank 1 as a rank-1 model fails to learn the data structure.}
\end{figure}

\subsection{Compressed Neural Network}
\label{sec:neural_net}
Tensor-based neural network compression has recently emerged as a popular technique given the dramatically increasing size and computational complexity of deep learning models \citep{yang2017deep, zvezdana2025tensor}. However, selecting the optimal compression structure (i.e., the tensor decomposition ranks) remains a challenge \citep{zvezdana2025tensor}. Classical criteria like AIC and BIC do not apply to modern neural networks and large language models, which are often massively overparameterized, non-identifiable, and trained with complex regularization \citep{shao1997asymptotic}. 
In the following examples, we examine the performance of optimism in such tasks. We showcase the potential of using the optimism framework to guide tensor-based compression, which offers an optimal balance between network complexity and out-of-sample performance under the real training pipeline.

A straightforward application is the compression of Convolution Neural Networks (CNN), where the convolutional layers are naturally formulated as 4-mode tensors (i.e., input channels $\times$ output channels $\times$ kernel height $\times$ kernel width). Using the same FGNET dataset and settings from Section~\ref{sec:cp_regresion_FGNET}, we fit a 3-layer CNN to evaluate the optimism. The first two layers are convolutional layers with kernel dimensions of $3 \times 32 \times 3 \times 3$ and $32 \times 64 \times 3 \times 3$, respectively. These are followed by a fully connected layer with 128 hidden nodes. We compressed the two convolutional layers using CP decomposition with varying ranks, and the results are illustrated in Figure \ref{fig:CP_CNN_opt_aic_bic}. Similar to the conclusion in Figure \ref{fig:CP_regression_opt_aic_bic}, both AIC and BIC select ranks resulting in a higher test MSE compared to the rank chosen by optimism. And at these suboptimal ranks, the models have not yet achieved their full generalization capacity, as test error is still decreasing. This again reinforces the advantage of optimism as a more reliable criterion, demonstrating its capability to guide tensor-based compression by appropriately accounting for out-of-sample prediction error.

\begin{figure}[H]
    \begin{minipage}[c]{0.7\textwidth}
        \includegraphics[width=\linewidth]{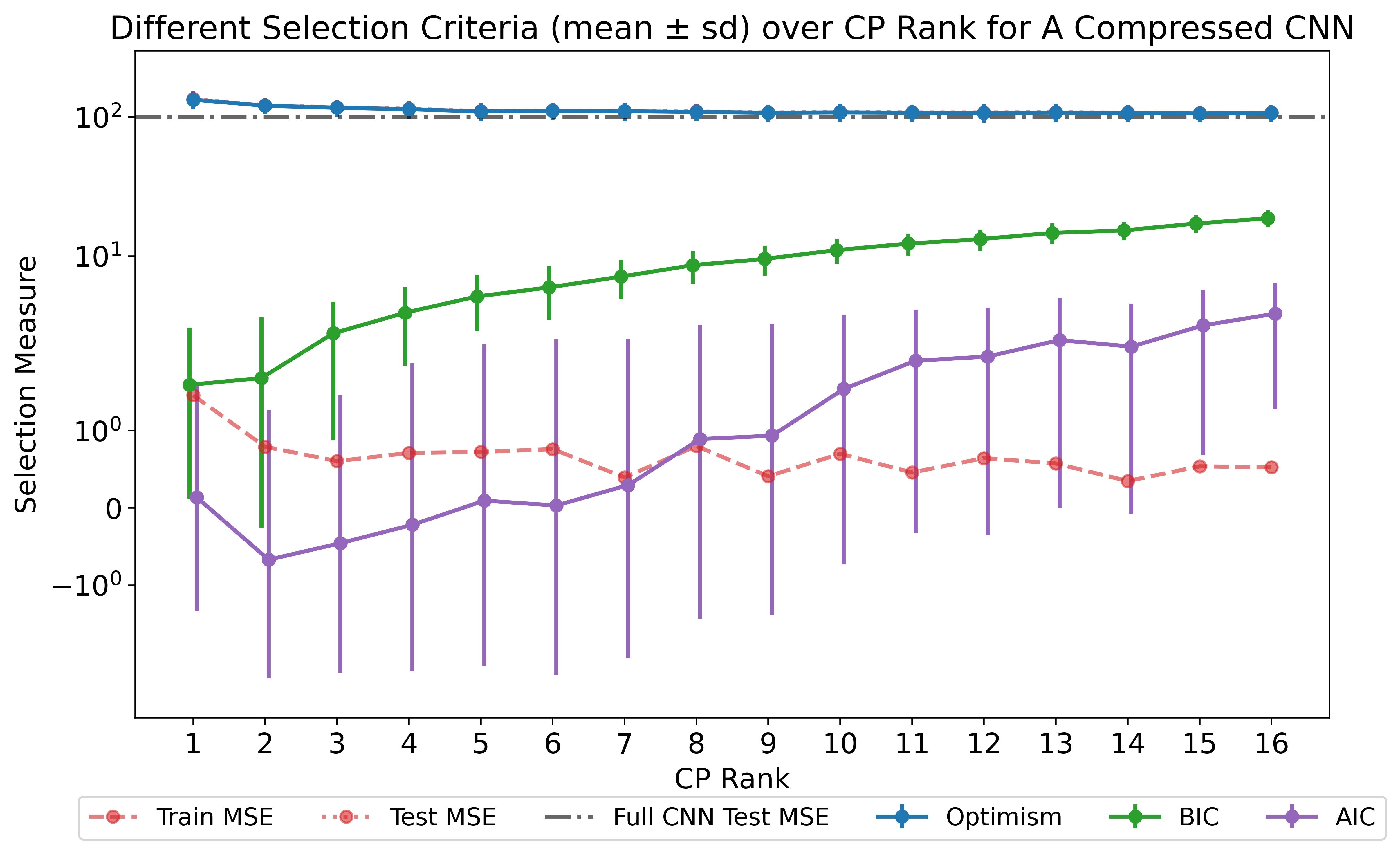}
    \end{minipage}%
     \hspace{0.02\textwidth} 
    \begin{minipage}[c]{0.25\textwidth}
        \centering
        \captionof{table}{Selected Ranks}
        \label{tab:selected_ranks_CNN}
        
        \resizebox{\linewidth}{!}{%
            \begin{tabular}{lcc}
                \toprule
                \textbf{Criterion} & \textbf{Rank} & \textbf{Test MSE} \\
                \midrule
                Optimism & 15 & 106 \\
                AIC       & 2  & 122 \\
                BIC       & 1  & 135 \\
                \bottomrule
            \end{tabular}%
        } 
    \end{minipage}

    \caption{\label{fig:CP_CNN_opt_aic_bic}Different Selection Criteria for a three-layer CNN fitted on the FGNET dataset \citet{fu2015robust}, with two convolutional layers compressed by CP decomposition of varying ranks (x-axis). The model is trained on an 80\% training and 20\% testing split using the Adam optimizer (fixed learning rate 0.001) with MSE loss and 500 epochs. Optimism is calculated via the hold-out algorithm ($100$ MC replicates) in \citet{luo2025optimism}. Error bars show one standard deviation. The table on the right summarizes the optimal rank selected by each criterion (the minimum value on its curve) and the corresponding Test MSE at that rank.}
\end{figure}

Besides CNNs, this tensor-based compression also works for Multilayer Perceptron (MLP) networks. Recent work has revealed that imposing a low-rank structure on the weights of fully-connected feedforward layers can improve, rather than just maintain, the model's predictive performance \citep{lu2024alphapruning, qing2024alphalora, sharmatruth, chendistributional}. We investigate this by constructing a 2-layer MLP with 144 hidden nodes to model the Infrared Thermography Temperature dataset, which contains 1020 samples with 30-dimensional inputs and 2-dimensional outputs \citep{wang2021infrared}. We reformulate the input layer weights ($\M{W} \in \mathbb{R}^{30 \times 144}$) as a 4-mode tensor ($\T{W} \in \mathbb{R}^{5 \times 6 \times 12 \times 12}$)  and compress it via CP decomposition with varying ranks. Note that such tensorization is purely artificial and its purpose is to enable tensor-based compression. Any reformulation that preserves the total number of parameters could be chosen. We explore alternative tensor formulation in the Supplementary Materials and obtain similar results. Figure~\ref{fig:CP_MLP_opt_aic_bic} presents the results for optimism, AIC, and BIC. In this case, all measures favor simple models, as both training and test MSEs remain consistently low across all ranks. Notably, the test MSE for the compressed network is lower than that of the full, uncompressed MLP model, reinforcing the finding that low-rank structures can offer a beneficial regularization effect in feedforward layers.

\section{Discussion}
\label{sec:conclusion}
In this work, we approached the rank selection challenge in tensor regression through an optimism framework. We argued that classical criteria such as AIC and BIC, which are based on in-sample error and ``Fixed-$\M{X}$" assumptions, are ill-suited for modern predictive tasks. To overcome these limitations, we studied the principle of optimism under a ``Random-$\T{X}$" design for tensor regression. By leveraging an equivalence between low-rank tensor regression and kernel ridge regression with a multi-linear kernel \citep{yu2018tensor}, we derived closed-form expressions for the expected optimism under both CP and Tucker decompositions. Our primary contribution is the demonstration that this optimism measure is minimized at the true underlying tensor rank, establishing it as a valid and theoretically-grounded criterion for tensor model selection. These findings were further extended to ensemble CP regression. We validated our theoretical results through numerical studies and demonstrated optimism's utility on a real-world application of image-based age regression. We further showcase its practical utility for tensor-based neural network compression \citep{zvezdana2025tensor}, where our optimism-based criterion proved more reliable and effective at balancing model fit and predictive performance than traditional metrics, highlighting its potential for model selection in deep learning.

\begin{figure}[H]
    \begin{minipage}[c]{0.7\textwidth}
        \includegraphics[width=\linewidth]{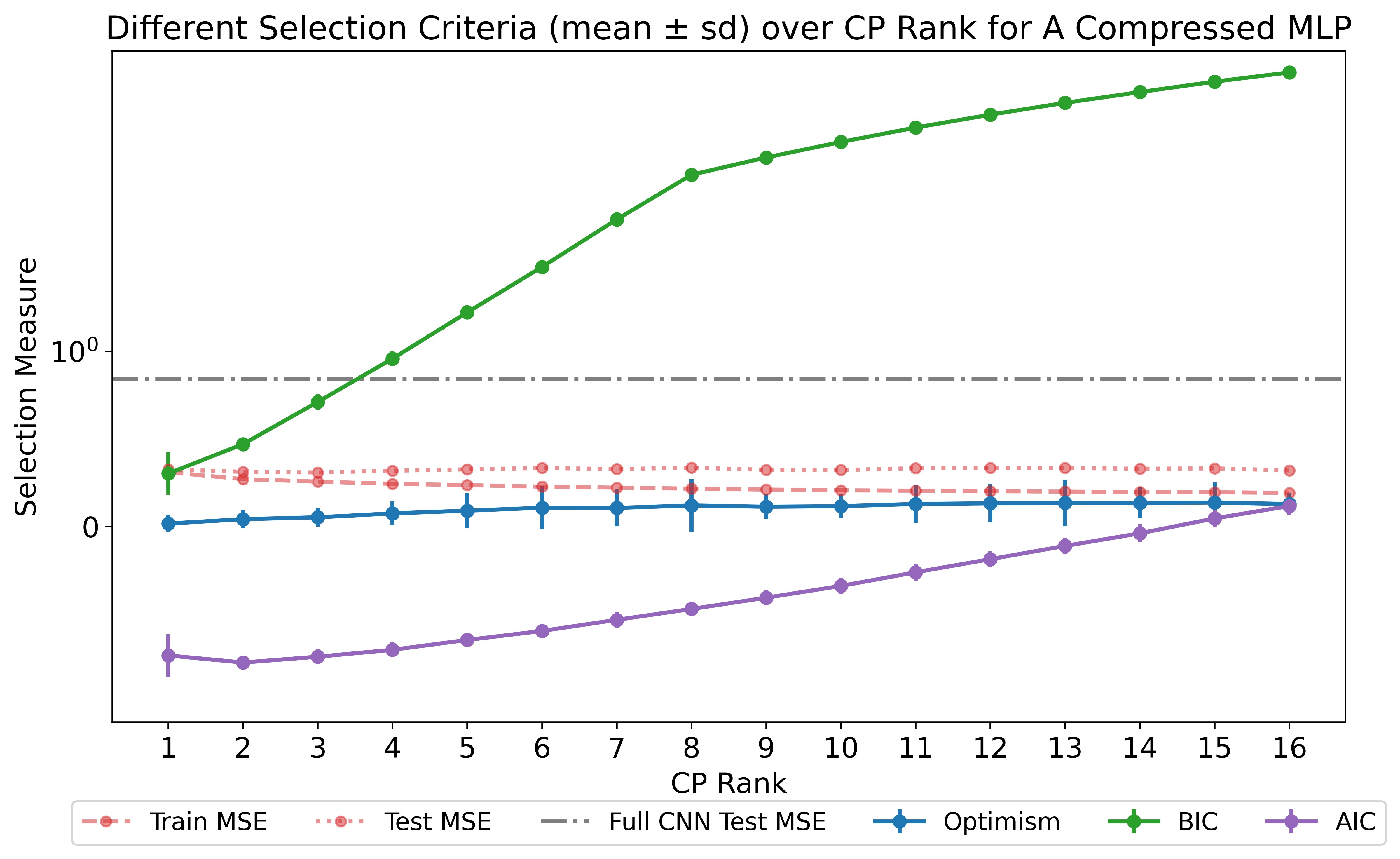}
    \end{minipage}%
     \hspace{0.02\textwidth} 
    \begin{minipage}[c]{0.25\textwidth}
        \centering
        \captionof{table}{Selected Ranks}
        \label{tab:selected_ranks_MLP}
        
        \resizebox{\linewidth}{!}{%
            \begin{tabular}{lcc}
                \toprule
                \textbf{Criterion} & \textbf{Rank} & \textbf{Test MSE} \\
                \midrule
                Optimism & 1 & 0.3 \\
                AIC       & 2  & 0.3 \\
                BIC       & 1 & 0.3 \\
                \bottomrule
            \end{tabular}%
        } 
    \end{minipage}

    \caption{\label{fig:CP_MLP_opt_aic_bic}Different Selection Criteria for a two-layer MLP with a CP-decomposed input layer, fitted on the Infrared Thermography Temperature dataset \citet{wang2021infrared} with 1020 samples, 30-dimensional inputs, and 2-dimensional outputs. The input layer is compressed using CP decompositions with varying ranks (x-axis). The model is trained on an 80\% training and 20\% testing split using Adam optimizer (fixed learning rate 0.001) with MSE loss and 200 epochs. Optimism is calculated via the hold-out algorithm ($100$ MC replicates) in \citet{luo2025optimism}. Error bars show one standard deviation. The table on the right summarizes the optimal rank selected by each criterion (the minimum value on its curve) and the corresponding Test MSE at that rank.}
\end{figure}

Our work opens several directions for future investigation. A natural next step is to extend this analysis beyond CP and Tucker to more complex structures, such as Tensor-Train decomposition, which presents its own hierarchical rank-selection challenges \citep{basu2015regularized,si2022efficient}. Furthermore, adapting the optimism framework to tree-based boosting models \citep{rashmi2015dart, linero2018bayesian, friedberg2020local} particularly for boosting ensembles \citep{buhlmann2007boosting, lv2014model} in contrast to the bagging-style ensembles discussed here, remains a significant open problem \citep{hill2020bayesian}. Our preliminary results on neural network compression also suggest a promising connection to the ``benign overfitting" literature \citep{bartlett2020benign, cao2022benign}. An investigation into how optimism characterizes generalization in over-parameterized, non-linear models like tensor regression network \citep{kossaifi2020tensor}, deeply supervised net \citep{lee2015deeply}, and Bayesian tensorized neural network \citep{hawkins2021bayesian}, where classical parameter-counting fails, would be a valuable contribution.





\bibliographystyle{apalike}
\bibliography{ref}

\clearpage 

\appendix 

\section*{Supplementary Material for ``Asymptotic Optimism for Tensor Regression Models with Applications to Neural Network Compression"}




\section{Proofs and Technical Details}
This supplement contains cross-references to equations that appear in the main manuscript. Any reference colored in purple refers to the corresponding equation numbers in the main text. 

We first restate the original KRR optimism results (Theorem 12) from \citet{luo2025optimism} as it provides a foundation for our tensor results. Let $\V{\eta}_{\phi} = \E_{\V{x}_*}(\phi(\V{x}_{*})y_{*})$, $\M{\Sigma}_{\phi} = \E_{\V{x}_*}(\phi(\V{x}_{*})\phi(\V{x}_{*})\Tra+\lambda\M{I})$, and $\phi$ be the feature map associated with the kernel function. We assume that
\begin{equation*}
    \lVert\hat{\V{\eta}}_{\phi}-\V{\eta}_{\phi}\rVert_{2}=\mathcal{O}_{p}(\frac{1}{\sqrt{n}}),\quad\lVert\hat{\M{\Sigma}}_{\phi,\lambda}-\M{\Sigma}_{\phi,\lambda}\rVert_{2}=\mathcal{O}_{p}(\frac{1}{\sqrt{n}})
\end{equation*}
where $\V{\eta}_{\phi}=\E_{\T{X}_{*}}[\phi(\T{X}_{*})y_*]$ and
$\M{\Sigma}_{\phi,\lambda}=\M{\Sigma}_{\phi} + \lambda\M{I}_{R_t} = \E_{\T{X}_{*}}[\phi(\T{X}_{*})\phi(\T{X}_{*})\Tra]+\lambda\M{I}_{R_t}$. Notice that the above assumptions are the vector-version of the first portion of Assumption~\ref{Assumption:assumption_cp}. Then under these assumptions, the expected random optimism for KRR is given by:
\begin{equation}
    \mathrm{OptR}_{\M{X}}\! = \frac{2}{n}\E_{\M{X}}\left[\left\Vert \left(\M{\Sigma}_{\phi}^{1/2}\M{\Sigma}_{\phi,\lambda}^{-1}\right)\!\left[\phi(\V{x}_{*})y_{*}-\left(\phi(\V{x}_{*})\phi(\V{x}_{*})\Tra+\lambda\M{I}\right)\M{\Sigma}_{\phi,\lambda}^{-1}\V{\eta}_{\phi}\right]\right\Vert _{2}^{2}\right] + \mathcal{O}_{p}\left(\frac{1}{n^{3/2}}\right)
    \label{eq:optimism_KRR_random_design_restate}
\end{equation}

\subsection{Proof of Lemma \ref{lemma:lemma31}}

\label{sec:Proof-of-Lemma3_1}

We prove it by contradiction. Suppose $\{\V{v}_{r}\}_{r=1}^{R}$
are linearly dependent, then $\exists\,a_{1},\dots,a_{R}\neq0$ such
that 
\[
\sum_{r=1}^{R}a_{r}\V{v}_{r}=\V{0}
\]
Without loss of generality, we assume $a_{R}\neq0$, then the above
equality implies that 
\[
\V{v}_{R}=-\sum_{r=1}^{R-1}\frac{a_{r}}{a_{R}}\V{v}_{r}
\]
But since we know $\T{B}$ admits a rank R decomposition in
\eqref{eq:CP_decomposition} and unfolding it gives us $\vec(\T{B}) = \sum_{r=1}^{R}\V{\beta}_{1}^{(r)}\Kron\cdots\Kron\V{\beta}_{m}^{(r)}$,
we will reach the following equality 
\begin{align*}
\vec(\T{B}) & =\sum_{r=1}^{R}\V{v}_{r}\\
 & =\sum_{r=1}^{R-1}\V{v}_{r}-\sum_{r=1}^{R-1}\frac{a_{r}}{a_{R}}\V{v}_{r}\\
 & =\sum_{r=1}^{R-1}(1-\frac{a_{r}}{a_{R}})\V{v}_{r}
\end{align*}
which indicates that $\T{B}$ can be written as sum of $R-1$
terms 
\[
\T{B}=\sum_{r=1}^{R-1}(1-\frac{a_{r}}{a_{R}})\V{\beta}_{1}^{(r)}\circ\cdots\circ\V{\beta}_{M}^{(r)}
\]
which contradicts that $\T{B}$ is a rank R tensor (i.e., rank (CP) $R$ tensor cannot be represented using less than $R$ rank-1 terms \citep{kolda2009tensor}). Hence, by
contradiction, $\{\V{v}_{r}\}_{r=1}^{R}$ are linearly independent.
\qed

\subsection{Proof of Lemma \ref{lemma:lemma32}}
\label{sec:Proof-of-lemma3_2}
By Lemma \ref{lemma:lemma31}, write 
\begin{equation*}
    \V{v}_{r} = \V{\beta}_{M}^{(r)}\Kron\cdots\Kron\V{\beta}_{1}^{(r)} \in \Real^{\prod_m I_m}
\end{equation*}
for $r = 1, \dots, R$. Then for each entry in $\phi(\T{X})$, we have
\begin{equation*}
    \phi_{r}(\T{X})=\langle\V{v}_{r},\vec(\T{X})\rangle
\end{equation*}
Given $\vec(\T{X}_i) \sim \mathrm{N}(0, \M{I}_{\prod_m I_m})$, the $(r,s)$ entry of $\M{\Sigma}_\phi$ is
\begin{equation*}
    \M{\Sigma}_{\phi, (r,s)}
= \E\bigl[ \phi_{r}(\T{X}) \, \phi_{s}(\T{X})\Tra \bigr]
= \V{v}_{r}\Tra \E\bigl[ \vec(\T{X}) \, \vec(\T{X})\Tra \bigr] \V{v}_{s} = \V{v}_{r}\Tra \V{v}_{s}
\end{equation*}
And thus
\begin{equation*}
    \M{\Sigma}_\phi = \M{P}\Tra \M{P}
\end{equation*}
where $\M{P} = (\V{v}_{1}, \dots, \V{v}_{R})\Tra \in \Real^{\prod_m I_m \times R}$. This immediately implies that $\M{\Sigma}_\phi$ is symmetric positive semidefinite. And according to results in Lemma \ref{lemma:lemma31}, $\V{v}_{1}, \dots, \V{v}_{R}$ are linearly independent in $\Real^{\prod_m I_m}$. Hence, $\M{P}$ has full column rank, indicating that $\M{\Sigma}_\phi$ is also positive definite.


\qed
\subsection{Proof of Theorem \ref{thm:thm_CP_true_rank}}

\label{sec:Proof-of-Theorem3_1}

From \eqref{eq:optimism_KRR_random_design_restate} and by Assumption~\ref{Assumption:assumption_cp}, the general form
of expected optimism in tensor KRR is: 
\begin{equation}
\mathrm{OptR}_{\T{X}} =\frac{2}{n}\E_{\T{X}_{*}}\left[\left\Vert \M{\Sigma}_{\phi}^{\frac{1}{2}}\M{\Sigma}_{\phi,\lambda}^{-1}\left[\phi(\T{X}_{*})y_{*}-\left(\phi(\T{X}_{*})\phi(\T{X}_{*})\Tra+\lambda\M{I}_{R}\right)\M{\Sigma}_{\phi,\lambda}^{-1}\V{\eta}_{\phi}\right]\right\Vert_2 ^{2}\right] +\mathcal{O}_{p}\left(n^{-3/2}\right)\label{eq:KRR_optimism}
\end{equation}
Since we know the specific form of kernel (with rank $R_{t}=R$) in
this oracle case, we proceed by plugging the expression of $\V{\eta}_{\phi}$,
$\M{\Sigma}_{\phi}$ and $\M{\Sigma}_{\phi,\lambda}$ in
(\ref{eq:KRR_optimism}) and work through some algebra. First notice
that under the true rank $R$, the response $y_{*}$ can be expressed
as 
\begin{align*}
y_{*}(\T{X}_{*})=\sum_{r=1}^{R}\phi_{r}(\T{X}_{*})+\epsilon_{*}=\phi(\T{X}_{*})\Tra\mathbf{1}_{R}+\epsilon_{*}
\end{align*}
where $\mathbf{1}_{R}=(1,\dots,1)\Tra\in\Real^{R\times1}$.
Then we have a key equality (by noticing that $\V{\epsilon}\perp\!\!\!\perp\vec(\T{X}_{*})$
and $\E\epsilon_{*}=0$) 
\begin{align}
\V{\eta}_{\phi} & =\E_{\T{X}_{*}}\phi(\T{X}_{*})y_{*}\nonumber \\
 & =\E_{\T{X}_{*},\epsilon_{*}}\phi(\T{X}_{*})(\phi(\T{X}_{*})\Tra\mathbf{1}_{R}+\epsilon_{*})\nonumber \\
 & =\E_{\T{X}_{*}}\phi(\T{X}_{*})\phi(\T{X}_{*})\Tra\mathbf{1}_{R}\nonumber \\
 & =\M{\Sigma}_{\phi}\mathbf{1}_{R}\label{eq:eta_form}
\end{align}
Then for the term 
\[
\E_{\T{X}_{*}}\left[\left\Vert \M{\Sigma}_{\phi}^{\frac{1}{2}}\M{\Sigma}_{\phi,\lambda}^{-1}\left[\phi(\T{X}_{*})y_{*}-\left(\phi(\T{X}_{*})\phi(\T{X}_{*})\Tra+\lambda\M{I}_{R}\right)\M{\Sigma}_{\phi,\lambda}^{-1}\V{\eta}_{\phi}\right]\right\Vert_2 ^{2}\right] 
\]
denote 
\[
\M{M}= \M{\Sigma}_{\phi}^{\frac{1}{2}}\M{\Sigma}_{\phi,\lambda}^{-1}
=
\mathbf{U}\,\text{diag}(\frac{v_{1}^{\frac{1}{2}}}{v_{1}+\lambda},\dots,\frac{v_{R}^{\frac{1}{2}}}{v_{R}+\lambda})\,\mathbf{U}\Tra
\]
where $\M{M}\in\Real^{R\times R}$ is symmetric. Expanding
the 2-norm square produces (using (\ref{eq:eta_form})) 
\begin{align}
 & \left\Vert \M{\Sigma}_{\phi}^{\frac{1}{2}}\M{\Sigma}_{\phi,\lambda}^{-1}\left[\phi(\T{X}_{*})y_{*}-\left(\phi(\T{X}_{*})\phi(\T{X}_{*})\Tra+\lambda\M{I}_{R}\right)\M{\Sigma}_{\phi,\lambda}^{-1}\V{\eta}_{\phi}\right]\right\Vert_2 ^{2}\nonumber \\
 & =\left\Vert \M{M}\left[\phi(\T{X}_{*})(\phi(\T{X}_{*})\Tra\mathbf{1}_{R}+\epsilon_{*})-\left(\phi(\T{X}_{*})\phi(\T{X}_{*})\Tra+\lambda\M{I}_{R}\right)\M{\Sigma}_{\phi,\lambda}^{-1}\M{\Sigma}_{\phi}\mathbf{1}_{R}\right]\right\Vert_2 ^{2}\nonumber \\
 & =\left\Vert \M{M}\phi(\T{X}_{*})\epsilon_{*}+\M{M}\left[\phi(\T{X}_{*})\phi(\T{X}_{*})\Tra\mathbf{1}_{R}-\left(\phi(\T{X}_{*})\phi(\T{X}_{*})\Tra+\lambda\M{I}_{R}\right)\M{\Sigma}_{\phi,\lambda}^{-1}\M{\Sigma}_{\phi}\mathbf{1}_{R}\right]\right\Vert_2 ^{2}\nonumber \\
 & =\underbrace{\left\Vert \M{M}\phi(\T{X}_{*})\epsilon_{*}\right\Vert_2 ^{2}}_{(a)}\nonumber \\
 & +\underbrace{2\epsilon_{*}\phi(\T{X}_{*})\Tra\M{M}\Tra\M{M}\left[\phi(\T{X}_{*})\phi(\T{X}_{*})\Tra\mathbf{1}_{R}-\left(\phi(\T{X}_{*})\phi(\T{X}_{*})\Tra+\lambda\M{I}_{R}\right)\M{\Sigma}_{\phi,\lambda}^{-1}\M{\Sigma}_{\phi}\mathbf{1}_{R}\right]}_{(b)}\nonumber \\
 & +\underbrace{\left\Vert \M{M}\left[\phi(\T{X}_{*})\phi(\T{X}_{*})\Tra\mathbf{1}_{R}-\left(\phi(\T{X}_{*})\phi(\T{X}_{*})\Tra+\lambda\M{I}_{R}\right)\M{\Sigma}_{\phi,\lambda}^{-1}\M{\Sigma}_{\phi}\mathbf{1}_{R}\right]\right\Vert_2 ^{2}}_{(c)}\nonumber \\
\label{eq:1_2_3_decomposition}
\end{align}
Now for term (a) in \eqref{eq:1_2_3_decomposition}, by assumptions
$\vec(\T{X}_{*})\sim \mathrm{N}(\V{0},\M{I}_{\prod_{m}I_{m}})$,
$\V{\epsilon}\sim \mathrm{N}(\mathbf{0},\sigma^{2}\M{I}_{n})$ and
$\V{\epsilon}\perp\!\!\!\perp\vec(\T{X}_{*})$, we obtain 
\begin{align}
\E_{\T{X}_{*},\epsilon_{*}}[(a)] & =\E_{\T{X}_{*},\epsilon_{*}}[\epsilon_{*}^{2}\left\Vert \M{M}\phi(\T{X}_{*})\right\Vert_2 ^{2}]\nonumber \\
 & =\sigma^{2}\E_{\T{X}_{*}}[\left\Vert \M{M}\phi(\T{X}_{*})\right\Vert_2 ^{2}]\nonumber \\
 & =\sigma^{2}[\tr(\M{M}\Tra\M{M}\,\text{Var}(\phi(\T{X}_{*})))+\E_{\T{X}_{*}}\phi(\T{X}_{*})\M{M}\Tra\M{M}\E_{\T{X}_{*}}\phi(\T{X}_{*})]\nonumber \\
 & =\sigma^{2}\tr(\M{M}^{2}\M{\Sigma}_{\phi})\nonumber \\
 & =\sigma^{2}\sum_{r=1}^{R}\frac{v_{r}^2}{(v_{r}+\lambda)^{2}}\label{eq:term_1}
\end{align}
where the second to third equality is by covariance decomposition
formula and $\E_{\T{X}_{*}}\phi(\T{X}_{*})=\mathbf{0}$
yields the third to the fourth equality. For term (b) in \eqref{eq:1_2_3_decomposition},
the independence between $\epsilon_{*}$ and $\vec(\T{X}_{i})$
and $\E\epsilon_{*}=0$ produces that 
\begin{align}
\E_{\T{X}_{*}}[(b)] & =0\label{eq:term_2}
\end{align}
Finally for term (c) in \eqref{eq:1_2_3_decomposition}, we first
notice that 
\begin{align}
 & \E_{\T{X}_{*}}\left[\phi(\T{X}_{*})\phi(\T{X}_{*})\Tra\mathbf{1}_{R}-\left(\phi(\T{X}_{*})\phi(\T{X}_{*})\Tra+\lambda\M{I}_{R}\right)\M{\Sigma}_{\phi,\lambda}^{-1}\M{\Sigma}_{\phi}\mathbf{1}_{R}\right]\nonumber \\
 & =\E_{\T{X}_{*}}\left[\phi(\T{X}_{*})\phi(\T{X}_{*})\Tra\mathbf{1}_{R}\right]-\E_{\T{X}_{*}}\left[\left(\phi(\T{X}_{*})\phi(\T{X}_{*})\Tra+\lambda\M{I}_{R}\right)\M{\Sigma}_{\phi,\lambda}^{-1}\M{\Sigma}_{\phi}\mathbf{1}_{R}\right]\nonumber \\
 & =\M{\Sigma}_{\phi}\mathbf{1}_{R}-\M{\Sigma}_{\phi,\lambda}\M{\Sigma}_{\phi,\lambda}^{-1}\M{\Sigma}_{\phi}\mathbf{1}_{R}=\mathbf{0}\label{eq:term_3_mean}
\end{align}
and the covariance structure can be expanded and simplified as follows
\begin{align}
 & \text{Var}\left(\phi(\T{X}_{*})\phi(\T{X}_{*})\Tra\mathbf{1}_{R}-\left(\phi(\T{X}_{*})\phi(\T{X}_{*})\Tra+\lambda\M{I}_{R}\right)\M{\Sigma}_{\phi,\lambda}^{-1}\M{\Sigma}_{\phi}\mathbf{1}_{R}\right)\nonumber \\
 & =\text{Var}\left(\phi(\T{X}_{*})\phi(\T{X}_{*})\Tra\mathbf{1}_{R}-\phi(\T{X}_{*})\phi(\T{X}_{*})\Tra\M{\Sigma}_{\phi,\lambda}^{-1}\M{\Sigma}_{\phi}\mathbf{1}_{R}\right)\nonumber \\
 & =\text{Var}\left(\phi(\T{X}_{*})\phi(\T{X}_{*})\Tra\underbrace{(\mathbf{1}_{R}-\M{\Sigma}_{\phi,\lambda}^{-1}\M{\Sigma}_{\phi}\mathbf{1}_{R})}_{\V{a}\in\Real^{R\times1}}\right)\nonumber \\
 & =\E_{\T{X}_{*}}\left[(\phi(\T{X}_{*})\phi(\T{X}_{*})\Tra\V{a})(\phi(\T{X}_{*})\phi(\T{X}_{*})\Tra\V{a})\Tra\right]-\E_{\T{X}_{*}}\left[\phi(\T{X}_{*})\phi(\T{X}_{*})\Tra\V{a}\right]\E_{\T{X}_{*}}\left[\phi(\T{X}_{*})\phi(\T{X}_{*})\Tra\V{a}\right]\Tra\nonumber \\
 & =\E_{\T{X}_{*}}\left[\phi(\T{X}_{*})\phi(\T{X}_{*})\Tra\V{a}\V{a}\Tra\phi(\T{X}_{*})\phi(\T{X}_{*})\Tra\right]-\E_{\T{X}_{*}}\left[\phi(\T{X}_{*})\phi(\T{X}_{*})\Tra\right]\V{a}\V{a}\Tra\E_{\T{X}_{*}}\left[\phi(\T{X}_{*})\phi(\T{X}_{*})\Tra\right]\Tra\nonumber \\
 & =\E_{\T{X}_{*}}\left[(\phi(\T{X}_{*})\Tra\V{a})^{2}\phi(\T{X}_{*})\phi(\T{X}_{*})\Tra\right]-\M{\Sigma}_{\phi}\V{a}\V{a}\Tra\M{\Sigma}_{\phi}\nonumber \\
 & =\E_{\T{X}_{*}}\left[\phi(\T{X}_{*})\Tra\V{a}\V{a}\Tra\phi(\T{X}_{*})\;\phi(\T{X}_{*})\phi(\T{X}_{*})\Tra\right]-\M{\Sigma}_{\phi}\V{a}\V{a}\Tra\M{\Sigma}_{\phi}\nonumber \\
 & =2\M{\Sigma}_{\phi}\V{a}\V{a}\Tra\M{\Sigma}_{\phi}+(\V{a}\Tra\M{\Sigma}_{\phi}\V{a})\M{\Sigma}_{\phi}-\M{\Sigma}_{\phi}\V{a}\V{a}\Tra\M{\Sigma}_{\phi}\nonumber \\
 & =\M{\Sigma}_{\phi}\V{a}\V{a}\Tra\M{\Sigma}_{\phi}+(\V{a}\Tra\M{\Sigma}_{\phi}\V{a})\M{\Sigma}_{\phi}\label{eq:term_3_cov}
\end{align}
where the equality of 
\[
\E_{\T{X}_{*}}\left[\phi(\T{X}_{*})\Tra\V{a}\V{a}\Tra\phi(\T{X}_{*})\;\phi(\T{X}_{*})\phi(\T{X}_{*})\Tra\right]=2\M{\Sigma}_{\phi}\V{a}\V{a}\Tra\M{\Sigma}_{\phi}+(\V{a}\Tra\M{\Sigma}_{\phi}\V{a})\M{\Sigma}_{\phi}
\]
is given by the Isserlis' theorem. Then using (\ref{eq:term_3_mean})
and (\ref{eq:term_3_cov}), and the covariance decomposition trick,
the term (c) can be expressed as 
\begin{align}
 & \E_{\T{X}_{*}}[(c)]\nonumber \\
 & =\tr\left(\M{M}\Tra\M{M}\text{Var}\left(\phi(\T{X}_{*})\phi(\T{X}_{*})\Tra\mathbf{1}_{R}-\left(\phi(\T{X}_{*})\phi(\T{X}_{*})\Tra+\lambda\M{I}_{R}\right)\M{\Sigma}_{\phi,\lambda}^{-1}\M{\Sigma}_{\phi}\mathbf{1}_{R}\right)\right)+0\nonumber \\
 & =\tr\left(\M{M}\Tra\M{M}(\M{\Sigma}_{\phi}\V{a}\V{a}\Tra\M{\Sigma}_{\phi}+(\V{a}\Tra\M{\Sigma}_{\phi}\V{a})\M{\Sigma}_{\phi})\right)\nonumber \\
 & =\tr\left(\M{M}^{2}\M{\Sigma}_{\phi}\V{a}\V{a}\Tra\M{\Sigma}_{\phi}\right)+(\V{a}\Tra\M{\Sigma}_{\phi}\V{a})\tr\left(\M{M}^{2}\M{\Sigma}_{\phi}\right)\nonumber \\
 & =\tr\left(\V{a}\Tra\M{\Sigma}_{\phi}\M{M}^{2}\M{\Sigma}_{\phi}\V{a}\right)+(\V{a}\Tra\M{\Sigma}_{\phi}\V{a})\tr\left(\M{M}^{2}\M{\Sigma}_{\phi}\right)\nonumber
\end{align}
Note we can write the vector $\V{a}$ as
\begin{equation*}
    \V{a} = \mathbf{1}_{R}-\M{\Sigma}_{\phi,\lambda}^{-1}\M{\Sigma}_{\phi}\mathbf{1}_{R} = (\mathbf{I} - \M{\Sigma}_{\phi,\lambda}^{-1}\M{\Sigma}_{\phi})\mathbf{1}_{R}
\end{equation*}
and using the eigen-decomposition, we can have:
\begin{align}
 & \E_{\T{X}_{*}}[(c)]\nonumber \\
 & = \tr\left(\V{a}\Tra\M{\Sigma}_{\phi}\M{M}^{2}\M{\Sigma}_{\phi}\V{a}\right)+(\V{a}\Tra\M{\Sigma}_{\phi}\V{a})\tr\left(\M{M}^{2}\M{\Sigma}_{\phi}\right)\nonumber\\
 & = \tr\left(\V{1}_R\Tra (\mathbf{I} - \M{\Sigma}_{\phi,\lambda}^{-1}\M{\Sigma}_{\phi})\M{\Sigma}_{\phi}\M{M}^{2}\M{\Sigma}_{\phi}(\mathbf{I} - \M{\Sigma}_{\phi,\lambda}^{-1}\M{\Sigma}_{\phi})\V{1}_R\right)+(\V{1}_R\Tra (\mathbf{I} - \M{\Sigma}_{\phi,\lambda}^{-1}\M{\Sigma}_{\phi})\M{\Sigma}_{\phi}(\mathbf{I} - \M{\Sigma}_{\phi,\lambda}^{-1}\M{\Sigma}_{\phi})\V{1}_R)\tr\left(\M{M}^{2}\M{\Sigma}_{\phi}\right)\nonumber\\
 & \asymp \lambda_{\max}\left( (\mathbf{I} - \M{\Sigma}_{\phi,\lambda}^{-1}\M{\Sigma}_{\phi})\M{\Sigma}_{\phi}\M{M}^{2}\M{\Sigma}_{\phi}(\mathbf{I} - \M{\Sigma}_{\phi,\lambda}^{-1}\M{\Sigma}_{\phi})\right)
 +\lambda_{\max}\left((\mathbf{I} - \M{\Sigma}_{\phi,\lambda}^{-1}\M{\Sigma}_{\phi})\M{\Sigma}_{\phi}(\mathbf{I} - \M{\Sigma}_{\phi,\lambda}^{-1}\M{\Sigma}_{\phi})\right)\sum_{r=1}^{R}\frac{v_{r}^2}{(v_{r}+\lambda)^{2}}\nonumber \\
 & =\frac{\lambda^2v_1}{(v_1+\lambda)^2}\frac{v_{1}^{2}}{(v_{1}+\lambda)^{2}}+\frac{\lambda^2v_1}{(v_1+\lambda)^2}\sum_{r=1}^{R}\frac{v_{r}^2}{(v_{r}+\lambda)^{2}}\nonumber \\
 & \asymp \frac{\lambda^2v_1}{(v_1+\lambda)^2}\sum_{r=1}^{R}\frac{v_{r}^2}{(v_{r}+\lambda)^{2}}\label{eq:term_3}
\end{align}
Combining these three terms together in (\ref{eq:term_1}), (\ref{eq:term_2}),
and (\ref{eq:term_3}) yields 
\begin{align}
 & \frac{2}{n}\E_{\T{X}_{*}}\left[\left\Vert \M{\Sigma}_{\phi}^{\frac{1}{2}}\M{\Sigma}_{\phi,\lambda}^{-1}\left[\phi(\T{X}_{*})y_{*}-\left(\phi(\T{X}_{*})\phi(\T{X}_{*})\Tra+\lambda\M{I}_{R}\right)\M{\Sigma}_{\phi,\lambda}^{-1}\V{\eta}_{\phi}\right]\right\Vert_2 ^{2}\right]\nonumber \\
 & =\frac{2}{n}((a)+(b)+(c))\nonumber \\
 & =\frac{2\sigma^{2}}{n}\sum_{r=1}^{R}\frac{v_{r}^2}{(v_{r}+\lambda)^{2}}+0+ \frac{2\lambda^2v_1}{n(v_1+\lambda)^2}\sum_{r=1}^{R}\frac{v_{r}^2}{(v_{r}+\lambda)^{2}}\nonumber \\
 & =\frac{2\left(\sigma^{2}+\frac{\lambda^2v_1}{(v_1+\lambda)^2}\right)}{n}\sum_{r=1}^{R}\frac{v_{r}^2}{(v_{r}+\lambda)^{2}}\label{eq:first_term_bound}
\end{align}
which gives the expected optimism as: 
\begin{equation*}
\mathrm{OptR}_{\T{X}}^{(\mathrm{true})} =\frac{2\left(\sigma^{2}+\frac{\lambda^2v_1}{(v_1+\lambda)^2}\right)}{n}\sum_{r=1}^{R}\frac{v_{r}^2}{(v_{r}+\lambda)^{2}}+\mathcal{O}_{p}(n^{-\frac{3}{2}})
\end{equation*}
\qed

\begin{remark}
\label{remark:remark_CP_lambda}
We can further analyze the asymptotic order of optimism
in terms of the regularization parameter $\lambda$. As $\lambda \gg v_r$ for $r = 1, \dots, R$, we have 
\begin{equation*}
\frac{2\left(\sigma^{2}+\frac{\lambda^2 v_1}{(v_1+\lambda)^2}\right)}{n}\sum_{r=1}^{R}\frac{v_{r}^2}{(v_{r}+\lambda)^{2}} \leq \frac{2\left(\sigma^{2}+\frac{\lambda^2 v_1}{(v_1+\lambda)^2}\right)}{n}R\frac{v_{1}^2}{(v_{1}+\lambda)^{2}} =\mathcal{O}(\lambda^{-2}) 
\end{equation*}
Thus, $\mathrm{OptR}_{\T{X}}^{(\mathrm{true})}=\mathcal{O}(\lambda^{-2}) +\mathcal{O}_{p}(n^{-\frac{3}{2}})$.
\end{remark}

\subsection{Proof of Theorem \ref{thm:thm_CP_over_rank}}

\label{sec:Proof-of-thm_3_2}

The proof will follow the same arguments as in Theorem \ref{thm:thm_CP_true_rank}
by noting that given the minimal CP rank $R$ for a tensor $\T{B}$,
any $R_{t}>R$ decompositions of $\T{B}$ will be guaranteed
to hold with strict equality (i.e. $\T{B}=\sum_{r=1}^{R_{t}}\tilde{\V{\beta}}_{1}^{(r)}\circ\cdots\circ\tilde{\V{\beta}}_{M}^{(r)}$
instead of $\approx$). Thus, we can express $y_{*}$ as 
\begin{align*}
y_{*}(\T{X}_{*}) & =\left\langle \vec(\T{X}_{*}),\vec(\T{B})\right\rangle +\epsilon_{*}\\
 & =\sum_{r=1}^{R_{t}}\tilde{\phi}_{r}(\T{X}_{*})+\epsilon_{*}\\
 & =\phi_{R_{t}}(\T{X}_{*})\Tra\mathbf{1}_{R_{t}}+\epsilon_{*}
\end{align*}
where $\tilde{\phi}_{r}(\T{X}_{*})=\langle\tilde{\V{\beta}}_{M}^{(r)}\Kron\cdots\Kron\tilde{\V{\beta}}_{1}^{(r)},\vec(\T{X}_{*})\rangle\in\Real$.
Then the key equality (\ref{eq:eta_form}) used in section \ref{sec:Proof-of-Theorem3_1} can be
written as 
\begin{align}
\V{\eta}_{\phi_{R_{t}}} & =\E_{\T{X}_{*}}\phi_{R_{t}}(\T{X}_{*})y_{*}\nonumber \\
 & =\E_{\T{X}_{*},\epsilon_{*}}\phi_{R_{t}}(\T{X}_{*})(\phi_{R_{t}}(\T{X}_{*})\Tra\mathbf{1}_{R_{t}}+\epsilon_{*})\nonumber \\
 & =\E_{\T{X}_{*}}\phi_{R_{t}}(\T{X}_{*})\phi_{R_{t}}(\T{X}_{*})\Tra\mathbf{1}_{R_{t}}\nonumber \\
 & =\M{\Sigma}_{\phi_{R_{t}}}\mathbf{1}_{R_{t}}\label{eq:over_eta_form}
\end{align}
As a result, we can apply the identical arguments of section \ref{sec:Proof-of-Theorem3_1} in (\ref{eq:term_1}), (\ref{eq:term_2}),
and (\ref{eq:term_3}) by
replacing the use of $\V{\eta}_{\phi_{R}}=\M{\Sigma}_{\phi_{R}}\mathbf{1}_{R}$
with (\ref{eq:over_eta_form}), and eventually produces the optimism
under the over-specified rank as desired (replace every $\M{\Sigma}_{\phi}$
with $\M{\Sigma}_{\phi_{R_{t}}}$) 
\begin{align*}
\mathrm{OptR}_{\T{X}}^{(\mathrm{over})}=\frac{2\left(\sigma^{2}+\frac{\lambda^2\tilde{v}_1}{(\tilde{v}_1+\lambda)^2}\right)}{n}\sum_{r=1}^{R_{t}}\frac{\tilde{v}_{r}^2}{(\tilde{v}_{r}+\lambda)^{2}}+\mathcal{O}_{p}(n^{-\frac{3}{2}})
\end{align*}
\qed

\subsection{Proof of Theorem \ref{thm:thm_CP_under_rank}}

\label{sec:Proof-of-thm_3_3}

The proof also follows the same arguments as in \ref{sec:Proof-of-Theorem3_1} with extra error terms introduced by the approximation error from $\V{\Delta}$.
Note that since $\T{B}$ has a true CP rank of $R$, it cannot
be perfectly reconstructed using fewer than $R$ components. Therefore,
for any $R_{t}<R$, the tensor $\T{B}_{R_{t}}$ represents only
the best rank-$R_{t}$ approximation to $\T{B}$, in the sense
of minimizing the residual norm. Then given $\T{B}_{R_{t}}=\sum_{r=1}^{R_{t}}\tilde{\V{\beta}}_{1}^{(r)}\circ\cdots\circ\tilde{\V{\beta}}_{M}^{(r)}$
as a minimizer of (\ref{eq:cp_low_rank_approximation_obj}), let 
\[
\tilde{\mathbf{v}}_{r}=\tilde{\V{\beta}}_{M}^{(r)}\Kron\cdots\Kron\tilde{\V{\beta}}_{1}^{(r)}\in\Real^{\prod_{m}I_{m}\times1}
\]
be the vectorization of each component, by the first-order optimality
condition, the residual $\V{\V{\Delta}}$ will be orthogonal to $\tilde{\mathbf{v}}_{r}$
(i.e. $\langle\tilde{\mathbf{v}}_{r},\V{\V{\Delta}}\rangle=0$ for $r=1,\dots,R_{t}$),
which is analogous to the orthogonality between residual and regressors. To see this rigorously, consider we perturb a component $\tilde{\V{\beta}}_{m}^{(r)}$ for arbitrary $1 \leq r \leq R_t$ and $1 \leq m \leq M$ by 
\begin{equation*}
    \tilde{\V{\beta}}_{m}^{(r)} (\epsilon) = \tilde{\V{\beta}}_{m}^{(r)} + \epsilon \V{h}
\end{equation*}
where $\epsilon \in \Real$ and $\V{h} \in \Real^{I_m}$. Then the perturbed residual can be written as:
\begin{align*}
    \V{\Delta} (\epsilon) & = \T{B} - \T{B}_{R_t} \\
    & = \T{B} - \left( \sum_{j\neq r}^{R_{t}}\tilde{\V{\beta}}_{1}^{(j)}\circ\cdots\circ\tilde{\V{\beta}}_{M}^{(j)} + \tilde{\V{\beta}}_{1}^{(r)}\circ\cdots \circ \tilde{\V{\beta}}_{m}^{(r)}(\epsilon)\circ \cdots \circ \tilde{\V{\beta}}_{M}^{(r)} \right)\\
    & = \T{B} - \left( \sum_{j\neq r}^{R_{t}}\tilde{\V{\beta}}_{1}^{(j)}\circ\cdots\circ\tilde{\V{\beta}}_{M}^{(j)} + \tilde{\V{\beta}}_{1}^{(r)}\circ\cdots \circ (\tilde{\V{\beta}}_{m}^{(r)} + \epsilon \V{h})\circ \cdots \circ \tilde{\V{\beta}}_{M}^{(r)} \right) \\
    & = \T{B} - \left( \sum_{j=1}^{R_{t}}\tilde{\V{\beta}}_{1}^{(j)}\circ\cdots\circ\tilde{\V{\beta}}_{M}^{(j)} + \epsilon \tilde{\V{\beta}}_{1}^{(r)}\circ\cdots \circ \V{h}\circ \cdots \circ \tilde{\V{\beta}}_{M}^{(r)} \right) \\
    & = \V{\Delta} - \epsilon \T{E}
\end{align*}
where $\T{E} = \tilde{\V{\beta}}_{1}^{(r)}\circ\cdots \circ \V{h}\circ \cdots \circ \tilde{\V{\beta}}_{M}^{(r)}$. Now define $f(\epsilon) = \lVert \V{\Delta} (\epsilon) \rVert^2$, we have 
\begin{equation*}
    \frac{d}{d\epsilon} f(\epsilon) |_{\epsilon = 0} = -2 \langle\!\langle\V{\Delta},\T{E}\rangle\!\rangle
\end{equation*}
Now because $\T{B}_{R_t}$ is the minimizer, by the first-order optimality condition, the directional derivative at any direction $\V{h}$ should be zero, which gives
\begin{equation*}
    \langle\!\langle\V{\Delta},\T{E}\rangle\!\rangle = \langle\!\langle\V{\Delta},\tilde{\V{\beta}}_{1}^{(r)}\circ\cdots \circ \V{h}\circ \cdots \circ \tilde{\V{\beta}}_{M}^{(r)}\rangle\!\rangle = 0
\end{equation*}
Now just choose $\V{h} = \tilde{\V{\beta}}_m^{(r)}$ and vectorized everything yields the desired results
\begin{equation*}
    \langle\tilde{\mathbf{v}}_{r},\V{\V{\Delta}}\rangle=0.
\end{equation*}

Then in this case, we can express $y_{*}$ as 
\begin{align*}
y_{*}(\T{X}_{*}) & =\left\langle \vec(\T{X}_{*}),\vec(\T{B})\right\rangle +\epsilon_{*}\\
 & =\left\langle \vec(\T{X}_{*}),\vec(\T{B}_{R_{t}})+\V{\V{\Delta}}\right\rangle +\epsilon_{*}\\
 & =\left\langle \vec(\T{X}_{*}),\sum_{r=1}^{R_{t}}\tilde{\mathbf{v}}_{r}\right\rangle +\left\langle \vec(\T{X}_{*}),\V{\V{\Delta}}\right\rangle +\epsilon_{*}\\
 & =\sum_{r=1}^{R_{t}}\tilde{\phi}_{r}(\T{X}_{*})+g(\T{X}_{*},\V{\V{\Delta}})+\epsilon_{*}\\
 & =\phi_{R_{t}}(\T{X}_{*})\Tra\mathbf{1}_{R_{t}}+g(\T{X}_{*},\V{\V{\Delta}})+\epsilon_{*}
\end{align*}
where $\tilde{\phi}_{r}(\T{X}_{*})=\langle\tilde{\mathbf{v}}_{r},\vec(\T{X}_{*})\rangle\in\Real$
and $g(\T{X}_{*},\V{\V{\Delta}})=\left\langle \vec(\T{X}_{*}),\V{\V{\Delta}}\right\rangle$.
Then the key equality (\ref{eq:eta_form}) used in \ref{sec:Proof-of-Theorem3_1}
can be written as 
\begin{align}
\V{\eta}_{\phi_{R_{t}}} & =\E_{\T{X}_{*}}\phi_{R_{t}}(\T{X}_{*})y_{*}\nonumber \\
 & =\E_{\T{X}_{*},\epsilon_{*}}\phi_{R_{t}}(\T{X}_{*})(\phi_{R_{t}}(\T{X}_{*})\Tra\mathbf{1}_{R_{t}}+g(\T{X}_{*},\V{\V{\Delta}})+\epsilon_{*})\nonumber \\
 & =\E_{\T{X}_{*}}\phi_{R_{t}}(\T{X}_{*})\phi_{R_{t}}(\T{X}_{*})\Tra\mathbf{1}_{R_{t}}+\E_{\T{X}_{*}}\phi_{R_{t}}(\T{X}_{*})g(\T{X}_{*},\V{\V{\Delta}})\nonumber \\
 & =\M{\Sigma}_{\phi_{R_{t}}}\mathbf{1}_{R_{t}}\label{eq:under_eta_form}
\end{align}
where the third to the last equality comes from $\langle\tilde{\mathbf{v}}_{r},\V{\V{\Delta}}\rangle=0$
\begin{align*}
\E_{\T{X}_{*}}\phi_{R_{t}}(\T{X}_{*})g(\T{X}_{*},\V{\V{\Delta}}) & =\E_{\T{X}_{*}}\sum_{r=1}^{R_{t}}\tilde{\phi}_{r}(\T{X}_{*})\left\langle \V{\V{\Delta}},\vec(\T{X}_{*})\right\rangle \\
 & =\sum_{r=1}^{R_{t}}\E_{\T{X}_{*}}\langle\tilde{\mathbf{v}}_{r},\vec(\T{X}_{*})\rangle\left\langle \V{\V{\Delta}},\vec(\T{X}_{*})\right\rangle \\
 & =\sum_{r=1}^{R_{t}}\left\langle \tilde{\mathbf{v}}_{r},\V{\V{\Delta}}\right\rangle \\
 & =0
\end{align*}
Note that here since $\phi_{R_{t}}(\T{X}_{*})$ and $g(\T{X}_{*},\V{\V{\Delta}})$ are Normal vectors (with mean $\V{0}$ as they are all linear transformation from $\vec{(\T{X})}$), $\E_{\T{X}_{*}}\phi_{R_{t}}(\T{X}_{*})g(\T{X}_{*},\V{\V{\Delta}}) = 0$ implies that they are uncorrelated normal, which also implies independence. Then by replacing the use of $\V{\eta}_{\phi_{R}}=\M{\Sigma}_{\phi_{R}}\mathbf{1}_{R}$
with (\ref{eq:under_eta_form}), we can apply the identical arguments
of Theorem 1 for (\ref{eq:term_1}), (\ref{eq:term_2}),
and (\ref{eq:term_3})
in this case. And the approximation difference $\V{\V{\Delta}}$ only affects
(\ref{eq:first_term_bound}) by introducing positive error terms via
$y_{*}=\phi_{R_{t}}(\T{X}_{*})\Tra\mathbf{1}_{R_{t}}+g(\T{X}_{*},\V{\V{\Delta}})+\epsilon_{*}$.
Specifically, (\ref{eq:first_term_bound}) under rank $R_{t}<R$ will
have 3 additional terms as ($\M{M} = \M{\Sigma}_{\phi_{R_t}}^{\frac{1}{2}} \M{\Sigma}_{\phi_{R_t}, \,\lambda}^{-1}$)
\begin{align}
 & \frac{2}{n}\E_{\T{X}_{*}}\left[\left\Vert \M{M}\left[\phi_{R_{t}}(\T{X}_{*})y_{*}-\left(\phi_{R_{t}}(\T{X}_{*})\phi_{R_{t}}(\T{X}_{*})\Tra+\lambda\M{I}_{R_{t}}\right)\M{\Sigma}_{\phi_{R_{t}},\lambda}^{-1}\V{\eta}_{\phi_{R_{t}}}\right]\right\Vert_2 ^{2}\right]\nonumber \\
 & =\frac{2\left(\sigma^{2}+\frac{\lambda^2\tilde{v}_1}{(\tilde{v}_1+\lambda)^2}\right)}{n}\sum_{r=1}^{R_{t}}\frac{\tilde{v}_{r}^2}{(\tilde{v}_{r}+\lambda)^{2}}\nonumber \\
 & +\underbrace{\frac{2}{n}\E_{\T{X}_{*}}\left\Vert \M{M}\phi_{R_{t}}(\T{X}_{*})g(\T{X}_{*},\V{\V{\Delta}})\right\Vert_2 ^{2}}_{(a)}\nonumber \\
 & +\underbrace{\frac{4}{n}\E_{\T{X}_{*}}\left[g(\T{X}_{*},\V{\V{\Delta}})\phi_{R_{t}}(\T{X}_{*})\Tra\M{M}\Tra\M{M}\left[\phi_{R_{t}}(\T{X}_{*})\phi_{R_{t}}(\T{X}_{*})\Tra\mathbf{1}_{R_t}-\left(\phi_{R_{t}}(\T{X}_{*})\phi_{R_{t}}(\T{X}_{*})\Tra+\lambda\M{I}_{R_t}\right)\M{\Sigma}_{\phi_{R_{t}},\lambda}^{-1}\M{\Sigma}_{\phi_{R_{t}}}\mathbf{1}_{R_t}\right]\right]}_{(b)}\nonumber \\
 & +\underbrace{\frac{4}{n}\E_{\T{X}_{*}}\left[\epsilon_{*}\phi_{R_{t}}(\T{X}_{*})\Tra\M{M}\Tra\M{M}\phi_{R_{t}}(\T{X}_{*})g(\T{X}_{*},\V{\V{\Delta}})\right]}_{(c)}\label{eq:1_2_3_decomposition2}\\
  & =\frac{2\left(\sigma^{2}+\frac{\lambda^2\tilde{v}_1}{(\tilde{v}_1+\lambda)^2}\right)}{n}\sum_{r=1}^{R_{t}}\frac{\tilde{v}_{r}^2}{(\tilde{v}_{r}+\lambda)^{2}}+ \frac{2}{n} \lVert\V{\V{\Delta}}\rVert^{2} \tr{(\M{M}^2 \M{\Sigma}_{\phi_{R_t}})}+0+0\nonumber \\
    & =\frac{2\left(\sigma^{2}+\frac{\lambda^2\tilde{v}_1}{(\tilde{v}_1+\lambda)^2}\right)}{n}\sum_{r=1}^{R_{t}}\frac{\tilde{v}_{r}^2}{(\tilde{v}_{r}+\lambda)^{2}}+ \frac{2}{n} \lVert\V{\V{\Delta}}\rVert^{2} \sum_{r=1}^{R_{t}}\frac{\tilde{v}_{r}^2}{(\tilde{v}_{r}+\lambda)^{2}}\nonumber 
\end{align}
where the term (b) and (c) in \eqref{eq:1_2_3_decomposition2} are
$0$ due to the orthogonality of $\V{\V{\Delta}}$ and $\phi_{R_t}(\T{X}_{*})$
and the assumptions on $\epsilon_{*}$ and $\vec(\T{X}_*)$ (i.e., $\vec(\T{X}_{*})\sim \mathrm{N}(\mathbf{0},\M{I}_{\prod_{m}I_{m}})$,
$\epsilon_{*}\sim \mathrm{N}(0,\sigma^{2})$ and
$\epsilon_{*}\perp\!\!\!\perp\vec(\T{X}_{*})$).  Thus, we
eventually produce the optimism under-specified rank as 
\begin{align*}
\mathrm{OptR}_{\T{X}}^{(\mathrm{under})}=\frac{2\left(\sigma^{2}+\frac{\lambda^2\tilde{v}_1}{(\tilde{v}_1+\lambda)^2}\right)}{n}\sum_{r=1}^{R_{t}}\frac{\tilde{v}_{r}^2}{(\tilde{v}_{r}+\lambda)^{2}}+\frac{2}{n} \lVert\V{\V{\Delta}}\rVert^{2} \sum_{r=1}^{R_{t}}\frac{\tilde{v}_{r}^2}{(\tilde{v}_{r}+\lambda)^{2}}+\mathcal{O}_{p}(n^{-\frac{3}{2}})
\end{align*}
\qed

\subsection{Proof of Proposition \ref{prop:CP_proposition}}
\label{sec:proof_CP_prop3_1}
We start with the over specified rank case $R_t > R$. Under the condition $\lambda \ll v_R\,\wedge \tilde{v}_{R_t}$, the fraction terms (i.e., $\frac{v_r^2}{(v_r + \lambda)^2}$) appearing in equations \eqref{eq:tensorkrr_true_rank_CP} and \eqref{eq:tensorkrr_over_rank_CP} will roughly become
\[
\frac{v_r^2}{(v_r+\lambda)^2} \approx 1 
\]
for $r = 1,\dots, R$ and similarly for $\frac{\tilde{v}_r^2}{(\tilde{v}_r+\lambda)^2} \approx 1$ for $r = 1, \dots, R_t$. And 
\[
\frac{\lambda^2\tilde{v}_1}{(\tilde{v}_1+\lambda)^2} = \frac{\lambda^2\tilde{v}_1}{\lambda^2(1+\frac{\tilde{v}_1}{\lambda})^2} \approx \frac{1}{\tilde{v}_1}
\]
which will be negligible compared to $\sigma^2$ so long as the leading eigenvalue $\tilde{v}_1$ remains large. Consequently, the optimism with over-specified rank can be written as:
\begin{align*}
\mathrm{OptR}_{\T{X}}^{(\mathrm{over})}
& = \frac{2\left(\sigma^{2}+\frac{\lambda^2\tilde{v}_1}{(\tilde{v}_1+\lambda)^2}\right)}{n}\sum_{r=1}^{R_{t}}\frac{\tilde{v}_{r}^2}{(\tilde{v}_{r}+\lambda)^{2}}+\mathcal{O}_{p}(n^{-\frac{3}{2}})\\
& \approx \frac{2\sigma^{2}}{n}R + \frac{2\sigma^{2}}{n} \sum_{r=R+1}^{R_{t}}\frac{\tilde{v}_{r}^2}{(\tilde{v}_{r}+\lambda)^{2}} +
\mathcal{O}_{p}(n^{-\frac{3}{2}})\\
& \approx \mathrm{OptR}_{\T{X}}^{(\mathrm{true})}+ \frac{2\sigma^{2}}{n}\sum_{r=R+1}^{R_{t}}\frac{\tilde{v}_{r}^2}{(\tilde{v}_{r}+\lambda)^{2}}
\end{align*}
Hence,
if at least one $\{\tilde{v}_{r}\}_{r=R+1}^{R_{t}}$ stays positive, we will have $\frac{2\sigma^2}{n}\sum_{r=R+1}^{R_{t}}\frac{\tilde{v}_{r}^2}{(\tilde{v}_{r}+\lambda)^{2}} > 0$, leading to $\mathrm{OptR}_{\T{X}}^{(\mathrm{over})} \geq \mathrm{OptR}_{\T{X}}^{(\mathrm{true})}$.
Consequently, if $(\tilde{v}_{R+1},\dots, \tilde{v}_{R_t}) = 0$ (i.e., when $R_t = R$), the optimism $\mathrm{OptR}_{\T{X}}^{(\mathrm{over})}$
reduces to $\mathrm{OptR}_{\T{X}}^{(\mathrm{true})}$
in \eqref{eq:tensorkrr_true_rank_CP}.

We next consider the under-specified rank case. Similarly, with small enough $\lambda$, we can have the optimism with under-specified rank ($R_t < R$) as:
\[
\mathrm{OptR}_{\T{X}}^{(\mathrm{under})} \approx \frac{2\sigma^{2}}{n} R_t + \frac{2}{n}\lVert\V{\V{\Delta}}\rVert^{2}R_t
\]
and the difference with respect to the true-rank model will become:
\begin{equation*}
    \mathrm{OptR}_{\T{X}}^{(\mathrm{under})}-\mathrm{OptR}_{\T{X}}^{(\mathrm{true})}
    = 
    \frac{2\sigma^{2}}{n}(R_t - R)
    +\frac{2}{n} \lVert\V{\V{\Delta}}\rVert^{2}R_t
\end{equation*}
Then $\lVert\V{\V{\Delta}}\rVert^{2} \geq \sigma^2 \frac{R - R_t}{R_t}$ guarantees that the dominant contribution
in $\mathrm{OptR}_{\T{X}}^{(\mathrm{under})} - \mathrm{OptR}_{\T{X}}^{(\mathrm{true})}$
will come from the positive error term $\lVert\Delta\rVert^{2}$
introduced by the approximation. Consequently, we get 
\[
\mathrm{OptR}_{\T{X}}^{(\mathrm{under})} \geq \mathrm{OptR}_{\T{X}}^{(\mathrm{true})}.
\]

\qed

\subsection{Proof of Theorem \ref{thm:thm_tucker_true_rank}}

\label{sec:Proof-of-Theorem4_1}

We employ a similar proof strategy as in Section~\ref{sec:Proof-of-Theorem3_1}. Since we know the specific form of kernel (with rank $R_{m_t}=R_m$) in this oracle case, we proceed by plugging the expression of $\V{\eta}_{\varphi}$,
$\M{\Sigma}_{\varphi}$ and $\M{\Sigma}_{\varphi,\lambda}$ in
(\ref{eq:KRR_optimism}) and work through some algebra. First note
that under the true rank $R_{m_t}=R_m$, the response $y_{*}$ can be expressed as 
\begin{align*}
y_{*}(\T{X}_{*})=\bigl\langle \M{P}^{\!\top}\vec(\T{X}_*),\,\V{g}\bigr\rangle  +\epsilon_{*} =\varphi(\T{X}_{*})\Tra\V{g}+\epsilon_{*}
\end{align*}
where $\V{g}\in\Real^{R}$ is the vectorized core tensor $\T{G}$ (note here $R = \prod_m R_m$). Then we have a key equality (by noticing that $\epsilon_{*}\perp\!\!\!\perp\vec(\T{X}_{*})$
and $\E(\epsilon_{*})=0$) 
\begin{align}
\V{\eta}_{\varphi} & =\E_{\T{X}_{*}}\varphi(\T{X}_{*})y_{*}\nonumber \\
& =\E_{\T{X}_{*},\epsilon_{*}}\varphi(\T{X}_{*})(\varphi(\T{X}_{*})\Tra\V{g}+\epsilon_{*})\nonumber \\
& =\E_{\T{X}_{*}}\varphi(\T{X}_{*})\varphi(\T{X}_{*})\Tra\V{g}\nonumber \\
& =\M{\Sigma}_{\varphi}\V{g}\label{eq:eta_form_tucker}
\end{align}
Then for the term 
\[
\E_{\T{X}_{*}}\left[\left\Vert \M{\Sigma}_{\varphi}^{\frac{1}{2}}\M{\Sigma}_{\varphi,\lambda}^{-1}\left[\varphi(\T{X}_{*})y_{*}-\left(\varphi(\T{X}_{*})\varphi(\T{X}_{*})\Tra+\lambda\M{I}_{R}\right)\M{\Sigma}_{\varphi,\lambda}^{-1}\V{\eta}_{\varphi}\right]\right\Vert_2 ^{2}\right] 
\]
denote 
\[
\M{M}= \M{\Sigma}_{\varphi}^{\frac{1}{2}}\M{\Sigma}_{\varphi,\lambda}^{-1}
=
\mathbf{U}\,\text{diag}(\frac{v_{1}^{\frac{1}{2}}}{v_{1}+\lambda},\dots,\frac{v_{R}^{\frac{1}{2}}}{v_{R}+\lambda})\,\mathbf{U}\Tra
\]
where $\M{M}\in\Real^{R\times R}$ is symmetric. Expanding
the 2-norm square produces (using (\ref{eq:eta_form_tucker})) 
\begin{align}
& \left\Vert \M{\Sigma}_{\varphi}^{\frac{1}{2}}\M{\Sigma}_{\varphi,\lambda}^{-1}\left[\varphi(\T{X}_{*})y_{*}-\left(\varphi(\T{X}_{*})\varphi(\T{X}_{*})\Tra+\lambda\M{I}_{R}\right)\M{\Sigma}_{\varphi,\lambda}^{-1}\V{\eta}_{\varphi}\right]\right\Vert_2 ^{2}\nonumber \\
& =\left\Vert \M{M}\left[\varphi(\T{X}_{*})(\varphi(\T{X}_{*})\Tra\V{g}+\epsilon_{*})-\left(\varphi(\T{X}_{*})\varphi(\T{X}_{*})\Tra+\lambda\M{I}_{R}\right)\M{\Sigma}_{\varphi,\lambda}^{-1}\M{\Sigma}_{\varphi}\V{g}\right]\right\Vert_2 ^{2}\nonumber \\
& =\left\Vert \M{M}\varphi(\T{X}_{*})\epsilon_{*}+\M{M}\left[\varphi(\T{X}_{*})\varphi(\T{X}_{*})\Tra\V{g}-\left(\varphi(\T{X}_{*})\varphi(\T{X}_{*})\Tra+\lambda\M{I}_{R}\right)\M{\Sigma}_{\varphi,\lambda}^{-1}\M{\Sigma}_{\varphi}\V{g}\right]\right\Vert_2 ^{2}\nonumber \\
& =\underbrace{\left\Vert \M{M}\varphi(\T{X}_{*})\epsilon_{*}\right\Vert_2 ^{2}}_{(a)}\nonumber \\
& +\underbrace{2\epsilon_{*}\varphi(\T{X}_{*})\Tra\M{M}\Tra\M{M}\left[\varphi(\T{X}_{*})\varphi(\T{X}_{*})\Tra\V{g}-\left(\varphi(\T{X}_{*})\varphi(\T{X}_{*})\Tra+\lambda\M{I}_{R}\right)\M{\Sigma}_{\varphi,\lambda}^{-1}\M{\Sigma}_{\varphi}\V{g}\right]}_{(b)}\nonumber \\
& +\underbrace{\left\Vert \M{M}\left[\varphi(\T{X}_{*})\varphi(\T{X}_{*})\Tra\V{g}-\left(\varphi(\T{X}_{*})\varphi(\T{X}_{*})\Tra+\lambda\M{I}_{R}\right)\M{\Sigma}_{\varphi,\lambda}^{-1}\M{\Sigma}_{\varphi}\V{g}\right]\right\Vert_2 ^{2}}_{(c)}\label{eq:1_2_3_decomposition_tucker}\\
\end{align}
Now for term (a) in \eqref{eq:1_2_3_decomposition_tucker}, by assumptions
$\vec(\T{X}_{*})\sim \mathrm{N}(\mathbf{0},\M{I}_{\prod_{m}I_{m}})$,
$\V{\epsilon}\sim \mathrm{N}(\mathbf{0},\sigma^{2}\M{I}_{n})$ and
$\V{\epsilon}\perp\!\!\!\perp \vec(\T{X}_{*})$, we obtain 
\begin{align}
\E_{\T{X}_{*},\epsilon_{*}}[(a)] & =\E_{\T{X}_{*},\epsilon_{*}}[\epsilon_{*}^{2}\left\Vert \M{M}\varphi(\T{X}_{*})\right\Vert_2 ^{2}]\nonumber \\
& =\sigma^{2}\E_{\T{X}_{*}}[\left\Vert \M{M}\varphi(\T{X}_{*})\right\Vert_2 ^{2}]\nonumber \\
& =\sigma^{2}[\tr(\M{M}\Tra\M{M}\,\text{Var}(\varphi(\T{X}_{*})))+\E_{\T{X}_{*}}\varphi(\T{X}_{*})\M{M}\Tra\M{M}\E_{\T{X}_{*}}\varphi(\T{X}_{*})]\nonumber \\
& =\sigma^{2}\tr(\M{M}^{2}\M{\Sigma}_{\varphi})\nonumber \\
& =\sigma^{2}\sum_{r=1}^{R}\frac{v_{r}^2}{(v_{r}+\lambda)^{2}}\label{eq:term_1_tucker}
\end{align}
where the second to third equality is by covariance decomposition
formula and $\E_{\T{X}_{*}}(\varphi(\T{X}_{*}))=\V{0}$
yields the third to the fourth equality. For term (b) in \eqref{eq:1_2_3_decomposition_tucker},
the independence between $\epsilon_{*}$ and $\vec(\T{X}_{i})$
and $\E(\epsilon_{*})=0$ produces that 
\begin{align}
\E_{\T{X}_{*}}[(b)] & =0\label{eq:term_2_tucker}
\end{align}
Finally for term (c) in \eqref{eq:1_2_3_decomposition_tucker}, we first
notice that 
\begin{align}
& \E_{\T{X}_{*}}\left[\varphi(\T{X}_{*})\varphi(\T{X}_{*})\Tra\V{g}-\left(\varphi(\T{X}_{*})\varphi(\T{X}_{*})\Tra+\lambda\M{I}_{R}\right)\M{\Sigma}_{\varphi,\lambda}^{-1}\M{\Sigma}_{\varphi}\V{g}\right]\nonumber \\
& =\E_{\T{X}_{*}}\left[\varphi(\T{X}_{*})\varphi(\T{X}_{*})\Tra\V{g}\right]-\E_{\T{X}_{*}}\left[\left(\varphi(\T{X}_{*})\varphi(\T{X}_{*})\Tra+\lambda\M{I}_{R}\right)\M{\Sigma}_{\varphi,\lambda}^{-1}\M{\Sigma}_{\varphi}\V{g}\right]\nonumber \\
& =\M{\Sigma}_{\varphi}\V{g}-\M{\Sigma}_{\varphi,\lambda}\M{\Sigma}_{\varphi,\lambda}^{-1}\M{\Sigma}_{\varphi}\V{g}=\mathbf{0}\label{eq:term_3_mean_tucker}
\end{align}
and the covariance structure can be expanded and simplified as follows
\begin{align}
& \text{Var}\left(\varphi(\T{X}_{*})\varphi(\T{X}_{*})\Tra\V{g}-\left(\varphi(\T{X}_{*})\varphi(\T{X}_{*})\Tra+\lambda\M{I}_{R}\right)\M{\Sigma}_{\varphi,\lambda}^{-1}\M{\Sigma}_{\varphi}\V{g}\right)\nonumber \\
& =\text{Var}\left(\varphi(\T{X}_{*})\varphi(\T{X}_{*})\Tra\V{g}-\varphi(\T{X}_{*})\varphi(\T{X}_{*})\Tra\M{\Sigma}_{\varphi,\lambda}^{-1}\M{\Sigma}_{\varphi}\V{g}\right)\nonumber \\
& =\text{Var}\left(\varphi(\T{X}_{*})\varphi(\T{X}_{*})\Tra\underbrace{(\V{g}-\M{\Sigma}_{\varphi,\lambda}^{-1}\M{\Sigma}_{\varphi}\V{g})}_{\V{a}\in\Real^{R\times1}}\right)\nonumber \\
& =\E_{\T{X}_{*}}\left[(\varphi(\T{X}_{*})\varphi(\T{X}_{*})\Tra\V{a})(\varphi(\T{X}_{*})\varphi(\T{X}_{*})\Tra\V{a})\Tra\right]-\E_{\T{X}_{*}}\left[\varphi(\T{X}_{*})\varphi(\T{X}_{*})\Tra\V{a}\right]\E_{\T{X}_{*}}\left[\varphi(\T{X}_{*})\varphi(\T{X}_{*})\Tra\V{a}\right]\Tra\nonumber \\
& =\E_{\T{X}_{*}}\left[\varphi(\T{X}_{*})\varphi(\T{X}_{*})\Tra\V{a}\V{a}\Tra\varphi(\T{X}_{*})\varphi(\T{X}_{*})\Tra\right]-\E_{\T{X}_{*}}\left[\varphi(\T{X}_{*})\varphi(\T{X}_{*})\Tra\right]\V{a}\V{a}\Tra\E_{\T{X}_{*}}\left[\varphi(\T{X}_{*})\varphi(\T{X}_{*})\Tra\right]\Tra\nonumber \\
& =\E_{\T{X}_{*}}\left[(\varphi(\T{X}_{*})\Tra\V{a})^{2}\varphi(\T{X}_{*})\varphi(\T{X}_{*})\Tra\right]-\M{\Sigma}_{\varphi}\V{a}\V{a}\Tra\M{\Sigma}_{\varphi}\nonumber \\
& =\E_{\T{X}_{*}}\left[\varphi(\T{X}_{*})\Tra\V{a}\V{a}\Tra\varphi(\T{X}_{*})\;\varphi(\T{X}_{*})\varphi(\T{X}_{*})\Tra\right]-\M{\Sigma}_{\varphi}\V{a}\V{a}\Tra\M{\Sigma}_{\varphi}\nonumber \\
& =2\M{\Sigma}_{\varphi}\V{a}\V{a}\Tra\M{\Sigma}_{\varphi}+(\V{a}\Tra\M{\Sigma}_{\varphi}\V{a})\M{\Sigma}_{\varphi}-\M{\Sigma}_{\varphi}\V{a}\V{a}\Tra\M{\Sigma}_{\varphi}\nonumber \\
& =\M{\Sigma}_{\varphi}\V{a}\V{a}\Tra\M{\Sigma}_{\varphi}+(\V{a}\Tra\M{\Sigma}_{\varphi}\V{a})\M{\Sigma}_{\varphi}\label{eq:term_3_cov_tucker}
\end{align}
where the equality of 
\[
\E_{\T{X}_{*}}\left[\varphi(\T{X}_{*})\Tra\V{a}\V{a}\Tra\varphi(\T{X}_{*})\;\varphi(\T{X}_{*})\varphi(\T{X}_{*})\Tra\right]=2\M{\Sigma}_{\varphi}\V{a}\V{a}\Tra\M{\Sigma}_{\varphi}+(\V{a}\Tra\M{\Sigma}_{\varphi}\V{a})\M{\Sigma}_{\varphi}
\]
is given by the Isserlis' theorem. Then using (\ref{eq:term_3_mean_tucker})
and (\ref{eq:term_3_cov_tucker}), and the covariance decomposition trick,
the term (c) can be expressed as 
\begin{align}
& \E_{\T{X}_{*}}[(c)]\nonumber \\
& =\tr\left(\M{M}\Tra\M{M}\text{Var}\left(\varphi(\T{X}_{*})\varphi(\T{X}_{*})\Tra\V{g}-\left(\varphi(\T{X}_{*})\varphi(\T{X}_{*})\Tra+\lambda\M{I}_{R}\right)\M{\Sigma}_{\varphi,\lambda}^{-1}\M{\Sigma}_{\varphi}\V{g}\right)\right)+0\nonumber \\
& =\tr\left(\M{M}\Tra\M{M}(\M{\Sigma}_{\varphi}\V{a}\V{a}\Tra\M{\Sigma}_{\varphi}+(\V{a}\Tra\M{\Sigma}_{\varphi}\V{a})\M{\Sigma}_{\varphi})\right)\nonumber \\
& =\tr\left(\M{M}^{2}\M{\Sigma}_{\varphi}\V{a}\V{a}\Tra\M{\Sigma}_{\varphi}\right)+(\V{a}\Tra\M{\Sigma}_{\varphi}\V{a})\tr\left(\M{M}^{2}\M{\Sigma}_{\varphi}\right)\nonumber \\
& =\tr\left(\V{a}\Tra\M{\Sigma}_{\varphi}\M{M}^{2}\M{\Sigma}_{\varphi}\V{a}\right)+(\V{a}\Tra\M{\Sigma}_{\varphi}\V{a})\tr\left(\M{M}^{2}\M{\Sigma}_{\varphi}\right)\nonumber
\end{align}
Note we can write the vector $\V{a}$ as
\begin{equation*}
    \V{a} = \V{g}-\M{\Sigma}_{\varphi,\lambda}^{-1}\M{\Sigma}_{\varphi}\V{g} = (\mathbf{I} - \M{\Sigma}_{\varphi,\lambda}^{-1}\M{\Sigma}_{\varphi})\V{g}
\end{equation*}
and using the eigen-decomposition, we can have:
\begin{align}
& \E_{\T{X}_{*}}[(c)]\nonumber \\
& = \tr\left(\V{a}\Tra\M{\Sigma}_{\varphi}\M{M}^{2}\M{\Sigma}_{\varphi}\V{a}\right)+(\V{a}\Tra\M{\Sigma}_{\varphi}\V{a})\tr\left(\M{M}^{2}\M{\Sigma}_{\varphi}\right)\nonumber\\
& = \tr\left(\V{g}\Tra (\mathbf{I} - \M{\Sigma}_{\varphi,\lambda}^{-1}\M{\Sigma}_{\varphi})\M{\Sigma}_{\varphi}\M{M}^{2}\M{\Sigma}_{\varphi}(\mathbf{I} - \M{\Sigma}_{\varphi,\lambda}^{-1}\M{\Sigma}_{\varphi})\V{g}\right)+(\V{g}\Tra (\mathbf{I} - \M{\Sigma}_{\varphi,\lambda}^{-1}\M{\Sigma}_{\varphi})\M{\Sigma}_{\varphi}(\mathbf{I} - \M{\Sigma}_{\varphi,\lambda}^{-1}\M{\Sigma}_{\varphi})\V{g})\tr\left(\M{M}^{2}\M{\Sigma}_{\varphi}\right)\nonumber\\
& \asymp \lambda_{\max}\left( (\mathbf{I} - \M{\Sigma}_{\varphi,\lambda}^{-1}\M{\Sigma}_{\varphi})\M{\Sigma}_{\varphi}\M{M}^{2}\M{\Sigma}_{\varphi}(\mathbf{I} - \M{\Sigma}_{\varphi,\lambda}^{-1}\M{\Sigma}_{\varphi})\right)
 +\lambda_{\max}\left((\mathbf{I} - \M{\Sigma}_{\varphi,\lambda}^{-1}\M{\Sigma}_{\varphi})\M{\Sigma}_{\varphi}(\mathbf{I} - \M{\Sigma}_{\varphi,\lambda}^{-1}\M{\Sigma}_{\varphi})\right)\sum_{r=1}^{R}\frac{v_{r}^2}{(v_{r}+\lambda)^{2}}\nonumber \\
& =\frac{\lambda^2v_1}{(v_1+\lambda)^2}\frac{v_{1}^{2}}{(v_{1}+\lambda)^{2}}+\frac{\lambda^2v_1}{(v_1+\lambda)^2}\sum_{r=1}^{R}\frac{v_{r}^2}{(v_{r}+\lambda)^{2}}\nonumber \\
& \asymp \frac{\lambda^2v_1}{(v_1+\lambda)^2}\sum_{r=1}^{R}\frac{v_{r}^2}{(v_{r}+\lambda)^{2}}\label{eq:term_3_tucker}
\end{align}
Combining these three terms together in (\ref{eq:term_1_tucker}), (\ref{eq:term_2_tucker}),
and (\ref{eq:term_3_tucker}) yields 
\begin{align}
& \frac{2}{n}\E_{\T{X}_{*}}\left[\left\Vert \M{\Sigma}_{\varphi}^{\frac{1}{2}}\M{\Sigma}_{\varphi,\lambda}^{-1}\left[\varphi(\T{X}_{*})y_{*}-\left(\varphi(\T{X}_{*})\varphi(\T{X}_{*})\Tra+\lambda\M{I}_{R}\right)\M{\Sigma}_{\varphi,\lambda}^{-1}\V{\eta}_{\varphi}\right]\right\Vert_2 ^{2}\right]\nonumber \\
& =\frac{2}{n}((a)+(b)+(c))\nonumber \\
& =\frac{2\sigma^{2}}{n}\sum_{r=1}^{R}\frac{v_{r}^2}{(v_{r}+\lambda)^{2}}+0+ \frac{2\lambda^2v_1}{n(v_1+\lambda)^2}\sum_{r=1}^{R}\frac{v_{r}^2}{(v_{r}+\lambda)^{2}}\nonumber \\
& =\frac{2\left(\sigma^{2}+\frac{\lambda^2v_1}{(v_1+\lambda)^2}\right)}{n}\sum_{r=1}^{R}\frac{v_{r}^2}{(v_{r}+\lambda)^{2}}\label{eq:first_term_bound_tucker}
\end{align}
which gives the expected optimism as: 
\begin{equation*}
\mathrm{OptR}_{\T{X}}^{\mathrm{(true)}} =\frac{2\left(\sigma^{2}+\frac{\lambda^2v_1}{(v_1+\lambda)^2}\right)}{n}\sum_{r=1}^{R}\frac{v_{r}^2}{(v_{r}+\lambda)^{2}}+\mathcal{O}_{p}(n^{-\frac{3}{2}})
\end{equation*}
Using the eigenvalues decomposition results in Remark~\ref{remark:remark_tucker_eigenvalue}, we obtain
\begin{equation*}
    \sum_{r=1}^{R}\frac{v_{r}^2}{(v_{r}+\lambda)^{2}} = 
    \sum_{r_1=1}^{R_1}\sum_{r_2=1}^{R_2}\cdots\sum_{r_M=1}^{R_M}\frac{\bigl(\prod_{m=1}^M v_{r_m}^{(m)}\bigr)^2}
    {\bigl(\prod_{m=1}^M v_{r_m}^{(m)} + \lambda\bigr)^2}
\end{equation*}
Then $\E_{\T{X}}[\mathrm{Opt}_{R_{\T{X}}}]$ can also be expressed in terms of spectrum of each factor matrix $\M{U}_m$ for $m = 1, \dots, M$:
\begin{equation*}
\mathrm{OptR}_{\T{X}}^{\mathrm{(true)}} =\frac{2\left(\sigma^{2}+\frac{\lambda^2\prod_{m=1}^M v_1^{(m)}}{(\prod_{m=1}^M v_1^{(m)}+\lambda)^2}\right)}{n}\sum_{r_1=1}^{R_1}\sum_{r_2=1}^{R_2}\cdots\sum_{r_M=1}^{R_M}\frac{\bigl(\prod_{m=1}^M v_{r_m}^{(m)}\bigr)^2}
    {\bigl(\prod_{m=1}^M v_{r_m}^{(m)} + \lambda\bigr)^2}+\mathcal{O}_{p}(n^{-\frac{3}{2}})
\end{equation*}

\qed
\subsection{Proof of Theorem \ref{thm:thm_tucker_over_rank}}

\label{sec:Proof-of-Theorem4_2}

The proof will follow the same arguments as in Section~\ref{sec:Proof-of-Theorem4_1}. First note that given the true (\textit{m-rank}) Tucker rank  $R_1, R_2, \dots, R_M$ for each mode, by \citet{kolda2009tensor}, we can perfectly reconstruct the tensor $\T{B}$ using target ranks $R_{m_t} \geq R_m$ for $m = 1, 2, \dots, M$
(i.e. $\T{B}= \T{B}_{R_t} = \tilde{\T{G}}\times_{1} \tilde{\M{U}}_{1}\;\cdots\;\times_{m}\tilde{\M{U}}_{m}\;\cdots\;\times_{M}\tilde{\M{U}}_{M}$).
Thus, we can express $y_{*}$ as 
\begin{align*}
y_{*}(\T{X}_{*})&= \bigl\langle \vec(\T{B}_{R_t}), \vec(\T{X}_*) \bigr\rangle  +\epsilon_{*} \\
& =\bigl\langle \bigl(\tilde{\M{U}}_{M}\Kron\cdots\Kron \tilde{\M{U}}_{1}\bigr)\vec(\tilde{\T{G}}), \, \vec(\T{X}_*)\bigr\rangle  +\epsilon_{*} \\
& = \bigl\langle \tilde{\M{P}} \tilde{\V{g}},\,\vec(\T{X}_*) \bigr\rangle  +\epsilon_{*} \\
& = \bigl\langle \tilde{\M{P}}^{\!\top}\vec(\T{X}_*),\,\tilde{\V{g}}\bigr\rangle  +\epsilon_{*} \\
&=\varphi_{R_t}(\T{X}_{*})\Tra\tilde{\V{g}}+\epsilon_{*}
\end{align*}
where $\tilde{\M{P}}=\tilde{\M{U}}_{M}\Kron\cdots\Kron \tilde{\M{U}}_{1}\;\in\;\Real^{D\times R_t},\, D=\prod_{m=1}^{M}I_{m},\, R_t=\prod_{m=1}^{M}R_{m_t}$, and
$\tilde{\V{g}}=\vec(\tilde{\T{G}})\in\Real^{R_t\times1}$.
Then the key equality (\ref{eq:eta_form_tucker}) used in Section~\ref{sec:Proof-of-Theorem4_1} can be written as 
\begin{align}
\V{\eta}_{\varphi_{R_{t}}} & =\E_{\T{X}_{*}}\varphi_{R_{t}}(\T{X}_{*})y_{*}\nonumber \\
 & =\E_{\T{X}_{*},\epsilon_{*}}\varphi_{R_{t}}(\T{X}_{*})(\varphi_{R_{t}}(\T{X}_{*})\Tra\tilde{\V{g}}+\epsilon_{*})\nonumber \\
 & =\E_{\T{X}_{*}}\varphi_{R_{t}}(\T{X}_{*})\varphi_{R_{t}}(\T{X}_{*})\Tra\tilde{\V{g}}\nonumber \\
 & =\M{\Sigma}_{\varphi_{R_{t}}}\tilde{\V{g}}\label{eq:over_eta_form_tucker}
\end{align}
As a result, we can apply the identical arguments of \eqref{eq:term_1_tucker}), \eqref{eq:term_2_tucker},
and \eqref{eq:term_3_tucker} by
replacing the use of $\V{\eta}_{\varphi_{R}}=\M{\Sigma}_{\varphi_{R}}\tilde{\V{g}}$
with (\ref{eq:over_eta_form_tucker}), and eventually produces the optimism
under the over-specified rank as desired
\begin{align*}
\mathrm{OptR}_{\T{X}}^{\mathrm{(over)}}=\frac{2\left(\sigma^{2}+\frac{\lambda^2\tilde{v}_1}{(\tilde{v}_1+\lambda)^2}\right)}{n}\sum_{r=1}^{R_{t}}\frac{\tilde{v}_{r}^2}{(\tilde{v}_{r}+\lambda)^{2}}+\mathcal{O}_{p}(n^{-\frac{3}{2}})
\end{align*}
Again, using the eigenvalues decomposition results \eqref{eq:tucker_KRR_eigenvalues} in Remark~\ref{remark:remark_tucker_eigenvalue}, we have
\begin{equation*}
    \sum_{r=1}^{R_{t}}\frac{\tilde{v}_{r}^2}{(\tilde{v}_{r}+\lambda)^{2}} = \sum_{r_1=1}^{R_{1_t}}\sum_{r_2=1}^{R_{2_t}}\cdots\sum_{r_M=1}^{R_{M_t}}\frac{\bigl(\prod_{m=1}^M \tilde{v}_{r_m}^{(m)}\bigr)^2}
    {\bigl(\prod_{m=1}^M \tilde{v}_{r_m}^{(m)} + \lambda\bigr)^2}
\end{equation*}
and 
\begin{equation*}
    \tilde{v}_1 = \prod_{m=1}^M \tilde{v}_{1}^{(m)}
\end{equation*}
\qed

\subsection{Proof of Theorem \ref{thm:thm_tucker_under_rank}}

\label{sec:Proof-of-theorem_4_3}
\citet{kolda2009tensor} points out that the Tucker decomposition becomes inexact whenever the selected mode rank $R_{m_t}$ is lower than the true (mode $m$) rank for one or more $m$. Accordingly, our optimism derivation here for the under-specified Tucker case will follow the same arguments as Section~\ref{sec:Proof-of-Theorem4_1}, but include additional manipulations to account for the extra terms introduced by the truncation residual $\V{\Delta}$. 

Recall that $\tilde{\M{P}} = \tilde{\M{U}}_{M}\Kron\cdots\Kron\tilde{\M{U}}_{1}$, $\tilde{\V{g}} = \vec(\tilde{\T{G}})$, and $\V{\Delta} = \vec(\T{B})-\tilde{\M{P}}\tilde{\V{g}}$, we first show that the residual $\V{\Delta}$  is orthogonal to the column space of the Kronecker-product matrix $\tilde{\M{P}}$ (i.e., $\tilde{\M{P}}\Tra \V{\Delta} = 0$). To see this, one can view $\tilde{\V{g}}$ as the solution of a least-square problem.
\begin{equation*}
    \tilde{\V{g}} = \mathrm{argmin}_{\{\V{g}\}}\lVert \vec(\T{B})-\tilde{\M{P}}\V{g}\rVert^{2}
\end{equation*}
then $\tilde{\M{P}}\Tra \V{\Delta} = 0$ directly follows from the normal equation (first-order optimality condition) of the least-square problem. Consequently, for a new observation $\T{X}_*$ we have $y_{*}$ as 
\begin{align*}
y_{*}(\T{X}_{*}) & =\left\langle \vec(\T{X}_{*}),\vec(\T{B})\right\rangle +\epsilon_{*}\\
 & =\left\langle \vec(\T{X}_{*}),\vec(\T{B}_{R_{t}})+\V{\Delta}\right\rangle +\epsilon_{*}\\
 & =\left\langle \vec(\T{X}_{*}),\tilde{\M{P}}\tilde{\V{g}}\right\rangle +\left\langle \vec(\T{X}_{*}),\V{\Delta}\right\rangle +\epsilon_{*}\\
 & =\left\langle \tilde{\M{P}}\Tra \vec(\T{X}_{*}),\tilde{\V{g}}\right\rangle+g(\T{X}_{*},\V{\Delta})+\epsilon_{*}\\
 & =\varphi_{R_{t}}(\T{X}_{*})\Tra\tilde{\V{g}}+g(\T{X}_{*},\V{\Delta})+\epsilon_{*}
\end{align*}
where $g(\T{X}_{*},\V{\Delta})=\left\langle \vec(\T{X}_{*}),\V{\Delta}\right\rangle$.
Then the key equality (\ref{eq:eta_form_tucker}) used in Section~\ref{sec:Proof-of-Theorem4_1}
can be written as 
\begin{align}
\V{\eta}_{\varphi_{R_{t}}} & =\E_{\T{X}_{*}}\varphi_{R_{t}}(\T{X}_{*})y_{*}\nonumber \\
& =\E_{\T{X}_{*},\epsilon_{*}}\varphi_{R_{t}}(\T{X}_{*})(\varphi_{R_{t}}(\T{X}_{*})\Tra \tilde{\V{g}}+g(\T{X}_{*},\V{\Delta})+\epsilon_{*})\nonumber \\
& =\E_{\T{X}_{*}}\phi_{R_{t}}(\T{X}_{*})\phi_{R_{t}}(\T{X}_{*})\Tra\tilde{\V{g}}+\E_{\T{X}_{*}}\varphi_{R_{t}}(\T{X}_{*})g(\T{X}_{*},\V{\Delta})\nonumber \\
& =\M{\Sigma}_{\varphi_{R_{t}}}\tilde{\V{g}}\label{eq:under_eta_form_tucker}
\end{align}
where the third to the last equality comes from $\tilde{\M{P}}\Tra \V{\Delta} = 0$
\begin{align*}
\E_{\T{X}_{*}}\varphi_{R_{t}}(\T{X}_{*})g(\T{X}_{*},\V{\Delta}) & =\E_{\T{X}_{*}}\left\langle \tilde{\M{P}}, \vec(\T{X}_{*})\right\rangle \left\langle \V{\Delta},\vec(\T{X}_{*})\right\rangle \\
& = \tilde{\M{P}}\Tra \E_{\T{X}_{*}}[\vec(\T{X}_{*}) \vec(\T{X}_{*})\Tra] \V{\Delta}\\
& = \tilde{\M{P}}\Tra \V{\Delta}\\
& =0
\end{align*}
Thus by replacing the use of $\V{\eta}_{\varphi_{R}}=\M{\Sigma}_{\varphi_{R}}\tilde{\V{g}}$
with (\ref{eq:under_eta_form_tucker}), we can apply the identical arguments
of (\ref{eq:term_1_tucker}), (\ref{eq:term_2_tucker}),
and (\ref{eq:term_3_tucker})
in this case. And the approximation difference $\V{\Delta}$ only affects
(\ref{eq:first_term_bound_tucker}) by introducing positive error terms via
$y_{*}=\varphi_{R_{t}}(\T{X}_{*})\Tra\tilde{\V{g}}+g(\T{X}_{*},\V{\Delta})+\epsilon_{*}$.
Specifically, (\ref{eq:first_term_bound_tucker}) will
have 3 additional terms as (recall $\M{M} = \M{\Sigma}_{\varphi_{R_t}}^{\frac{1}{2}} \M{\Sigma}_{\varphi_{R_t}, \,\lambda}^{-1}$)
\begin{align}
 & \frac{2}{n}\E_{\T{X}_{*}}\left[\left\Vert \M{M}\left[\varphi_{R_{t}}(\T{X}_{*})y_{*}-\left(\varphi_{R_{t}}(\T{X}_{*})\varphi_{R_{t}}(\T{X}_{*})\Tra+\lambda\M{I}_{R_{t}}\right)\M{\Sigma}_{\varphi_{R_{t}},\lambda}^{-1}\V{\eta}_{\varphi_{R_{t}}}\right]\right\Vert_2 ^{2}\right]\nonumber \\
 & =\frac{2\left(\sigma^{2}+\frac{\lambda^2\tilde{v}_1}{(\tilde{v}_1+\lambda)^2}\right)}{n}\sum_{r=1}^{R_{t}}\frac{\tilde{v}_{r}^2}{(\tilde{v}_{r}+\lambda)^{2}}\nonumber \\
 & +\underbrace{\frac{2}{n}\E_{\T{X}_{*}}\left\Vert \M{M}\varphi_{R_{t}}(\T{X}_{*})g(\T{X}_{*},\V{\Delta})\right\Vert_2 ^{2}}_{(a)}\nonumber \\
 & +\underbrace{\frac{4}{n}\E_{\T{X}_{*}}\left[g(\T{X}_{*},\V{\Delta})\varphi_{R_{t}}(\T{X}_{*})\Tra\M{M}\Tra\M{M}\left[\varphi_{R_{t}}(\T{X}_{*})\varphi_{R_{t}}(\T{X}_{*})\Tra\tilde{\V{g}}-\left(\varphi_{R_{t}}(\T{X}_{*})\varphi_{R_{t}}(\T{X}_{*})\Tra+\lambda\M{I}_{R_t}\right)\M{\Sigma}_{\varphi_{R_{t}},\lambda}^{-1}\M{\Sigma}_{\varphi_{R_{t}}}\tilde{\V{g}}\right]\right]}_{(b)}\nonumber \\
 & +\underbrace{\frac{4}{n}\E_{\T{X}_{*}}\left[\epsilon_{*}\varphi_{R_{t}}(\T{X}_{*})\Tra\M{M}\Tra\M{M}\varphi_{R_{t}}(\T{X}_{*})g(\T{X}_{*},\V{\Delta})\right]}_{(c)}\label{eq:1_2_3_decomposition2_tucker}\\
  & =\frac{2\left(\sigma^{2}+\frac{\lambda^2\tilde{v}_1}{(\tilde{v}_1+\lambda)^2}\right)}{n}\sum_{r=1}^{R_{t}}\frac{\tilde{v}_{r}^2}{(\tilde{v}_{r}+\lambda)^{2}}
  +\frac{2}{n}\E_{\T{X}_{*}}\left\Vert \M{M}\varphi_{R_{t}}(\T{X}_{*})g(\T{X}_{*},\V{\Delta})\right\Vert_2 ^{2} +0+0\nonumber \\
 & =\frac{2\left(\sigma^{2}+\frac{\lambda^2\tilde{v}_1}{(\tilde{v}_1+\lambda)^2}\right)}{n}\sum_{r=1}^{R_{t}}\frac{\tilde{v}_{r}^2}{(\tilde{v}_{r}+\lambda)^{2}}
 + \frac{2}{n} \lVert\V{\Delta}\rVert^{2} \sum_{r=1}^{R_{t}}\frac{\tilde{v}_{r}^2}{(\tilde{v}_{r}+\lambda)^{2}}
 \nonumber 
\end{align}
where the term (b) and (c) in \eqref{eq:1_2_3_decomposition2_tucker} are
$0$ due to the orthogonality of $\V{\Delta}$ and $\varphi_{R_{t}}(\T{X}_{*})$
and the assumptions on $\epsilon_{*}$ and $\vec(\T{X}_*)$ (i.e., $\vec(\T{X}_{*})\sim \mathrm{N}(\mathbf{0},\M{I}_{\prod_{m}I_{m}})$,
$\epsilon_{*}\sim \mathrm{N}(0,\sigma^{2})$ and
$\epsilon_{*}\perp\!\!\!\perp\vec(\T{X}_{*})$). Thus, we
eventually produce the optimism under-specified rank as 
\begin{equation*}
\mathrm{OptR}_{\T{X}}^{\mathrm{(under)}}=\frac{2\left(\sigma^{2}+\frac{\lambda^2\tilde{v}_1}{(\tilde{v}_1+\lambda)^2}\right)}{n}\sum_{r=1}^{R_{t}}\frac{\tilde{v}_{r}^2}{(\tilde{v}_{r}+\lambda)^{2}}
+ \frac{2}{n} \lVert\V{\Delta}\rVert^{2} \sum_{r=1}^{R_{t}}\frac{\tilde{v}_{r}^2}{(\tilde{v}_{r}+\lambda)^{2}}
+\mathcal{O}_{p}(n^{-\frac{3}{2}})
\end{equation*}
Similar to Section \ref{sec:Proof-of-Theorem4_1} and \ref{sec:Proof-of-Theorem4_2}, we can further express the above results using the mode-wise eigenvalues. In fact, since here we have $\tilde{\M{U}}_{m} = {\M{U}}^{(R_{m_t})}_{m} \in \Real^{I_m \times R_{m_t}}$ be the best rank $R_{m_t}$ approximation of $\M{U}_m$, the spectrum of $\tilde{\M{U}}_{m}\Tra \tilde{\M{U}}_{m}$ will consist precisely of the top $R_{m_t}$ eigenvalues $\M{U}_m\Tra \M{U}_m$. Hence, we can have the following expression for $\sigma(\M{\Sigma}_{\varphi_{R_t}})$:
\begin{equation*}
    \sigma(\M{\Sigma}_{\varphi_{R_t}}) = \{\tilde{v}_1, \dots, \tilde{v}_{R_t}\} = \{\prod_m v_{r_m}^{(m)} : v_{r_m}^{(m)} \in \sigma(\M{U}_{m}\Tra \M{U}_{m}),\, r_m = 1, \dots, R_{m_t},\, m = 1, \dots, M\}
\end{equation*}
and the eigenvalue expression in the optimism equation can be written as:
\begin{equation*}
    \sum_{r=1}^{R_{t}}\frac{\tilde{v}_{r}^2}{(\tilde{v}_{r}+\lambda)^{2}} = \sum_{r_1=1}^{R_{1_t}}\sum_{r_2=1}^{R_{2_t}}\cdots\sum_{r_M=1}^{R_{M_t}}\frac{\bigl(\prod_{m=1}^M{v}_{r_m}^{(m)}\bigr)^2}
    {\bigl(\prod_{m=1}^M {v}_{r_m}^{(m)} + \lambda\bigr)^2} = \sum_{r=1}^{R_{t}}\frac{v_{r}^2}{(v_{r}+\lambda)^{2}}
\end{equation*}
and 
\begin{equation*}
    \tilde{v}_1 = \prod_{m=1}^M {v}_{1}^{(m)} = v_1
\end{equation*}
In other words, we have the following: 
\begin{equation*}
   \mathrm{OptR}_{\T{X}}^{\mathrm{(under)}}=\frac{2\left(\sigma^{2}+\frac{\lambda^2 v_1}{(v_1+\lambda)^2}\right)}{n}\sum_{r=1}^{R_t}\frac{v_{r}^2}{(v_{r}+\lambda)^{2}}
   + \frac{2}{n} \lVert\V{\Delta}\rVert^{2} \sum_{r=1}^{R_t}\frac{v_{r}^2}{(v_{r}+\lambda)^{2}}
    +\mathcal{O}_{p}(n^{-\frac{3}{2}})
\end{equation*}
where $v_r$ are the eigenvalues of $\M{\Sigma}_{\varphi}$ under true Tucker rank.
\qed

\subsection{Proof of Proposition \ref{prop:tucker_proposition}}
\label{sec:proof_tucker_prop4_1}
We start with the over specified rank case $R_{m_t} > R_m$ for at least one $m = 1, \dots, M$. Without loss of generality, let us assume only the mode-1 rank is misspecified at $R_{1_t} > R_1$ and the rest of rank are correctly specified at $R_{2_t} = R_2, \,R_{3_t} = R_3, \cdots, R_{M_t} = R_M$. It will lead to the following Tucker decomposition of $\T{B}_{R_t} = \T{B}$ with $R_t = R_{1_t} \times \prod_{m=2}^M R_m$:
\begin{equation*}
    \T{B}_{R_t} = \tilde{\T{G}}\times_{1} \tilde{\M{U}}_{1} \times_2 \M{U}_2\;\cdots\;\times_{m}{\M{U}}_{m}\;\cdots\;\times_{M}{\M{U}}_{M}
\end{equation*}
where $\tilde{\T{G}} \in \Real^{R_{1t} \times R_2 \times \cdots \times R_M}$, $\tilde{\M{U}}_1 \in \Real^{I_1 \times R_{1_t}}$, and $\M{U}_m \in \Real^{I_m \times R_m}$ for $m = 2, 3, \dots, M$. Then given the true (mode-1) rank as $R_1$ and $\tilde{\M{U}}_1$ has the column rank $R_{1_t} > R_1$, let $\text{col}(\M{A})$ be the column space of a matrix $\M{A}$, it is natural to have 
\begin{equation*}
    \text{col}(\M{U}_1) \subset \text{col}(\tilde{\M{U}}_1)
\end{equation*}
and the corresponding mode-1 spectrum (see Remark~\ref{remark:remark_tucker_eigenvalue}) can be decomposed as
\begin{equation*}
    \sigma(\tilde{\M{U}}_1\Tra \tilde{\M{U}}_1) = (v_{1}^{(1)}, \, v_{2}^{(1)}, \,\dots, \,v_{R_{1}}^{(1)},\,\tilde{v}_{R_1+1}^{(1)},\, \tilde{v}_{R_1+2}^{(1)}, \,\dots, \,\tilde{v}_{R_{1_t}}^{(1)})
\end{equation*}
where $(v_{1}^{(1)},\,\dots, \,v_{R_{1}}^{(1)})$ are the spectrum of the true mode-1 factor matrix $\M{U}_1\Tra \M{U}_1$ and $(\tilde{v}_{R_1+1}^{(1)},\, \dots, \,\tilde{v}_{R_{1_t}}^{(1)})$ are the extra eigenvalues introduced from the extra column basis of $\tilde{\M{U}}_1$ (by choosing $R_{1_t} > R_1$).

Then using the argument in Remark~\ref{remark:remark_tucker_eigenvalue}, the spectrum of $\M{\Sigma}_{\varphi_{R_t}}$ can be expressed as:
\begin{align*}
    \sigma(\M{\Sigma}_{\varphi_{R_t}}) 
    & = 
    \{\tilde{v}_{r_1}^{(1)} \times \prod_m v_{r_m}^{(m)} : 
    r_1 = 1, \dots, R_{1_t},\,
    r_m = 1, \dots, R_m,\, m = 2, \dots, M\} \\
& = 
    \{{v}_{r_1}^{(1)} \times \prod_m v_{r_m}^{(m)} : 
    r_1 = 1, \dots, R_{1},\,
    r_m = 1, \dots, R_m,\, m = 2, \dots, M\} \\
    &\cup  \{\tilde{v}_{r_1}^{(1)} \times \prod_m v_{r_m}^{(m)} : 
     r_1 = R_1 + 1, \dots, R_{1_t},\,
     r_m = 1, \dots, R_m,\, m = 2, \dots, M\} \\
    & = \sigma(\M{\Sigma}_{\varphi}) \cup
    \{\tilde{v}_{r_1}^{(1)} \times \prod_m v_{r_m}^{(m)} : 
     r_1 = R_1 + 1, \dots, R_{1_t},\,
     r_m = 1, \dots, R_m,\, m = 2, \dots, M\} 
\end{align*}
Consequently, the optimism with over-specified rank can be written as:
\begin{align*}
 \mathrm{OptR}_{\T{X}}^{(\mathrm{over})}&=\frac{2\left(\sigma^{2}+\frac{\lambda^2\tilde{v}_1}{(\tilde{v}_1+\lambda)^2}\right)}{n}\sum_{r=1}^{R_{t}}\frac{\tilde{v}_{r}^2}{(\tilde{v}_{r}+\lambda)^{2}}+\mathcal{O}_{p}(n^{-\frac{3}{2}})\\
& = \frac{2\left(\sigma^{2}+\frac{\lambda^2\tilde{v}_1}{(\tilde{v}_1+\lambda)^2}\right)}{n}\sum_{r_1=1}^{R_{1}}\sum_{r_2=1}^{R_{2}}\cdots\sum_{r_M=1}^{R_{M}}\frac{\bigl(\prod_{m=1}^M {v}_{r_m}^{(m)}\bigr)^2}
    {\bigl(\prod_{m=1}^M {v}_{r_m}^{(m)} + \lambda\bigr)^2} +\mathcal{O}_{p}(n^{-\frac{3}{2}}) \\
& + \frac{2\left(\sigma^{2}+\frac{\lambda^2\tilde{v}_1}{(\tilde{v}_1+\lambda)^2}\right)}{n}\sum_{r_1=R_1 + 1}^{R_{1_t}}\sum_{r_2=1}^{R_{2}}\cdots\sum_{r_M=1}^{R_{M}}\frac{\bigl(\tilde{v}_{r_1}^{(1)}\prod_{m=2}^M{v}_{r_m}^{(m)}\bigr)^2}
{\bigl(\tilde{v}_{r_1}^{(1)}\prod_{m=2}^M{v}_{r_m}^{(m)} + \lambda\bigr)^2} \\
& = \frac{2\left(\sigma^{2}+\frac{\lambda^2\tilde{v}_1}{(\tilde{v}_1+\lambda)^2}\right)}{n}\sum_{r=1}^{R}\frac{v_{r}^2}{(v_{r}+\lambda)^2}+\mathcal{O}_{p}(n^{-\frac{3}{2}})  \\
& + \frac{2\left(\sigma^{2}+\frac{\lambda^2\tilde{v}_1}{(\tilde{v}_1+\lambda)^2}\right)}{n}\sum_{r_1=R_1 + 1}^{R_{1_t}}\sum_{r_2=1}^{R_{2}}\cdots\sum_{r_M=1}^{R_{M}}\frac{\bigl(\tilde{v}_{r_1}^{(1)}\prod_{m=2}^M{v}_{r_m}^{(m)}\bigr)^2}
{\bigl(\tilde{v}_{r_1}^{(1)}\prod_{m=2}^M{v}_{r_m}^{(m)} + \lambda\bigr)^2} \\
& \approx \frac{2\sigma^{2}}{n}\sum_{r=1}^{R}\frac{v_{r}^2}{(v_{r}+\lambda)^2}+\mathcal{O}_{p}(n^{-\frac{3}{2}})  \\
& + \frac{2\sigma^{2}}{n}\sum_{r_1=R_1 + 1}^{R_{1_t}}\sum_{r_2=1}^{R_{2}}\cdots\sum_{r_M=1}^{R_{M}}\frac{\bigl(\tilde{v}_{r_1}^{(1)}\prod_{m=2}^M{v}_{r_m}^{(m)}\bigr)^2}
{\bigl(\tilde{v}_{r_1}^{(1)}\prod_{m=2}^M{v}_{r_m}^{(m)} + \lambda\bigr)^2} \\
& = \mathrm{OptR}_{\T{X}}^{(\mathrm{true})}+ \frac{2\sigma^{2}}{n}\sum_{r_1=R_1 + 1}^{R_{1_t}}\sum_{r_2=1}^{R_{2}}\cdots\sum_{r_M=1}^{R_{M}}\frac{\bigl(\tilde{v}_{r_1}^{(1)}\prod_{m=2}^M{v}_{r_m}^{(m)}\bigr)^2}
{\bigl(\tilde{v}_{r_1}^{(1)}\prod_{m=2}^M{v}_{r_m}^{(m)} + \lambda\bigr)^2}
\end{align*}
where we use the same technique in Proposition~\ref{prop:CP_proposition} to argue that (note here we only need a weaker $\lambda$ condition that $\lambda \ll v_1 \wedge \tilde{v}_1$)
\[
\frac{2\left(\sigma^{2}+\frac{\lambda^2\tilde{v}_1}{(\tilde{v}_1+\lambda)^2}\right)}{n} \approx \frac{2\sigma^{2}}{n}, \quad \frac{2\left(\sigma^{2}+\frac{\lambda^2{v}_1}{({v}_1+\lambda)^2}\right)}{n} \approx \frac{2\sigma^{2}}{n}
\]
Hence,
if at least one $\{\tilde{v}_{r_1}^{(1)}\}_{r_1=R_1+1}^{R_{1_t}}$ stays positive, the second term will be greater than $0$, leading to $\mathrm{OptR}_{\T{X}}^{(\mathrm{over})} >\mathrm{OptR}_{\T{X}}^{(\mathrm{true})}$.
Consequently, if all $(\tilde{v}_{R_1+1}^{(1)},\dots, \tilde{v}_{R_{1_t}}^{(1)}) = 0$, the optimism $\mathrm{OptR}_{\T{X}}^{(\mathrm{over})}$
reduces to $\mathrm{OptR}_{\T{X}}^{(\mathrm{true})}$
in \eqref{eq:tensorkrr_true_rank_tucker}. 

Then we switch to the under-specified rank case. Similar as above, without loss of generality, let us assume that only the mode-1 rank is misspecified at $R_{1_t} < R_1$ and the rest of the mode ranks are correctly specified at $R_{2_t} = R_2, \,R_{3_t} = R_3, \cdots, R_{M_t} = R_M$, which leads to the following tucker decomposition of $\T{B}_{R_t} = \T{B}$ with $R_t = R_{1_t} \times \prod_{m=2}^M R_m$:
\begin{equation*}
    \T{B}_{R_t} = \tilde{\T{G}}\times_{1} \tilde{\M{U}}_{1} \times_2 \M{U}_2\;\cdots\;\times_{m}{\M{U}}_{m}\;\cdots\;\times_{M}{\M{U}}_{M}
\end{equation*}
where $\tilde{\T{G}} \in \Real^{R_{1t} \times R_2 \times \cdots \times R_M}$, $\tilde{\M{U}}_1 \in \Real^{I_1 \times R_{1_t}}$, and $\M{U}_m \in \Real^{I_m \times R_m}$ for $m = 2, 3, \dots, M$. Using the results in Section~\ref{sec:Proof-of-theorem_4_3}, the expected optimism in this case can be expressed as
\begin{equation*}
   \mathrm{OptR}_{\T{X}}^{(\mathrm{under})}=\frac{2\left(\sigma^{2}+\frac{\lambda^2 v_1}{(v_1+\lambda)^2}\right)}{n}\sum_{r_1=1}^{R_{1_t}}\sum_{r_2=1}^{R_{2}}\cdots\sum_{r_M=1}^{R_{M}}\frac{\bigl(\prod_{m=1}^M{v}_{r_m}^{(m)}\bigr)^2}
    {\bigl(\prod_{m=1}^M {v}_{r_m}^{(m)} + \lambda\bigr)^2}
    + \frac{2\lVert\V{\Delta}\rVert^{2}}{n} \sum_{r=1}^{R_t}\frac{v_{r}^2}{(v_{r}+\lambda)^{2}}
    +\mathcal{O}_{p}(n^{-\frac{3}{2}})
\end{equation*}
And given small enough $\lambda$, it will become (using the same arguments in Section~\ref{sec:proof_CP_prop3_1})
\begin{equation*}
    \mathrm{OptR}_{\T{X}}^{(\mathrm{under})} \approx \frac{2\left(\sigma^{2}+1\right)}{n} R_t + \frac{2\lVert\V{\Delta}\rVert^{2}}{n} R_t +\mathcal{O}_{p}(n^{-\frac{3}{2}})
\end{equation*}
Then the difference between the expected optimism of the under-specified model ($R_t <R = \prod_{m=1}^{M} R_{m}$) and that of the true-rank model is roughly:
\begin{equation*}
    \mathrm{OptR}_{\T{X}}^{(\mathrm{under})}-\mathrm{OptR}_{\T{X}}^{(\mathrm{true})}
    \approx \frac{2\sigma^{2}}{n} (R_t - R) + \frac{2\lVert\V{\Delta}\rVert^{2}}{n} R_t
\end{equation*}
and $\lVert\V{\V{\Delta}}\rVert^{2} \geq \sigma^2 \frac{R - R_t}{R_t}$ guarantees that the dominant contribution
in $\mathrm{OptR}_{\T{X}}^{(\mathrm{under})} - \mathrm{OptR}_{\T{X}}^{(\mathrm{true})}$
will come from the positive error term $\lVert\V{\Delta}\rVert^{2}$
introduced by the approximation. Consequently, we get 
\[
\mathrm{OptR}_{\T{X}}^{(\mathrm{under})} \geq \mathrm{OptR}_{\T{X}}^{(\mathrm{true})}
\]
\qed

\subsection{Additive Features and Stationary Kernel}
The KRR optimism results in \citet{luo2025optimism} is based on inner-product kernels (i.e., $\M{K} = \V{\Phi}\Tra \V{\Phi}$ where $\V{\Phi} = (\phi(\V{x}_1), \dots, \phi(\V{x}_n))$). These results can be naturally extended to additive features and stationary kernels, where $K (\V{x}_i, \V{x}_j) = \varsigma(|\V{x}_i- \V{x}_j|)$. We first consider the simpler setting of an additive kernel. Under Assumption A3 of \cite{luo2025optimism}, denote the row feature vectors as $\phi(\V{x}_i) \in \Real^{q}$ with the feature mapping $\phi : \Real^p \rightarrow \Real^q$. Then if this feature mapping can be decomposed as 
\begin{equation*}
    \phi(\V{x}_i) = \phi_1(\V{x}_i) + \phi_2(\V{x}_i)
\end{equation*}
where $\phi_j : \Real^p \rightarrow \Real^q$ for $j = 1,2$, the following result provides an upper bound of the KRR optimism with $\phi$ in terms of each individual contributions from $\phi_1$ and $\phi_2$.

\begin{proposition}
(Expected Optimism of Kernel Ridge Regression with Additive Features) \label{prop:additive_kernel_bound} Under Assumption A3 of \cite{luo2025optimism} and suppose that the feature vectors $\phi(\V{x}_*) = \phi_1(\V{x}_*) + \phi_2(\V{x}_*)$. Denote $\V{\eta}_{\phi_i} =  \E_{\V{x}_*}\phi_i(\V{x}_*)y_* \in \Real^{q\times 1}$ and $\M{\Sigma}_{\phi_i,\, \lambda} =  \E_{\V{x}_*}[\phi_i(\V{x}_*)\phi_i(\V{x}_*)\Tra + \lambda \M{I}] \in \Real^{q \times q}$ for $i = 1,2$. Then the expected random optimism for the kernel ridge regression defined by $K(\cdot, \cdot) = \phi(\cdot)\Tra\phi(\cdot)$ is upper bounded by

\begin{equation}
   \mathrm{OptR}_{\M{X}}^{(\phi)} \leq \frac{2}{n} \E_{\M{X}}\lVert \M{M} \V{r}_{1*}\rVert_2^2 + \frac{2}{n} \E_{\M{X}}\lVert \M{M} \V{r}_{2*}\rVert^2 + \frac{2}{n} \E_{\M{X}}\lVert \M{M} \V{r}_{12*}\rVert_2^2 + \mathcal{O}_P(n^{-\frac{3}{2}})
   \label{eq:additive_kernel_general_bound}
\end{equation}
where 
\begin{align*}
    &\M{M} = \M{\Sigma}_{\phi}^{\frac{1}{2}}\M{\Sigma}_{\phi, \lambda}^{-1}\\
    &\V{r}_{i*} = \phi_i(\V{x}_*) y_* - (\phi_i(\V{x}_*)\phi_i(\V{x}_*)\Tra + \lambda \M{I})\M{\Sigma}_{\phi, \lambda}^{-1} \V{\eta}_{\phi_i} \quad \mathrm{for }\; i = 1,2 \\
    &\V{r}_{12*}  = -\left((\phi_1(\V{x}_*)\phi_1(\V{x}_*)\Tra)\M{\Sigma}_{\phi, \lambda}^{-1} \V{\eta}_{\phi_2} + (\phi_2(\V{x}_*)\phi_2(\V{x}_*)\Tra)\M{\Sigma}_{\phi, \lambda}^{-1} \V{\eta}_{\phi_1} + (\phi_1(\V{x}_*)\phi_2(\V{x}_*)\Tra + \phi_2(\V{x}_*)\phi_1(\V{x}_*)\Tra)\M{\Sigma}_{\phi, \lambda}^{-1} \V{\eta}_{\phi} \right)
\end{align*}

\end{proposition}

\begin{proof}
Under the additive feature mapping assumption, we can obtain 
\begin{align*}
    \V{\eta}_\phi &=  \E_{\V{x}_*}\phi(\V{x}_*)y_* \\
    & =  \E_{\V{x}_*}\left(\phi_1(\V{x}_*) + \phi_2(\V{x}_*)\right)y_* \\
    & =  \E_{\V{x}_*}\phi_1(\V{x}_*)y_* +  \E_{\V{x}_*}\phi_2(\V{x}_*)y_* \\
    & = \V{\eta}_{\phi_1} + \V{\eta}_{\phi_2}
\end{align*}
where $\V{\eta}_{\phi_i} \in \Real^{q \times 1}$ for $i = 1, 2$. Similarly, we also have
\begin{align*}
    \phi(\V{x}_*)\phi(\V{x}_*)\Tra & = {\phi_1(\V{x}_*)\phi_1(\V{x}_*)\Tra} \\
    & + {\phi_2(\V{x}_*)\phi_2(\V{x}_*)\Tra}\\
    & + {\phi_1(\V{x}_*)\phi_2(\V{x}_*)\Tra + \phi_2(\V{x}_*)\phi_1(\V{x}_*)\Tra}
\end{align*}
Then, let $\V{r}_*$ be
\begin{equation*}
    \V{r}_* = \phi(\V{x}_*)y_* - (\phi(\V{x}_*)\phi(\V{x}_*)\Tra + \lambda \M{I}) \M{\Sigma}_{\phi, \, \lambda}^{-1} \V{\eta}_{\phi}
\end{equation*}
we can express $\V{r}_*$ using the decomposition we have above:
\begin{align*}
    \V{r}_* & = \underbrace{\phi_1(\V{x}_*) y_* - (\phi_1(\V{x}_*)\phi_1(\V{x}_*)\Tra + \lambda \M{I})\M{\Sigma}_{\phi, \lambda}^{-1} \V{\eta}_{\phi_1}}_{\V{r_{1*}}}\\
    & + 
    \underbrace{\phi_2(\V{x}_*) y_* - (\phi_2(\V{x}_*)\phi_2(\V{x}_*)\Tra + \lambda \M{I})\M{\Sigma}_{\phi, \lambda}^{-1} \V{\eta}_{\phi_2}}_{\V{r_{2*}}}\\
    & \underbrace{-\left((\phi_1(\V{x}_*)\phi_1(\V{x}_*)\Tra)\M{\Sigma}_{\phi, \lambda}^{-1} \V{\eta}_{\phi_2} + (\phi_2(\V{x}_*)\phi_2(\V{x}_*)\Tra)\M{\Sigma}_{\phi, \lambda}^{-1} \V{\eta}_{\phi_1} + (\phi_1(\V{x}_*)\phi_2(\V{x}_*)\Tra + \phi_2(\V{x}_*)\phi_1(\V{x}_*)\Tra)\M{\Sigma}_{\phi, \lambda}^{-1} \V{\eta}_{\phi} \right)}_{\V{r_{12*}}}
\end{align*}
Denote 
\begin{equation*}
    \M{M} = \M{\Sigma}_{\phi}^{\frac{1}{2}}\M{\Sigma}_{\phi}^{-1}
\end{equation*}
we can then express the expected random optimism for the KRR defined by $K(\cdot, \cdot) = \phi(\cdot)\Tra\phi(\cdot)$ as:
\begin{align*}
    \mathrm{OptR}_{\M{X}}^{(\phi)} & = 
    \frac{2}{n} \E_{\V{x}_{*}}\left[\left\Vert \mathbf{\Sigma}_{\phi}^{\frac{1}{2}}\mathbf{\Sigma}_{\phi,\lambda}^{-1}\left[\phi(\V{x}_{*})y_{*}-\left(\phi(\V{x}_{*})\phi(\V{x}_{*})\Tra+\lambda\M{I}\right)\mathbf{\Sigma}_{\phi,\lambda}^{-1}\V{\eta}_{\phi}\right]\right\Vert_2 ^{2}\right] +\mathcal{O}_{p}\left(n^{-3/2}\right) \\
    & = \frac{2}{n} \E_{\V{x}_{*}}\left[\left\Vert \M{M} \V{r}_* \right\Vert_2 ^{2}\right] + \mathcal{O}_{p}\left(n^{-3/2}\right)\\
    & = \frac{2}{n} \E_{\V{x}_{*}}\left[\left\Vert \M{M} (\V{r}_{1*} + \V{r}_{2*} + \V{r}_{12*}) \right\Vert_2 ^{2}\right] + \mathcal{O}_{p}\left(n^{-3/2}\right)\\
    & \leq 
    \frac{2}{n} \E_{\M{X}}\lVert \M{M} \V{r}_{1*}\rVert^2 + \frac{2}{n} \E_{\M{X}}\lVert \M{M} \V{r}_{2*}\rVert^2 + \frac{2}{n} \E_{\M{X}}\lVert \M{M} \V{r}_{12*}\rVert^2 + \mathcal{O}_P(n^{-\frac{3}{2}})
\end{align*}
where the last inequality is by the triangle inequality. 
\end{proof}

Here, $\V{r}_{i*}$ can be viewed as the KRR optimism for each decomposed feature map ($\phi_1$ and $\phi_2$), and $\V{r}_{12*}$ captures the contribution resulting from the interaction between these two mappings. Under  additional assumptions on $\phi_1$ and $\phi_2$, we can further simplify the above results.

\begin{corollary}
\label{cor:additive_kernel_bound_simplify}Under the same assumptions of Proposition
\ref{prop:additive_kernel_bound}, if we further assume that the decomposed feature maps are disjoint 
\begin{equation*}
    \phi(\V{x}_*) = \phi_1(\V{x}_*) + \phi_2(\V{x}_*) = \begin{pmatrix}\varphi_{1}(\V{x}_*)\\
    \vdots\\
    \varphi_{d}(\V{x}_*)\\
    0\\
    \vdots\\
    0
    \end{pmatrix} + 
    \begin{pmatrix}0\\
    \vdots\\
    0\\
    \varphi_{d+1}(\V{x}_*)\\
    \vdots\\
    \varphi_{q}(\V{x}_*)
    \end{pmatrix}
\end{equation*}
where $0\leq d \leq q$. Then the expected optimism for the kernel ridge regression defined by $K(\cdot, \cdot) = \phi(\cdot)\Tra\phi(\cdot)$ can also be decomposed as
\begin{equation}
    \mathrm{OptR}_{\M{X}}^{(\phi)} = \mathrm{OptR}_{\M{X}}^{(\phi_1)} + \mathrm{OptR}_{\M{X}}^{(\phi_2)}
    \label{eq:additive_kernel_optimism_disjoint}
\end{equation}
where $\mathrm{OptR}_{\M{X}}^{(\phi_i)}$ is the expected optimism for the KRR defined by $K_i(\cdot, \cdot) = \phi_i(\cdot)\Tra\phi_i(\cdot)$ for $i=1,2$.
\end{corollary}
\begin{proof}
Let $\V{\varphi}_a(\V{x}_*) \in \Real^{d \times 1}$ and $\V{\varphi}_b(\V{x}_*) \in \Real^{(q-d) \times 1}$ be the corresponding non-zero feature components of $\phi_1(\V{x}_*)$ and $\phi_2(\V{x}_*)$ respectively:
\begin{equation*}
    \phi_1(\V{x}_*) = \begin{pmatrix}\varphi_{1}(\V{x}_*)\\
    \vdots\\
    \varphi_{d}(\V{x}_*)\\
    0\\
    \vdots\\
    0
    \end{pmatrix} = \begin{pmatrix}
        \V{\varphi}_a(\V{x}_*) \\
        \V{0}
    \end{pmatrix}, \quad
    \phi_2(\V{x}_*) =  \begin{pmatrix}0\\
    \vdots\\
    0\\
    \varphi_{d+1}(\V{x}_*)\\
    \vdots\\
    \varphi_{q}(\V{x}_*)
    \end{pmatrix} = \begin{pmatrix}
        \V{0} \\
        \V{\varphi}_b(\V{x}_*)
    \end{pmatrix}
\end{equation*}
Then we naturally have orthogonality $\phi_1(\V{x}_*)\Tra \phi_2(\V{x}_*) = 0$, which gives
\begin{align*}
    \M{\Sigma}_{\phi} & =  \E_{\V{x}_*} \phi(\V{x}_*)\phi(\V{x}_*)\Tra  \\
    & =  \E_{\V{x}_*}{\phi_1(\V{x}_*)\phi_1(\V{x}_*)\Tra} 
     +  \E_{\V{x}_*} {\phi_2(\V{x}_*)\phi_2(\V{x}_*)\Tra}\\
    & = \M{\Sigma}_{\phi_1} + \M{\Sigma}_{\phi_2}
\end{align*}
where $\M{\Sigma}_{\phi_i} \in \Real^{q \times q}$ for $i = 1, 2$. Note that $\M{\Sigma}_{\phi_i}$ will have a block structure:
\begin{equation*}
    \M{\Sigma}_{\phi_1} = 
    \begin{pmatrix}\mathbf{\Sigma}_{\varphi_{a}} & \M{0}\\
    \M{0} & \M{0}
    \end{pmatrix},
    \quad
    \M{\Sigma}_{\phi_2} = 
    \begin{pmatrix}\M{0} & \M{0}\\
    \M{0} & \mathbf{\Sigma}_{\varphi_{b}}
    \end{pmatrix}
\end{equation*}
where $\M{\Sigma}_{\varphi_{a}} =  \E_{\V{x}_*} \varphi_a(\V{x}_*)\varphi_a(\V{x}_*)\Tra \in \Real^{d \times d}$ and $\M{\Sigma}_{\varphi_{b}} =  \E_{\V{x}_*} \varphi_b(\V{x}_*)\varphi_b(\V{x}_*)\Tra \in \Real^{(q-d) \times (q-d)}$, and obviously $\M{\Sigma}_{\phi_1} \M{\Sigma}_{\phi_2} = \M{0}$. Then $\M{M}$ can be decomposed as (due to the block structure):
\begin{align*}
    \M{M} & = \M{\Sigma}_{\phi}^{\frac{1}{2}}\M{\Sigma}_{\phi}^{-1} \\
    & = \left(\M{\Sigma}_{\phi_1} + \M{\Sigma}_{\phi_2}\right)^{\frac{1}{2}} \left(\M{\Sigma}_{\phi_1} + \M{\Sigma}_{\phi_2} + \lambda \M{I} \right)^{-1} \\
    & = \begin{pmatrix}\M{\Sigma}_{\varphi_{a}} & \M{0}\\
    \M{0} & \M{\Sigma}_{\varphi_{b}}
    \end{pmatrix}^{\frac{1}{2}} 
    \left( \begin{pmatrix}\M{\Sigma}_{\varphi_{a}} & \M{0}\\
    \M{0} & \M{\Sigma}_{\varphi_{b}}
    \end{pmatrix} + \begin{pmatrix}\lambda \M{I}_d & \M{0}\\
    \M{0} & \lambda \M{I}_{q-d}
    \end{pmatrix} \right)^{-1} \\
    & = \begin{pmatrix}\M{\Sigma}_{\varphi_{a}}^{\frac{1}{2}}  & \M{0}\\
    \M{0} & \M{\Sigma}_{\varphi_{b}}^{\frac{1}{2}} 
    \end{pmatrix}
     \begin{pmatrix}(\M{\Sigma}_{\varphi_{a}} + \lambda \M{I}_d)^{-1} & \M{0}\\
    \M{0} & (\M{\Sigma}_{\varphi_{b}} + \lambda \M{I}_{q-d})^{-1}
    \end{pmatrix} \\
    & = \begin{pmatrix} \M{\Sigma}_{\varphi_{a}}^{\frac{1}{2}} (\M{\Sigma}_{\varphi_{a}} + \lambda \M{I}_d)^{-1} & \M{0}\\
    \M{0} & \M{\Sigma}_{\varphi_{b}}^{\frac{1}{2}} (\M{\Sigma}_{\varphi_{b}} + \lambda \M{I}_{q-d})^{-1}
    \end{pmatrix}
\end{align*}
where $\M{I}_{d} \in \Real^{d \times d}$ and $\M{I}_{q-d} \in \Real^{(q-d) \times (q-d)}$ are identity matrices. Similarly, notice that
\begin{align*}
    \V{\eta}_{\phi_1} & =  \E_{\V{x}_*}\phi_1(\V{x}_*)y_* =   \E_{\V{x}_*} \begin{pmatrix}
        \V{\varphi}_a(\V{x}_*)y_* \\
        \V{0}
    \end{pmatrix} = \begin{pmatrix}
        \V{\eta}_{\varphi_a} \\
        \V{0}
    \end{pmatrix} \\
    \V{\eta}_{\phi_2} & =  \E_{\V{x}_*}\phi_2(\V{x}_*)y_* =   \E_{\V{x}_*} \begin{pmatrix}
        \V{0} \\
        \V{\varphi}_b(\V{x}_*)y_*
    \end{pmatrix} = \begin{pmatrix}
        \V{0} \\
        \V{\eta}_{\varphi_b}
    \end{pmatrix}
\end{align*}
where $\V{\eta}_{\varphi_a} =  \E_{\V{x}_*}\V{\varphi}_a(\V{x}_*)y_* \in \Real^{d \times 1}$ and $\V{\eta}_{\varphi_b} =  \E_{\V{x}_*}\V{\varphi}_b(\V{x}_*)y_* \in \Real^{(q-d) \times 1}$. Then the interaction term $\V{r}_{12*}$ in Proposition \ref{prop:additive_kernel_bound} will vanish:
\begin{equation*}
    \V{r}_{12*} = \M{0}
\end{equation*}
as $\phi_1(\V{x}_*)\Tra \phi_2(\V{x}_*) = 0$ and
\begin{equation*}
    (\phi_1(\V{x}_*)\phi_1(\V{x}_*)\Tra)\M{\Sigma}_{\phi, \lambda}^{-1} \V{\eta}_{\phi_2} = 
    \begin{pmatrix}\mathbf{\Sigma}_{\varphi_{a}} & \M{0}\\
    \M{0} & \M{0}
    \end{pmatrix} \begin{pmatrix}(\M{\Sigma}_{\varphi_{a}} + \lambda \M{I}_d)^{-1} & \M{0}\\
    \M{0} & (\M{\Sigma}_{\varphi_{b}} + \lambda \M{I}_{q-d})^{-1}
    \end{pmatrix} \begin{pmatrix}
        \V{0} \\
        \V{\eta}_{\varphi_b}
    \end{pmatrix} = \M{0}.
\end{equation*}
The same argument implies
\begin{align*}
    \V{r}_{1*} & = \phi_1(\V{x}_*) y_* - (\phi_1(\V{x}_*)\phi_1(\V{x}_*)\Tra + \lambda \M{I})\M{\Sigma}_{\phi, \lambda}^{-1} \V{\eta}_{\phi_1} \\
    & = \begin{pmatrix}
        \V{\varphi}_a(\V{x}_*) \\
        \V{0}
    \end{pmatrix} y_* \\
    & -
    \left( \begin{pmatrix}\varphi_a(\V{x}_*)\varphi_a(\V{x}_*)\Tra & \M{0}\\
    \M{0} & \M{0}
    \end{pmatrix} + \begin{pmatrix}\lambda \M{I}_d & \M{0}\\
    \M{0} & \lambda \M{I}_{q-d}
    \end{pmatrix} \right) \begin{pmatrix}(\M{\Sigma}_{\varphi_{a}} + \lambda \M{I}_d)^{-1} & \M{0}\\
    \M{0} & (\M{\Sigma}_{\varphi_{b}} + \lambda \M{I}_{q-d})^{-1} 
    \end{pmatrix} \begin{pmatrix}
        \V{\eta}_{\varphi_a} \\
        \V{0}
    \end{pmatrix}\\
    & = \begin{pmatrix}
         \V{\varphi}_a(\V{x}_*)y_* - (\varphi_a(\V{x}_*)\varphi_a(\V{x}_*)\Tra + \lambda \M{I}_d) (\M{\Sigma}_{\varphi_{a}} + \lambda \M{I}_d)^{-1} \V{\eta}_{\varphi_a} \\
        \V{0}
    \end{pmatrix}
\end{align*}
and 
\begin{align*}
    \V{r}_{2*} & = \phi_2(\V{x}_*) y_* - (\phi_2(\V{x}_*)\phi_2(\V{x}_*)\Tra + \lambda \M{I})\M{\Sigma}_{\phi, \lambda}^{-1} \V{\eta}_{\phi_2} \\
    & = \begin{pmatrix}
        \V{0} \\
        \V{\varphi}_b(\V{x}_*)
    \end{pmatrix} y_* \\
    & - 
    \left( \begin{pmatrix}\M{0} & \M{0}\\
    \M{0} & \varphi_b(\V{x}_*)\varphi_b(\V{x}_*)\Tra
    \end{pmatrix} + \begin{pmatrix}\lambda \M{I}_d & \M{0}\\
    \M{0} & \lambda \M{I}_{q-d}
    \end{pmatrix} \right) \begin{pmatrix}(\M{\Sigma}_{\varphi_{a}} + \lambda \M{I}_d)^{-1} & \M{0}\\
    \M{0} & (\M{\Sigma}_{\varphi_{b}} + \lambda \M{I}_{q-d})^{-1} 
    \end{pmatrix} \begin{pmatrix}
        \V{0} \\
        \V{\eta}_{\varphi_b}
    \end{pmatrix}\\
    & = \begin{pmatrix}
        \V{0} \\
        \V{\varphi}_b(\V{x}_*)y_* - (\varphi_b(\V{x}_*)\varphi_b(\V{x}_*)\Tra + \lambda \M{I}_{q-d}) (\M{\Sigma}_{\varphi_{b}} + \lambda \M{I}_{q-d})^{-1} \V{\eta}_{\varphi_b}
    \end{pmatrix}.
\end{align*}
Hence, the expected random optimism for the KRR defined by $K(\cdot, \cdot) = \phi(\cdot)\Tra\phi(\cdot)$ can be expressed as 
\begin{align*}
    \mathrm{OptR}_{\M{X}}^{(\phi)} & = \frac{2}{n} \E_{\V{x}_{*}} \left\Vert \M{M} (\V{r}_{1*} + \V{r}_{2*}) \right\Vert_2 ^{2}  + \mathcal{O}_{p}\left(n^{-3/2}\right)\\
    & = \frac{2}{n} \E_{\V{x}_{*}} \left\Vert \M{M} \V{r}_{1*} + \M{M} \V{r}_{2*} \right\Vert_2 ^{2}  + \mathcal{O}_{p}\left(n^{-3/2}\right)\\
    & = \frac{2}{n} \E_{\V{x}_{*}}  \left\Vert\left(\M{\Sigma}_{\varphi_{a}}^{\frac{1}{2}} (\M{\Sigma}_{\varphi_{a}} + \lambda \M{I}_d)^{-1}\right) \V{r}_{1*}
    + \left(\M{\Sigma}_{\varphi_{b}}^{\frac{1}{2}} (\M{\Sigma}_{\varphi_{b}} + \lambda \M{I}_{q-d})^{-1}\right)  \V{r}_{2*} \right\Vert_2 ^{2}  + \mathcal{O}_{p}\left(n^{-3/2}\right)\\
    & = \frac{2}{n} \E_{\V{x}_{*}}  \left\Vert\M{\Sigma}_{\varphi_{a}}^{\frac{1}{2}} (\M{\Sigma}_{\varphi_{a}} + \lambda \M{I}_d)^{-1} \left(\V{\varphi}_a(\V{x}_*)y_* - (\varphi_a(\V{x}_*)\varphi_a(\V{x}_*)\Tra + \lambda \M{I}_d) (\M{\Sigma}_{\varphi_{a}} + \lambda \M{I}_d)^{-1} \V{\eta}_{\varphi_a}\right)
    \right\Vert_2 ^{2}  \\
    & + 
    \frac{2}{n} \E_{\V{x}_{*}}  \left\Vert\M{\Sigma}_{\varphi_{b}}^{\frac{1}{2}} (\M{\Sigma}_{\varphi_{b}} + \lambda \M{I}_{q-d})^{-1} \left(\V{\varphi}_b(\V{x}_*)y_* - (\varphi_b(\V{x}_*)\varphi_b(\V{x}_*)\Tra + \lambda \M{I}_{q-d}) (\M{\Sigma}_{\varphi_{b}} + \lambda \M{I}_{q-d})^{-1} \V{\eta}_{\varphi_b}\right) \right\Vert_2 ^{2}  \\
    & +\mathcal{O}_{p}\left(n^{-3/2}\right).\\
\end{align*}
The last two equalities follow by the same argument that the cross term vanishes due to disjointedness. Then notice that given $\phi_1(\V{x}_*) = (\V{\varphi}_a(\V{x}_*), \V{0})\Tra$, the corresponding expected random optimism for the KRR defined by $K(\cdot, \cdot) = \phi_1(\cdot)\Tra\phi_1(\cdot)$ will be the same as the one defined with kernel $K(\cdot, \cdot) = \varphi_a(\cdot)\Tra\varphi_a(\cdot)$ since all the $0$ entries vanish.
\begin{align*}
    & \mathrm{OptR}_{\M{X}}^{(\phi_1)} \\
    & =
    \frac{2}{n} \E_{\V{x}_{*}} \left\Vert \mathbf{\Sigma}_{\phi_1}^{\frac{1}{2}}\mathbf{\Sigma}_{\phi_1,\lambda}^{-1} \phi_1(\V{x}_{*})y_{*}-\left(\phi_1(\V{x}_{*})\phi_1(\V{x}_{*})\Tra+\lambda\M{I}\right)\mathbf{\Sigma}_{\phi_1,\lambda}^{-1}\V{\eta}_{\phi_1} \right\Vert_2 ^{2}  +\mathcal{O}_{p}\left(n^{-3/2}\right)\\
    & = \frac{2}{n} \E_{\V{x}_{*}}  \left\Vert\M{\Sigma}_{\varphi_{a}}^{\frac{1}{2}} (\M{\Sigma}_{\varphi_{a}} + \lambda \M{I}_d)^{-1} \left(\V{\varphi}_a(\V{x}_*)y_* - (\varphi_a(\V{x}_*)\varphi_a(\V{x}_*)\Tra + \lambda \M{I}_d) (\M{\Sigma}_{\varphi_{a}} + \lambda \M{I}_d)^{-1} \V{\eta}_{\varphi_a}\right)
    \right\Vert_2 ^{2}  + \mathcal{O}_{p}\left(n^{-3/2}\right)\\
    & = \mathrm{OptR}_{\M{X}}^{(\varphi_a)}.
\end{align*}
Hence, we finally obtain the desired decomposition results
\begin{align*}
    \mathrm{OptR}_{\M{X}}^{(\phi)} & = 
    \mathrm{OptR}_{\M{X}}^{(\varphi_a)} + \mathrm{OptR}_{\M{X}}^{(\varphi_b)}\\
    & = \mathrm{OptR}_{\M{X}}^{(\phi_1)} + \mathrm{OptR}_{\M{X}}^{(\phi_2)}
\end{align*}
\end{proof}

Now we are ready to present the KRR optimism results with a stationary kernel. Recall that for stationary kernel
\begin{equation*}
    K (\V{x}_i, \V{x}_j) = \varsigma(|\V{x}_i- \V{x}_j|)
\end{equation*}
Bochner's theorem states that $K$ can be represented as the Fourier transform of a finite nonnegative measure $\mu$ defined on $\Real^p$ as:
\begin{equation*}
    \M{K} (\V{x}_i, \V{x}_j) = \int_{\Real^p} \exp{\left(\mathrm{i} \V{\omega}\Tra (\V{x}_i - \V{x}_j)\right)} d\mu(\V{\omega})
\end{equation*}
where $\mathrm{i}$ is the imaginary unit. Let $p(\V{\omega})$ be the probability density defined by the measure $\mu$ and define a size $D$ feature mapping $ z(\V{x})$ as:
\begin{equation}
    z(\V{x}) = \frac{1}{\sqrt{D}}\left(z_{\V{\omega}_1}(\V{x}),z_{\V{\omega}_2}(\V{x}), \ldots, z_{\V{\omega}_D}(\V{x}) \right)\Tra \in \Real^{2D}
    \label{eq:stationary_approx_feature}
\end{equation} 
where $z_{\V{\omega}_j}(\V{x})$ is
\begin{equation*}
    z_{\V{\omega}_j}(\V{x}) = \left(\cos{\V{\omega}_j\Tra\V{x}},\, \sin{\V{\omega}_j\Tra\V{x}} \right)\Tra \in \Real^{2}
\end{equation*}
and $\V{\omega}_j$ are i.i.d.\@ draws from $p$ for $j = 1,\ldots, D$. Then the following theorem shows that the KRR optimism for a stationary kernel $K(\cdot, \cdot)$ can be approximated by the optimism of the inner-product kernel induced by the feature mapping \eqref{eq:stationary_approx_feature} (i.e., $K_z(\cdot, \cdot) = z(\cdot)\Tra z(\cdot)$), up to an error term of order $D$.

\begin{theorem}
(Expected Optimism of Kernel Ridge Regression with Stationary Kernel)
\label{thm:thm5} Under Assumption A3 of \cite{luo2025optimism}, the expected optimism $\mathrm{OptR}_{\M{X}}^{(K)}$ for the kernel ridge regression with stationary kernel $K (\V{x}_i, \V{x}_j) = \varsigma(|\V{x}_i- \V{x}_j|)$ can be approximated as:
\begin{equation}
    \mathrm{OptR}_{\M{X}}^{(K)} = \mathrm{OptR}_{\M{X}}^{(z)} + \mathcal{O}_p(\sqrt{\frac{\log D}{D}})
    \label{eq:stationary_kernel_optimism}
\end{equation}
where $\mathrm{OptR}_{\M{X}}^{(z)}$ is the expected optimism for the kernel ridge regression defined by $K_z(\cdot, \cdot) = z(\cdot)\Tra z(\cdot)$. By Corollary \ref{cor:additive_kernel_bound_simplify}
\begin{equation*}
    \mathrm{OptR}_{\M{X}}^{(K)} = \sum_{j=1}^D \mathrm{OptR}_{\M{X}}^{(z_{\V{\omega}_j})} + \mathcal{O}_p(\sqrt{\frac{\log D}{D}})
\end{equation*}
where $\mathrm{OptR}_{\M{X}}^{(z_{\V{\omega}_j})}$ is the expected optimism for the kernel ridge regression defined by $K_{z_{\V{\omega}_j}}(\cdot, \cdot) = z_{\V{\omega}_j}(\cdot)\Tra z_{\V{\omega}_j}(\cdot)$.
\end{theorem}

\begin{proof}
Bochner's theorem guarantees that the probability density $p(\V{\omega})$ is well defined given the existence of the finite nonnegative measure $\mu$ on $\Real^p$, thus  i.i.d.\@ draws of $\V{\omega}$ are also well defined. Using the feature mapping $z(\V{x})$ from \eqref{eq:stationary_approx_feature}, Claim 1 in \citet{rahimi2007random} establishes that $z(\V{x}_i)\Tra z(\V{x}_j)$ converges uniformly to the stationary kernel $\M{K}(\V{x}_i, \V{x}_j)$:
\begin{equation*}
    \mathbb{P}\left[ \underset{\V{x}_i,\V{x}_j \in \mathcal{M}}{\sup} |z(\V{x}_i)\Tra z(\V{x}_j) - \M{K}(\V{x}_i, \V{x}_j)| \geq \epsilon \right] \leq C \epsilon^{-2} \exp{\left(-\frac{D\epsilon^2}{4(p+2)} \right)}
\end{equation*}
where $C$ is a constant of size of $\mathcal{M}$, $\epsilon > 0$, and $\mathcal{M} \subseteq \Real^p$ is a compact subset of $\Real^p$. Consequently, the approximation error can be bounded as follows:
\begin{equation*}
     K(\V{x}_i, \V{x}_j) = z(\V{x}_i)\Tra z(\V{x}_j) + \mathcal{O}_p\left(\sqrt{\frac{\log D}{D}}\right)
\end{equation*}
where the result is obtained by setting $\epsilon = \alpha \sqrt{\frac{\log D}{D}}$ for some constant $\alpha > 2 \sqrt{p}$. As a result, the optimism induced by the inner-product kernel $K_z(\cdot, \cdot) = z(\cdot)\Tra z(\cdot)$ approximates the optimism $\mathrm{OptR}_{\M{X}}^{(K)}$ for the stationary kernel $K (\V{x}_i, \V{x}_j)$, with an error term of order $D$ as desired, i.e., 
\begin{equation*}
    \mathrm{OptR}_{\M{X}}^{(K)} = \mathrm{OptR}_{\M{X}}^{(z)} + \mathcal{O}_p(\sqrt{\frac{\log D}{D}}).
\end{equation*}
Also, since $z(\V{x})$ is the sum of  disjoint terms, i.e., 
\begin{equation*}
    z(\V{x}) = \frac{1}{\sqrt{D}}\begin{pmatrix}z_{\V{\omega}_1}(\V{x})\\
    0\\
    \vdots\\
    0
    \end{pmatrix} + 
    \frac{1}{\sqrt{D}}\begin{pmatrix} 0\\
    z_{\V{\omega}_2}(\V{x})\\
    0\\
    \vdots\\
    0
    \end{pmatrix} + \cdots+
    \frac{1}{\sqrt{D}}\begin{pmatrix}0\\
    \vdots\\
    0\\
    z_{\V{\omega}_D}(\V{x})
    \end{pmatrix},
\end{equation*}
Corollary \ref{cor:additive_kernel_bound_simplify} implies 
\begin{equation*}
    \mathrm{OptR}_{\M{X}}^{(K)} = \sum_{j=1}^D \mathrm{OptR}_{\M{X}}^{(z_{\V{\omega}_j})} + \mathcal{O}_p(\sqrt{\frac{\log D}{D}}).
\end{equation*}
The coefficient $\frac{1}{\sqrt{D}}$ vanishes because for any constant $c>0$ and feature mapping $\phi$ 
\begin{align*}
    \mathrm{OptR}_{\M{X}}^{(c \phi)} & = 
    \frac{2}{n} \E_{\V{x}_{*}}\left[\left\Vert c^{-1}\left(\mathbf{\Sigma}_{\phi}^{\frac{1}{2}}\mathbf{\Sigma}_{\phi,\lambda}^{-1}\right) c\left[\phi(\V{x}_{*})y_{*}-\left(\phi(\V{x}_{*})\phi(\V{x}_{*})\Tra+\lambda\M{I}\right)\mathbf{\Sigma}_{\phi,\lambda}^{-1}\V{\eta}_{\phi}\right]\right\Vert_2 ^{2}\right] +\mathcal{O}_{p}\left(n^{-3/2}\right) \\
 & = 
    \frac{2}{n} \E_{\V{x}_{*}}\left[\left\Vert \mathbf{\Sigma}_{\phi}^{\frac{1}{2}}\mathbf{\Sigma}_{\phi,\lambda}^{-1}\left[\phi(\V{x}_{*})y_{*}-\left(\phi(\V{x}_{*})\phi(\V{x}_{*})\Tra+\lambda\M{I}\right)\mathbf{\Sigma}_{\phi,\lambda}^{-1}\V{\eta}_{\phi}\right]\right\Vert_2 ^{2}\right] +\mathcal{O}_{p}\left(n^{-3/2}\right) \\
& = \mathrm{OptR}_{\M{X}}^{( \phi)}.
\end{align*}
\end{proof}

\subsection{Proof of Lemma \ref{lemma:lemma_ensemble_CP}}
\label{sec:proof_CPensemble_lemma5_1}
Given $\bar{\T{B}}=\frac{1}{K}\sum_{k=1}^{K}\T{B}^{(k)}$, its vectorization is
\begin{align*}
    \vec{(\bar{\T{B}})} &= \frac{1}{K}\sum_{k=1}^{K}\vec{(\T{B}^{(k)})} \\
    & = \frac{1}{K}\sum_{k=1}^{K} \sum_{r=1}^{R_{k}} \V{v}^{(k,r)},
\end{align*}
where 
$\V{v}^{(k,r)} = {\V{\beta}}_{1}^{(k,r)}\otimes\cdots\otimes{\V{\beta}}_{M}^{(k,r)} \in \Real^{\prod_m I_m}$ is the vectorized rank-1 tensor component. The definition of $\M{G}^{(k)}$ (in \eqref{eq:ensemble_feature_map}) implies that 
\begin{align*}
    \M{G} & = \left(\M{G}^{(1)}, \M{G}^{(2)}, \dots, \M{G}^{(K)} \right)\\
    & = \left(\V{v}^{(1,1)}, \dots, \V{v}^{(1,R_1)}, \dots, \V{v}^{(K,1)}, \dots \V{v}^{(K,R_K)} \right) \in \Real^{\prod_m I_m \times \sum_k R_k}.
\end{align*}
Consequently, we see that  $\vec{(\bar{\T{B}})} \in \text{col}(\M{G})$ where $\text{col}(\M{G})$ is the column space of matrix $\M{G}$. Assuming that $\sum_k R_k < \prod_m I_m$, we can denote the basis of $\text{col}(\M{G})$ as:
\begin{equation*}
    \{\V{u}_1, \V{u}_2, \dots, \V{u}_{R_{\rm ens}}\}
\end{equation*}
where $\V{u}_r \in \Real^{\prod_m I_m}$ for $r = 1, \dots, R_{\rm ens}$
and $R_{\rm ens}$ is the column rank of $\M{G}$. Note that when $\M{G}^{(k)}$ is full rank, $\max_k R_k \leq R_{\rm ens} \leq \sum_k R_k$. Then the double sum over $\V{v}^{(k,r)}$ can be expressed as a linear combination of $\V{u}_r$, i.e.,
\begin{align*}
    \vec{(\bar{\T{B}})} &= \frac{1}{K}\sum_{k=1}^{K} \sum_{r=1}^{R_{k}} \V{v}^{(k,r)}\\
    & = \frac{1}{K}\sum_{r=1}^{R_{\rm ens}} a_r \V{u}_r.
\end{align*}
Consequently, $\bar{\T{B}}$ has a rank-$R_{\rm ens}$ CP decomposition
\begin{equation*}
    \bar{\T{B}} = \frac{1}{K}\sum_{r=1}^{R_{\rm ens}}\bar{\V{\beta}}_{1}^{(r)}\circ\cdots\circ\bar{\V{\beta}}_{M}^{(r)}
\end{equation*}
where $\vec{(\bar{\V{\beta}}_{1}^{(r)}\circ\cdots\circ\bar{\V{\beta}}_{M}^{(r)})} = a_r \V{u}_r$.
\qed

\subsection{Proof of Theorem \ref{thm:thm_CP_ensemble_bound}}
\label{sec:proof_of_thm_5_1}
Each ensemble feature map $\phi^{(k)}$ for $k = 1, \dots, K$ can be expressed as
\begin{equation}
    \phi^{(k)}(\T{X}) = (\M{G}^{(k)}){\Tra} \vec{(\T{X})},
    \label{eq:ensemble_feature_map}
\end{equation}
where $\M{G}^{(k)} = \left(\V{v}^{(k,1)}, \dots, \V{v}^{(k,R_k)} \right) \in \Real^{\prod_m I_m \times R_k}$ (here $\M{G}^{(k)}$ has full column rank by Lemma \ref{lemma:lemma31}). Similarly, the ensemble-averaged feature map $\bar{\phi}(\T{X}) \in \Real^{R_{\rm ens}}$ can be expressed as
\begin{equation*}
    \bar{\phi}(\T{X}) = \frac{1}{K} \bar{\M{G}}\Tra \vec{(\T{X})},
\end{equation*}
where $\bar{\M{G}} = \left(\V{u}_1, \dots, \V{u}_{R_{\rm ens}} \right) \in \Real^{\prod_m I_m \times R_{\rm ens}}$. Under the assumption that $\vec(\T{X}_{*})\sim N(\mathbf{0},\M{I}_{\prod_{m}I_{m}})$, we can express $\M{\Sigma}_{\bar{\phi}} = \mathbb{E}_{\T{X}_*}[\bar{\phi}(\T{X}_*)\bar{\phi}(\T{X}_*)\Tra]$ as
\begin{align}
    \M{\Sigma}_{\bar{\phi}} & = \mathbb{E}_{\T{X}_*}\left(\bar{\phi}(\T{X}_*)\bar{\phi}(\T{X}_*)\Tra\right) \nonumber \\
    & = \frac{1}{K^2}\mathbb{E}_{\T{X}_{*}}\left(\bar{\M{G}}\Tra \vec(\T{X}_{*})\vec(\T{X}_{*})\Tra \bar{\M{G}}\right) \nonumber \\
    & = \frac{1}{K^2}\bar{\M{G}}\Tra \mathbb{E}_{\T{X}_{*}}\left(\vec(\T{X}_{*})\vec(\T{X}_{*})\Tra\right)\bar{\M{G}} \nonumber \\
    & = \frac{1}{K^2}\bar{\M{G}}\Tra \bar{\M{G}} \in \Real^{R_{\rm ens} \times R_{\rm ens}}.
    \label{eq:ensemble_averaged_cov}
\end{align}
Let the extended feature mapping $\varphi(\T{X}) \in \Real^{\sum_k R_k}$ be the  concatenation the individual ensemble feature maps $\phi^{(k)}(\T{X})$:
\begin{align}
    \varphi(\T{X}) & = \frac{1}{K}\left(\phi^{(1)}(\T{X})\Tra, \dots, \phi^{(K)}(\T{X})\Tra \right)\Tra \label{eq:ensemble_extended_feature} \\
    & = \frac{1}{K}\begin{pmatrix}\phi^{(1)}(\T{X})\\
    \V{0}\\
    \vdots\\
    \V{0}
    \end{pmatrix} + 
    \cdots+\frac{1}{K}
    \begin{pmatrix}\V{0}\\
    \vdots\\
    \V{0}\\
    \phi^{(K)}(\T{X})
    \end{pmatrix}.
    \nonumber
\end{align}
The extended feature mapping  $\varphi(\T{X})$ can also be expressed using the matrix-vector product:
\begin{equation*}
    \varphi(\T{X}) = \frac{1}{K} \M{G}\Tra \vec{(\T{X})},
\end{equation*}
where $\M{G} = \left(\M{G}^{(1)}, \M{G}^{(2)}, \dots, \M{G}^{(K)} \right) \in \Real^{\prod_m I_m \times \sum_k R_k}$. Similarly, the corresponding covariance $\M{\Sigma}_{\varphi}$ can be written as:
\begin{align}
    \M{\Sigma}_{\varphi} & = \mathbb{E}_{\T{X}_*}\left(\varphi(\T{X}_*)\varphi(\T{X}_*)\Tra\right) \nonumber \\
    & = \frac{1}{K^2}\mathbb{E}_{\T{X}_{*}}\left( {\M{G}}\Tra \vec(\T{X}_{j})\vec(\T{X}_{j})\Tra {\M{G}} \right)\nonumber \\
    & = \frac{1}{K^2}{\M{G}}\Tra {\M{G}} \in \Real^{\sum_k R_k \times \sum_k R_k}.
    \label{eq:ensemble_extended_cov}
\end{align}
According to the argument in Section~\ref{sec:proof_CPensemble_lemma5_1}, we have $\text{col}(\bar{\M{G}}) \subseteq \text{col}(\M{G})$. Therefore, there is a coefficient matrix $\M{W} \in \Real^{\sum_k R_k \times R_{\rm ens}}$ such that 
\begin{equation*}
    \bar{\M{G}} = \M{G} \M{W}.
\end{equation*}
Assume in addition that $\M{W}$ is contractive, i.e.
\begin{equation*}
    \M{W}\Tra \M{W}\preceq \M{I}_{R_{\rm ens}}.
\end{equation*}
Here we impose this contractive assumption with the belief that the ensemble-average feature map is an non-expansive representation of the weighted combination of individual components. In practice, this condition can be achieved by standardizing each ensemble feature (i.e., columns of $\M{G}$).
Consequently,
\begin{equation*}
    \bar{\M{G}}\Tra \bar{\M{G}} = \M{W}\Tra {\M{G}}\Tra {\M{G}} \M{W},
\end{equation*}
which implies
\begin{equation*}
    \M{\Sigma}_{\bar{\phi}} = \M{W}\Tra \M{\Sigma}_{\varphi} \M{W}.
\end{equation*}
Denote the QR decomposition of $\M{W}$ as $\M{W} = \M{Q}\M{R}$ where $\M{Q} \in \Real^{\sum_k R_k \times R_{\rm ens}}$ has orthonormal columns (i.e., $\M{Q}\Tra \M{Q} = \M{I}$) and $\M{R} \in \Real^{R_{\rm ens} \times R_{\rm ens}}$ is upper triangular. Consequently, 
\begin{equation}
    \M{\Sigma}_{\bar{\phi}} = \M{R}\Tra \M{Q}\Tra \M{\Sigma}_{\varphi} \M{Q} \M{R}.
    \label{eq:QR_decomposition}
\end{equation}

Recall from Section \ref{sec:CP_regression} that under Assumption~\ref{Assumption:assumption_CP_ensemble}, the expected optimism for each learner (fitted with training subset $\mathcal{D}_k$) is
\begin{equation}
\mathrm{OptR}_{\T{X}}^{\mathrm{(k)}}
=
\frac{2\left(\sigma^{2}+\frac{\lambda^2{v}_{1}^{(k)}}{({v}_{1}^{(k)}+\lambda)^2}\right)}{n_k}
\sum_{r=1}^{R_{k}}\frac{({v}_{r}^{(k)})^{2}}{({v}_{r}^{(k)}+\lambda)^{2}}
\;+\;\mathcal{O}_p(n_{k}^{-3/2}),
    \label{eq:ensemble_feature_optimism}
\end{equation}
where ${v}_{r}^{(k)}$ for $r = 1, \dots, R_k$ are the eigenvalues of $\M{\Sigma}_{\phi^{(k)}}  = \mathbb{E}_{\T{X}_*}[\phi^{(k)}(\T{X}_*)\phi^{(k)}(\T{X}_*)\Tra] = {\M{G}^{(k)}}\Tra {\M{G}^{(k)}} \in \Real^{R_k \times R_k}$ for $k=1,\dots, K$. Now Lemma \ref{lemma:lemma_ensemble_CP} demonstrates that the ensemble-averaged estimator $\bar{\T{B}}$ resembles a rank-$R_{\rm ens}$ CP decomposition with the corresponding CP feature mapping $\bar{\phi}(\T{X})$. And Section \ref{sec:CP_regression} shows that we can examine its expected random optimism under the tensor KRR structure with kernel $K(\cdot, \cdot) = \bar{\phi}(\cdot)\Tra \bar{\phi}(\cdot)$, which we define as $\mathrm{OptR}_{\T{X}}^{(\bar{\phi})}$. Hence, $\mathrm{OptR}_{\T{X}}^{\mathrm{(ens)}} \equiv \mathrm{OptR}_{\T{X}}^{(\bar{\phi})}$, and we only need to show that
\begin{equation*}
    \mathrm{OptR}_{\T{X}}^{(\bar{\phi})}
    \leq
    \sum_{k=1}^{K} \frac{n_k}{n} \mathrm{OptR}_{\T{X}}^{\mathrm{(k)}} + o(1).
\end{equation*}

From Section \ref{sec:CP_regression},
\begin{equation*}
    \mathrm{OptR}_{\T{X}}^{(\bar{\phi})}
    =
    \frac{2\left(\sigma^{2}+\frac{\lambda^2\bar{v}_{1}}{(\bar{v}_{1}+\lambda)^2}\right)}{n}
    \sum_{r=1}^{R_{\rm ens}}\frac{(\bar{v}_{r})^{2}}{(\bar{v}_{r}+\lambda)^{2}}
    \;+\;\mathcal{O}_p(n^{-3/2}),
\end{equation*}
where $\{\bar{v}_r\}_{r=1}^{R_{\rm ens}}$ are the eigenvalues of $\M{\Sigma}_{\bar{\phi}}$ in \eqref{eq:ensemble_averaged_cov} with $\bar{v}_1 \geq \bar{v}_2 \geq \dots \geq \bar{v}_{R_{\rm ens}} > 0$. Similarly, with the extended feature mapping $\varphi(\T{X})$ in \eqref{eq:ensemble_extended_feature}, we can express its expected random optimism as
\begin{equation*}
    \mathrm{OptR}_{\T{X}}^{{(\varphi)}}
    =
    \frac{2\left(\sigma^{2}+\frac{\lambda^2{u}_{1}}{({u}_{1}+\lambda)^2}\right)}{n}
    \sum_{r=1}^{\sum_k R_k}\frac{({u}_{r})^{2}}{({u}_{r}+\lambda)^{2}}
    \;+\;\mathcal{O}_p(n^{-3/2}),
\end{equation*}
where $\{{u}_r\}_{r=1}^{\sum_k R_k}$ are the eigenvalues of $\M{\Sigma}_{{\varphi}}$ in \eqref{eq:ensemble_extended_cov} with ${u}_1 \geq {u}_2 \geq \dots \geq {u}_{\sum_k R_k} \geq 0$. 

Since $\M{W}=\M{Q}\M{R}$ and $\M{W}\Tra \M{W}\preceq \M{I}_{R_{\rm ens}}$, we have
\[
\M{R}\Tra \M{R}=\M{W}\Tra \M{W}\preceq \M{I}_{R_{\rm ens}}.
\]
Define
\[
\widetilde{\M{Q}}
=
\begin{pmatrix}
\M{Q}\M{R}\\[3pt]
(\M{I}_{R_{\rm ens}}-\M{R}\Tra \M{R})^{1/2}
\end{pmatrix}
\in \Real^{(\sum_k R_k+R_{\rm ens})\times R_{\rm ens}},
\qquad
\widetilde{\M{\Sigma}}_{\varphi}
=
\begin{pmatrix}
\M{\Sigma}_{\varphi} & \M{0}\\
\M{0} & \M{0}
\end{pmatrix}.
\]
Then
\[
\widetilde{\M{Q}}\Tra \widetilde{\M{Q}}
=
\M{R}\Tra \M{Q}\Tra \M{Q}\M{R}
+
(\M{I}_{R_{\rm ens}}-\M{R}\Tra \M{R})
=
\M{I}_{R_{\rm ens}},
\]
and
\[
\widetilde{\M{Q}}\Tra \widetilde{\M{\Sigma}}_{\varphi}\widetilde{\M{Q}}
=
\M{R}\Tra \M{Q}\Tra \M{\Sigma}_{\varphi}\M{Q}\M{R}
=
\M{\Sigma}_{\bar{\phi}}.
\]
Hence $\M{\Sigma}_{\bar{\phi}}$ is an orthogonal compression of $\widetilde{\M{\Sigma}}_{\varphi}$. By the Cauchy interlacing theorem,
\[
\bar v_r \le u_r,\qquad r=1,\dots,R_{\rm ens}.
\]

Now define
\[
\psi(v)=\frac{v^2}{(v+\lambda)^2},
\qquad
h(v)=\frac{\lambda^2 v}{(v+\lambda)^2}.
\]
Since $\psi$ is increasing on $[0,\infty)$, the interlacing inequality implies
\begin{equation}
\sum_{r=1}^{R_{\rm ens}} \psi(\bar v_r)
\le
\sum_{r=1}^{R_{\rm ens}} \psi(u_r)
\le
\sum_{r=1}^{\sum_k R_k} \psi(u_r).
\label{eq:interlacing_psi_bound}
\end{equation}
Moreover,
\[
0\le h(v)\le \frac{\lambda}{4},\qquad v\ge 0.
\]
Therefore,
\begin{align*}
\mathrm{OptR}_{\T{X}}^{(\bar{\phi})}
&=
\frac{2\left(\sigma^{2}+h(\bar v_{1})\right)}{n}
\sum_{r=1}^{R_{\rm ens}}\psi(\bar v_r)
+\mathcal{O}_p(n^{-3/2})\\
&\le
\frac{2\left(\sigma^{2}+\lambda/4\right)}{n}
\sum_{r=1}^{\sum_k R_k}\psi(u_r)
+\mathcal{O}_p(n^{-3/2})\\
&=
\mathrm{OptR}_{\T{X}}^{(\varphi)}
+
\frac{2\left(\lambda/4-h(u_1)\right)}{n}
\sum_{r=1}^{\sum_k R_k}\psi(u_r)
+\mathcal{O}_p(n^{-3/2})\\
&\le
\mathrm{OptR}_{\T{X}}^{(\varphi)}
+
\frac{\lambda}{2n}
\sum_{r=1}^{\sum_k R_k}\psi(u_r)
+\mathcal{O}_p(n^{-3/2}).
\end{align*}
Since $0\le \psi(u_r)\le 1$, it follows that
\begin{equation*}
\mathrm{OptR}_{\T{X}}^{(\bar{\phi})}
\le
\mathrm{OptR}_{\T{X}}^{(\varphi)}
+
\frac{\lambda}{2n}\sum_{k=1}^K R_k
+\mathcal{O}_p(n^{-3/2}).
\end{equation*}
In particular, if $\sum_{k=1}^K R_k$ is fixed and $\lambda=\mathcal{O}(1)$ (more generally, if $\lambda \sum_{k=1}^K R_k=o(n)$), then
\begin{equation*}
\mathrm{OptR}_{\T{X}}^{(\bar{\phi})}
\le
\mathrm{OptR}_{\T{X}}^{(\varphi)} + o(1).
\end{equation*}
Now, given the disjoint additive structure of $\varphi(\T{X})$ in \eqref{eq:ensemble_extended_feature}, results in Corollary \ref{cor:additive_kernel_bound_simplify} show that
\begin{align*}
    \mathrm{OptR}_{\T{X}}^{{(\varphi)}}
    &=
    \sum_{k=1}^K \mathrm{OptR}_{\T{X}}^{\phi^{(k)}} \\
    &=
    \sum_{k=1}^K \frac{n_k}{n}\,\mathrm{OptR}_{\T{X}}^{\mathrm{(k)}} + \mathcal{O}_p(n^{-3/2}),
\end{align*}
where $\mathrm{OptR}_{\T{X}}^{\phi^{(k)}}$ is the expected random optimism for the ensemble CP feature $\phi^{(k)}$ with respect to the full training set $\mathcal{D}$ (of size $n$), and $0<\frac{n_k}{n}<1$ implies $\mathcal{O}_p(n_k^{-3/2})=\mathcal{O}_p(n^{-3/2})$. Hence, we finally obtain
\begin{equation*}
    \mathrm{OptR}_{\T{X}}^{\mathrm{(ens)}}
    \equiv
    \mathrm{OptR}_{\T{X}}^{(\bar{\phi})}
    \le
    \sum_{k=1}^K \frac{n_k}{n}\,\mathrm{OptR}_{\T{X}}^{\mathrm{(k)}} + o(1).
\end{equation*}
\qed

\subsection{TRMA and Optimism-Based Risk Estimates}
\label{sec:trma-optimism}

In the main paper we considered scalar-on-tensor regression
\begin{equation*}
    y_i = \langle\!\langle\T{X}_i,\T{B}\rangle\!\rangle + \epsilon_i
\end{equation*}
where $\T{B} \in \Real^{I_1 \times \cdots \times I_M}$ is the tensor coefficient, $\vec(\T{X}_i) \sim \mathrm{N}(0, \M{I}_{\prod_m I_m})$, and $\epsilon_i$ are i.i.d.\@ additive mean-zero Gaussian noises independent of $\T{X}_i$. For a fixed candidate target rank $r \in \mathcal{R} = \{1,\dots,R_{\max}\}$, let $\widehat f_r = \langle\!\langle\T{X},\hat{\T{B}}_r\rangle\!\rangle$ denote the CP/Tucker tensor regression estimator defined in \eqref{eq:tensor_regression_model_CP} and \eqref{eq:tensor_regression_model_tucker} the main text. We recall that the quantity of interest under ``Random-$\T{X}$" is the \emph{prediction error}
\begin{equation}
R(r) = \E_{(\mathcal{D}, \, \T{X}_*,y_*)}\bigl[ (y_* - \widehat f_r^{(\mathcal{D})}(\T{X}_*))^2 \bigr]
\label{eq:trma-target-risk}
\end{equation}
where $\mathcal{D} = \{(\T{X}_i,y_i)\}_{i=1}^n$ denotes the training sample and $\widehat f_r^{(\mathcal{D})}$ makes the explicit dependence of the fitted tensor regressor on the data. This is the same target as in the optimism expressions in the main paper.

\subsubsection{Two data-driven estimators of prediction risk}

For a fixed rank $r$, define the (empirical) training error as:
\begin{equation*}
    \widehat R_{\mathrm{train}}(r) = \frac{1}{n} \sum_{i=1}^n \bigl( y_i - \widehat f_r^{(\mathcal{D})}(\T{X}_i) \bigr)^2
\end{equation*}
In the Section~\ref{sec:CP_regression} of the main manuscript, we show that, under the Gaussian random-design assumption, the ``Random-$\T{X}$" optimism
\begin{equation*}
    \widehat{\mathrm{Opt}}(r) = \E_{(\T{X}_*,y_*)}\bigl[ (y_* - \widehat f_r^{(\mathcal{D})}(\T{X}_*))^2 \mid \mathcal{D} \bigr]
   - \widehat R_{\mathrm{train}}(r)
\end{equation*}
admits a computable plug-in approximation. The resulting optimism-corrected risk estimator
\begin{equation}
\widehat R_{\mathrm{opt}}(r) = \widehat R_{\mathrm{train}}(r) + \widehat{\mathrm{Opt}}(r)
\label{eq:opt-risk-est}
\end{equation}
is a consistent estimator of \eqref{eq:trma-target-risk}, uniformly over $r \in \mathcal{R}$, provided the eigenvalues of the feature covariance satisfy the corresponding conditions stated in Assumptions \ref{Assumption:assumption_cp}.

Similarly, the $K$-fold cross-validation (as used in TRMA) is defined as:
\begin{equation}
\widehat R_{\mathrm{CV}}(r) = \frac{1}{K} \sum_{k=1}^K \frac{1}{|N_k|}
   \sum_{i \in N_k} \bigl( y_i - \widehat f_r^{(-k)}(\T{X}_i) \bigr)^2
\label{eq:cv-risk-est}
\end{equation}
where $\{N_1,\dots,N_K\}$ is a partition of $\{1,\dots,n\}$ and $\widehat f_r^{(-k)}$ is the rank-$r$ tensor regressor fitted on the data excluding fold $k$. For the kernel tensor regressors we consider, standard arguments for linear smoothers under ``Random-$\T{X}$" (e.g. bounded operator norm of the hat matrix and Lipschitz dependence on the data) can imply that $\widehat R_{\mathrm{CV}}(r) \to R(r)$ in probability, uniformly over $r \in \mathcal{R}$, under the same design and noise assumptions. Then the following proposition demonstrates that these two estimators \eqref{eq:opt-risk-est} and \eqref{eq:cv-risk-est} are asymptotically equivalent.

\begin{proposition}(Asymptotic equivalence of optimism and CV risks)
\label{prop:opt-vs-cv}
Assume the Gaussian random-design model and the spectral/regularization conditions used in the Assumptions \ref{Assumption:assumption_cp} (in particular, the population feature covariance corresponding to each candidate rank $r$ is positive definite, and the regularization parameter $\lambda_n$ satisfies $\lambda_n \to 0$ and $n \lambda_n \to \infty$). Then, for any fixed finite candidate set $\mathcal{R} = \{1,\dots,R_{\max}\}$,
\begin{equation}
\max_{r \in \mathcal{R}} \bigl| \widehat R_{\mathrm{opt}}(r) - R(r) \bigr| \xrightarrow{P} 0
\quad\text{and}\quad
\max_{r \in \mathcal{R}} \bigl| \widehat R_{\mathrm{CV}}(r) - R(r) \bigr| \xrightarrow{P} 0.
\label{eq:trma_convergence_results}
\end{equation}
Consequently,
\[
\max_{r \in \mathcal{R}} \bigl| \widehat R_{\mathrm{opt}}(r) - \widehat R_{\mathrm{CV}}(r) \bigr| \xrightarrow{P} 0.
\]
\end{proposition}

\begin{proof}
The first convergence (left one in \eqref{eq:trma_convergence_results}) is exactly the finite-candidate-set version of the main optimism consistency results (Theorem~\ref{thm:thm_CP_true_rank} - \ref{thm:thm_CP_under_rank}) in the main manuscript. By construction, $\widehat R_{\mathrm{opt}}(r)
= \widehat R_{\mathrm{train}}(r) + \widehat{\mathrm{Opt}}(r)$ is an unbiased (up to $o_p(1)$) estimator of the ``Random-$\T{X}$" prediction error $R(r)$ for kernelized tensor regression when the design is Gaussian and the feature covariance is nonsingular. The uniformity over a finite set $\mathcal{R}$ is automatic.

For the cross-validation estimator, since each $\widehat f_r^{(-k)}$ is of the same form as $\widehat f_r^{(\mathcal{D})}$ but trained on $n(1-1/K)$ observations, the same ``Random-$\T{X}$" arguments imply
\begin{equation*}
    \E_{(\mathcal{D}, \, \T{X}_*,y_*)}\bigl[ (y_* - \widehat f_r^{(-k)}(\T{X}_*)\bigl] = R(r) + o_p(1)
\end{equation*}
uniformly in $k$ and $r$. Averaging over folds and applying a union bound over the finite set $\mathcal{R}$ gives the second convergence result (right one in \eqref{eq:trma_convergence_results}). Now the triangle inequality yields
\begin{align*}
    \bigl| \widehat R_{\mathrm{opt}}(r) - \widehat R_{\mathrm{CV}}(r) \bigr| & = \bigl| \widehat R_{\mathrm{opt}}(r)  - R(r) + R(r)- \widehat R_{\mathrm{CV}}(r) \bigr| \\
    & \leq  \bigl| \widehat R_{\mathrm{opt}}(r)  - R(r)\bigr| + \bigl | \widehat R_{\mathrm{CV}}(r) - R(r) \bigr|
\end{align*}
combing with \eqref{eq:trma_convergence_results}
gives the last claim.
\end{proof}

\subsubsection{Implication for TRMA}

A TRMA estimator forms the averaged predictor
\[
\widehat f_{\mathrm{TRMA}}(\T{X})
= \sum_{r \in \mathcal{R}} \omega_r \, \widehat f_r^{(\mathcal{D})}(\T{X}),
\qquad \omega_r \ge 0, \quad \sum_{r \in \mathcal{R}} \omega_r = 1,
\]
with weights
\[
\widehat \omega \in \arg\min_{\omega \in \Delta^{R_{\max}}}
\sum_{r \in \mathcal{R}} \omega_r \, \widehat R_{\mathrm{CV}}(r),
\]
where $\Delta^{R_{\max}}$ is the simplex. By Proposition~\ref{prop:opt-vs-cv}, we may replace $\widehat R_{\mathrm{CV}}(r)$ with $\widehat R_{\mathrm{opt}}(r)$ in the objective without changing the minimizer asymptotically:
\[
\sup_{\omega \in \Delta^{R_{\max}}}
\biggl| \sum_{r} \omega_r \widehat R_{\mathrm{CV}}(r)
       - \sum_{r} \omega_r \widehat R_{\mathrm{opt}}(r) \biggr|
\le \max_{r} \bigl| \widehat R_{\mathrm{CV}}(r) - \widehat R_{\mathrm{opt}}(r) \bigr|
\xrightarrow{P} 0.
\]
Hence TRMA and “optimism-weighted tensor model averaging” are optimizing the \emph{same} population criterion under the Assumption~\ref{Assumption:assumption_CP_ensemble} of the main paper. For each fixed rank, the $K$-fold cross-validation criterion used in TRMA and our optimism-corrected risk target the same population prediction error.


\section{Addition Experiment Results}
\subsection{Varying Regularization Parameters in Tensor KRR}
\label{sec:Cp_vary_lambda}
This section provides additional experimental results on optimism for the tensor KRR model under varying regularization parameters ($\lambda$). The simulation setting is identical to that in Section \ref{sec:simulation}. We test a range of $\lambda$ values from $0.1$ to $10^8$ to examine the behavior of the expected optimism. Figure \ref{fig:CP_naive_lambda_gen} displays a heatmap of these results, revealing three distinct regimes of optimism, which we discuss below.

\begin{figure}[H]
\centering
\includegraphics[width=1\textwidth, height = 8cm]{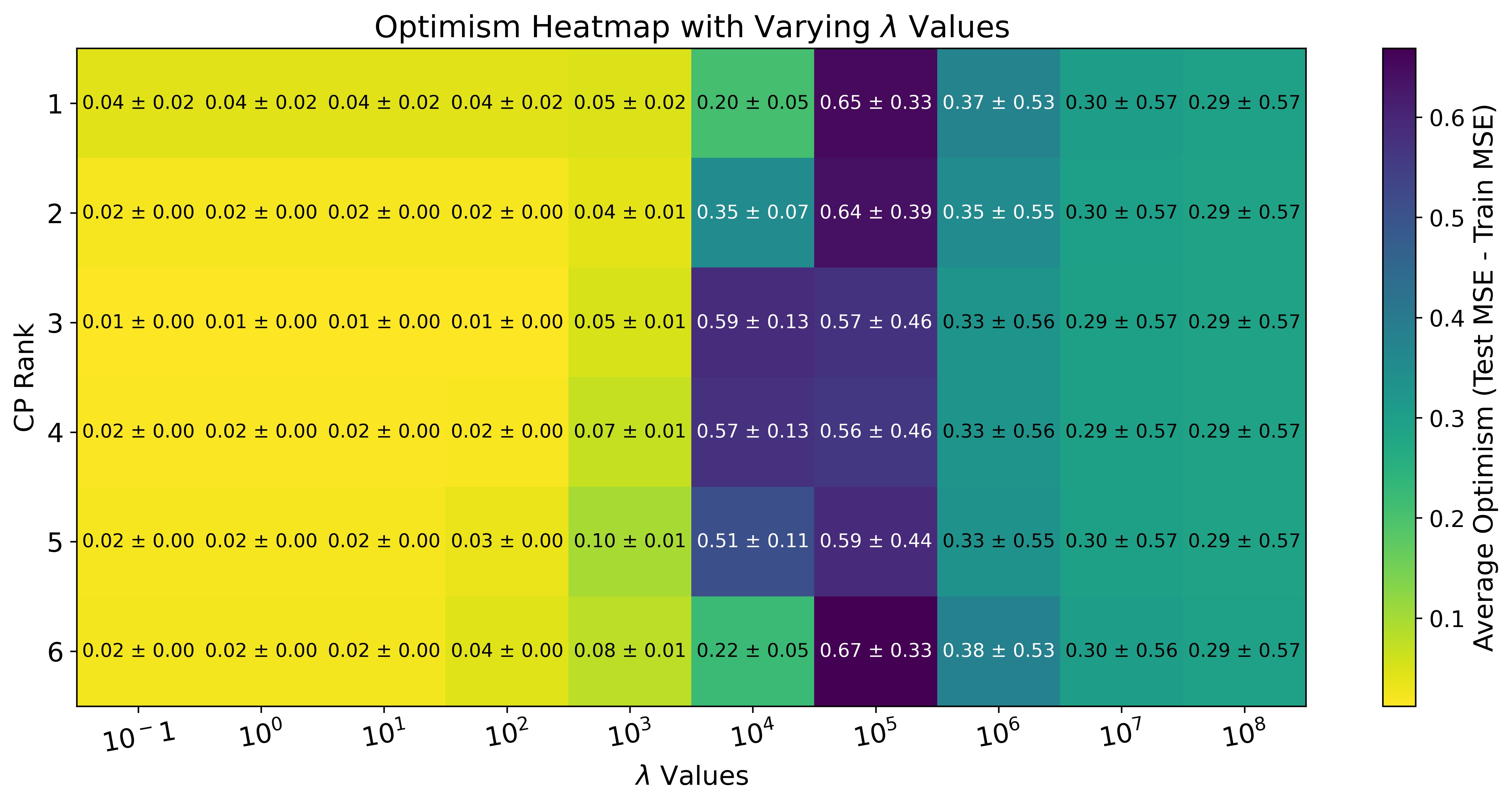}
\caption{Average optimism of tensor KRR model over $10000$ Monte Carlo runs for varying CP ranks (rows) and regularization strength (column): low regularization $\lambda = 10^{-1},\, 1,\, 10,\, 10^2$, moderate regularization $\lambda = 10^3,\,10^4,\, 10^5$, and high regularization $\lambda = 10^6,\,10^7,\,10^8$. The noise level is $5\%$ of signal standard deviation and training sample size is $n_{\text{train}} = 200$. Results are shown for the oracle case, where the CP kernel is constructed from the true tensor coefficient $\T{B}$, which has a rank of $3$.\label{fig:CP_naive_lambda_gen} }
\end{figure}

\subsubsection{Regime I: low shrinkage}
When $\lambda\ll \tilde{v}_{r}$, the conclusions of Theorem \ref{thm:thm_CP_over_rank} and \ref{thm:thm_CP_under_rank} apply directly. The expected optimism is non-decreasing whenever $R_t \neq R$. The trend observed in the left panel (with $\lambda = 10^{-5},\, 10^{-3},\, 10^{-1},\, 1$) of Figure \ref{fig:CP_naive_lambda_gen} precisely matches this result.

\subsubsection{Regime II: moderate shrinkage}

When $\lambda$ grows to the same order of magnitude as the signal strength $\tilde{v}_r$, the shrinkage effects start to be active. Each spike term will decrease as $\frac{\tilde{v}_r^2}{(\tilde{v}_r + \lambda)^2}$ is monotonically decreasing with $\lambda$. But once the penalty strength reaches the scale of the leading eigenvalue (i.e. $\lambda \approx \tilde{v}_1$), the original arguments of Theorem \ref{thm:thm_CP_over_rank} and \ref{thm:thm_CP_under_rank} will no longer hold. To see this, one can check
\[
\frac{\lambda^2\tilde{v}_1}{(\tilde{v}_1+\lambda)^2} = \frac{\lambda^2\tilde{v}_1}{\lambda^2(1+\frac{\tilde{v}_1}{\lambda})^2} = \frac{\tilde{v}_1}{(1+\frac{\tilde{v}_1}{\lambda})^2}
\]
increases with $\lambda$ and becomes $\mathcal{O}(\tilde{v}_1)$ when $\lambda \approx \tilde{v}_1$. And because for $\lambda > 0$ the sum
\[
\sum_{r=1}^{R_t}\frac{\tilde{v}_r^2}{(\tilde{v}_r + \lambda)^2} < R_t = \mathcal{O}(1)
\]
will remain bounded, the rising multiplicative coefficient 
\[
\frac{2\left(\sigma^{2}+\frac{\lambda^2\tilde{v}_1}{(\tilde{v}_1+\lambda)^2}\right)}{n}
\]
will outweigh the shrinking spikes 
\[
\sum_{r=1}^{R_{t}}\frac{\bigl(\tilde{v}_{r}\bigr)^{2}}{\bigl(\tilde{v}_{r}+\lambda\bigr)^{2}},
\]
leading to an overall increase in optimism.

\begin{figure}[H]
\centering
\includegraphics[width=17cm, height = 6cm]{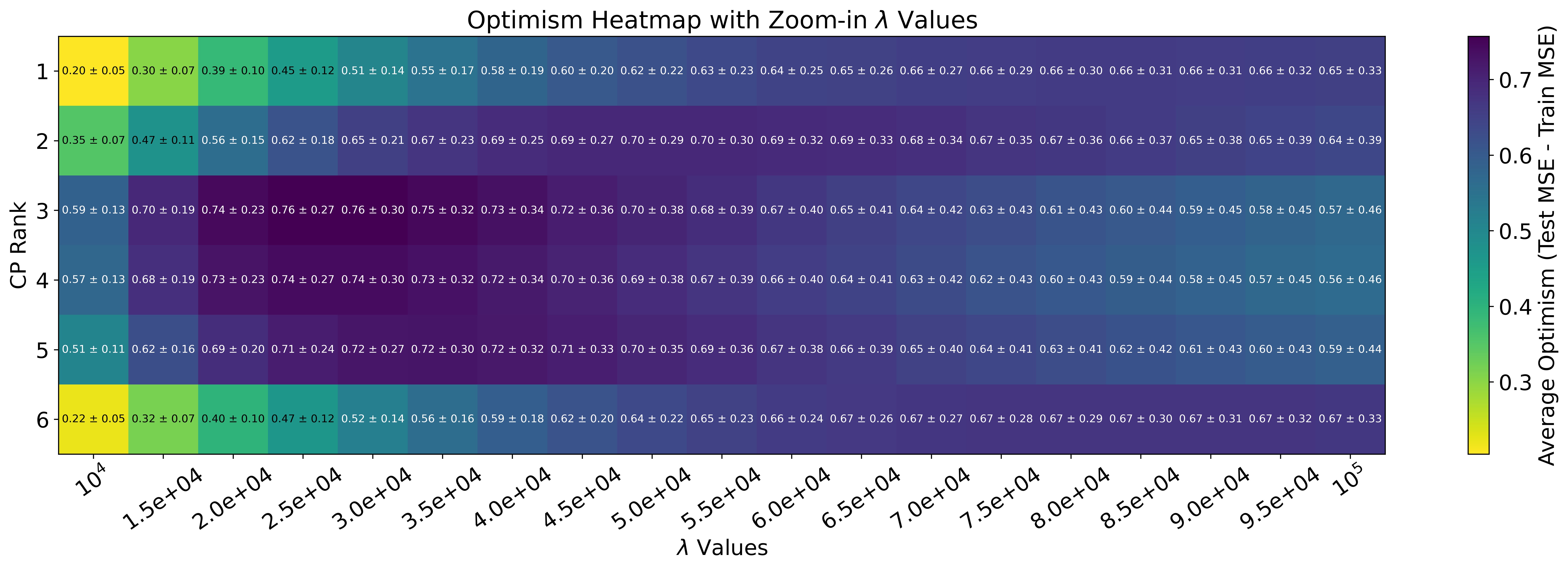}
\caption{Average optimism of tensor KRR model over $10000$ Monte Carlo runs for varying CP ranks (row) and a zoom-in moderate regularization strength regime (column) of Figure \ref{fig:CP_naive_lambda_gen}: $\lambda \in [10^4, 10^5]$ with $5000$ increments. The noise level is $5\%$ of signal standard deviation and training sample size is  $n_{\text{train}} = 200$. Results are shown for the oracle case, where the CP kernel is constructed from the true tensor coefficient $\T{B}$, which has a default rank of $3$.\label{fig:CP_naive_lambda_zoom_in}}
\end{figure}

This pattern is demonstrated in the middle panel of Figure \ref{fig:CP_naive_lambda_gen} (with $\lambda = 10^4,\,3\times10^4,\, 5\times 10^4,\, 7 \times 10^4,\, 9 \times 10^4,\, 10^5$), and the zoom-in plot of this moderate regime Figure \ref{fig:CP_naive_lambda_zoom_in}. In our simulation, we have $\tilde{v}_1^{(R_t)} \approx 10^4$ for $R_t = 3,4$ and $\tilde{v}_1^{(R_t)} \approx 4\times10^4$ for $R_t = 1,2,5,6$. We see an increasing trend of optimism as $\lambda$ approaches each rank's leading eigenvalue and gradually decreases as $\lambda$ moves away (to the high penalty region). In particular, a ``reversed'' optimism peak is observed (at $R_t = 3,4$) for $\lambda = 10^4, 3\times 10^4, 5 \times 10^4$ as they experience this increasing effect sooner than for the other ranks due to their relatively smaller leading eigenvalues.

\subsubsection{Regime III: high shrinkage}

For sufficiently large $\lambda$, the shrinkage term dominates each spike $\frac{\tilde{v}_r^2}{(\tilde{v}_r + \lambda)^2}$, and by conclusion of Remark~\ref{remark:remark_CP_lambda} we have
\[
\E_{\T{X}}\bigl[\mathrm{Opt}_{R_{\T{X}}}^{(R_{t})}\bigr]=O\!\Bigl(\lambda^{-2}\Bigr)\;\xrightarrow{\lambda \to \infty}\;0.
\]
This behavior is evident in the right panel of Figure \ref{fig:CP_naive_lambda_gen} (with $\lambda = 10^6,\,10^7,\,10^8,\,10^9$), where all six ranks collapse toward a common yellow band that
is \emph{numerically} close to $0$. The slight negative bias arises because the $O(n^{-3/2})$ Monte Carlo noise overtakes the $O(\lambda^{-2})$ signal when $\lambda\ge10^{6}$. 

\subsection{AIC and BIC Analysis}
\label{sec:AIC_BIC_tensor_regression}

\begin{figure}[H]
\captionsetup{font=small}
    \centering
    \subfloat[\centering AIC]{{\includegraphics[width=1\textwidth]{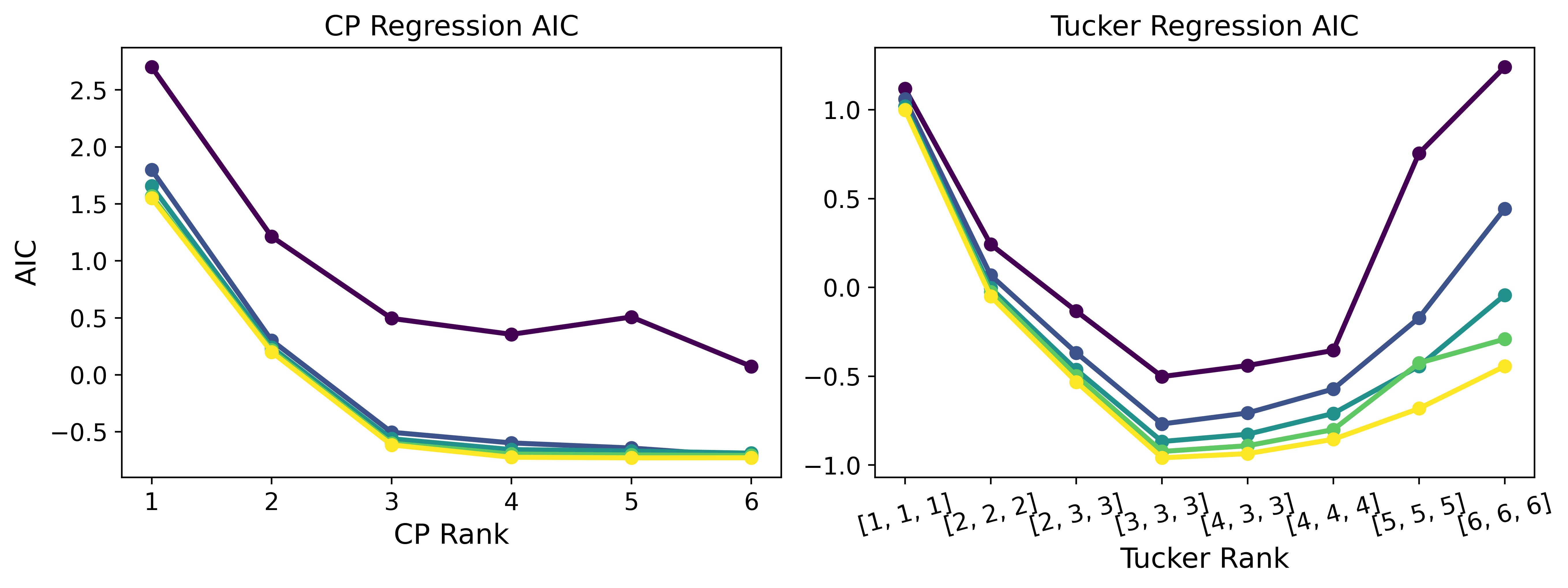} }}%
    \vspace{0.1cm}
    \subfloat[\centering BIC]{{\includegraphics[width=1\textwidth]{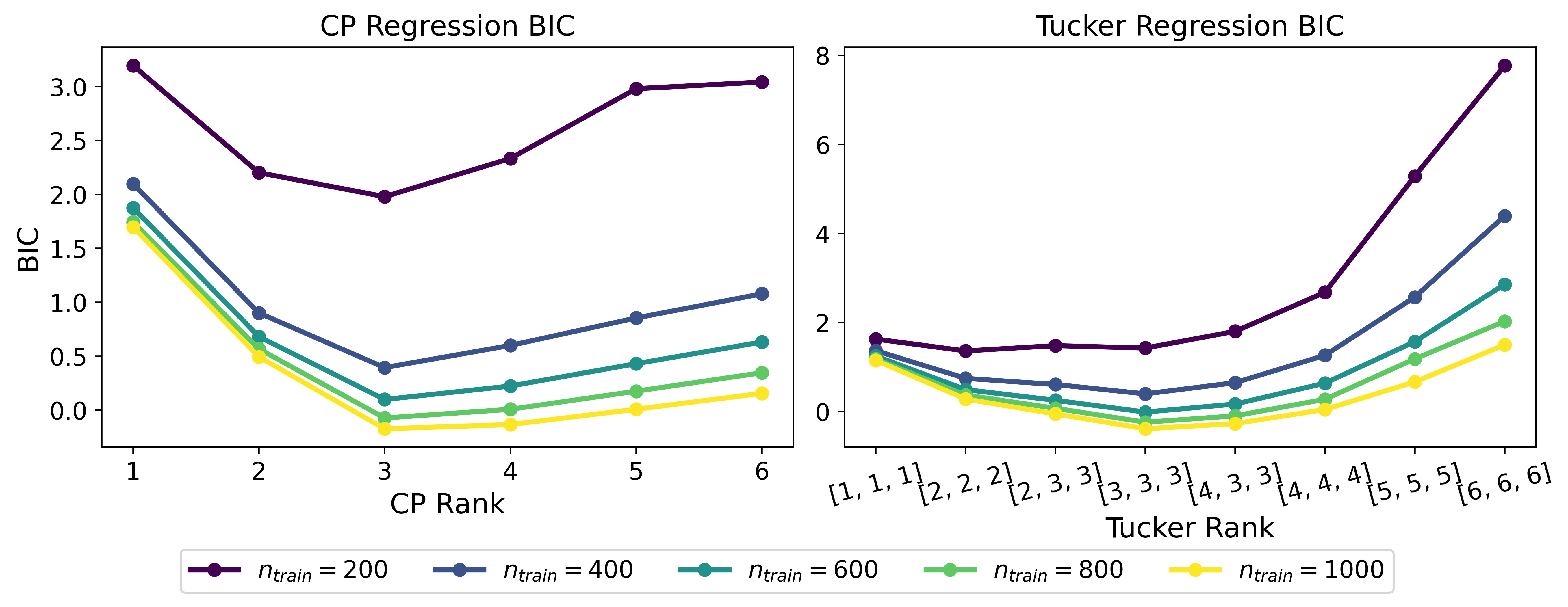} }}%
    \caption{\label{fig:aic_bic_analysis_tensor_regression}AIC and BIC selection criteria for low-rank CP and Tucker regression with varying ranks and sample size. The top row of plots (a) shows the AIC, and the bottom row of plots (b) shows the BIC. Within each row of plots, the left plot displays the results for CP regression over $10000$ MC runs, where the noise level is $5\%$ of the signal standard deviation and the true CP rank is $R = 3$. The right plot shows the results for Tucker regression (using TensorGP from \citet{yu2018tensor}) over $100$ MC runs, where the noise level is $1\%$ of the signal standard deviation and a true Tucker rank of $R = (3,3,3)$. In all plots, the x-axis represents the model rank, while different colors correspond to varying training sample sizes. The noise level for Tucker case is reduced to $1\%$ to maintain a similar noise magnitude as the CP case. The MC replicates are limited to $100$ due to the high computational complexity and our fixed computational power.}%
\end{figure}

Figure~\ref{fig:aic_bic_analysis_tensor_regression} presents the AIC and BIC measures for CP and Tucker regression models under the same settings of Section~\ref{sec:simulation}. The results highlight the inconsistency of these traditional criteria. For CP regression, AIC fails to identify the true rank at any sample size, whereas BIC succeeds. For Tucker regression, however, AIC correctly identifies the true rank in all cases, while BIC fails to do so when the training sample size is small. This inconsistent performance reinforces the superiority of our proposed optimism framework as a more reliable approach for tensor rank selection.

\subsection{CP Ensemble Regression}
\label{sec:CP_ensemble_results}
Here we evaluate the behavior of optimism for the ensemble CP regression under the same simulation setting of Section~\ref{sec:simulation}. The training sample size is  $n_{\text{train}} = 1000$ and each ensemble member is trained on a random subsample of size $n_k=200$ and a common target CP rank using the Python \texttt{tensorly} library \citep{tensorly}. The expected optimism is evaluated under the mean-squared-error (MSE) loss and computed by averaging $10000$ Monte Carlo (MC) runs. 

\begin{figure}[H]
\centering

\includegraphics[width=1\textwidth]{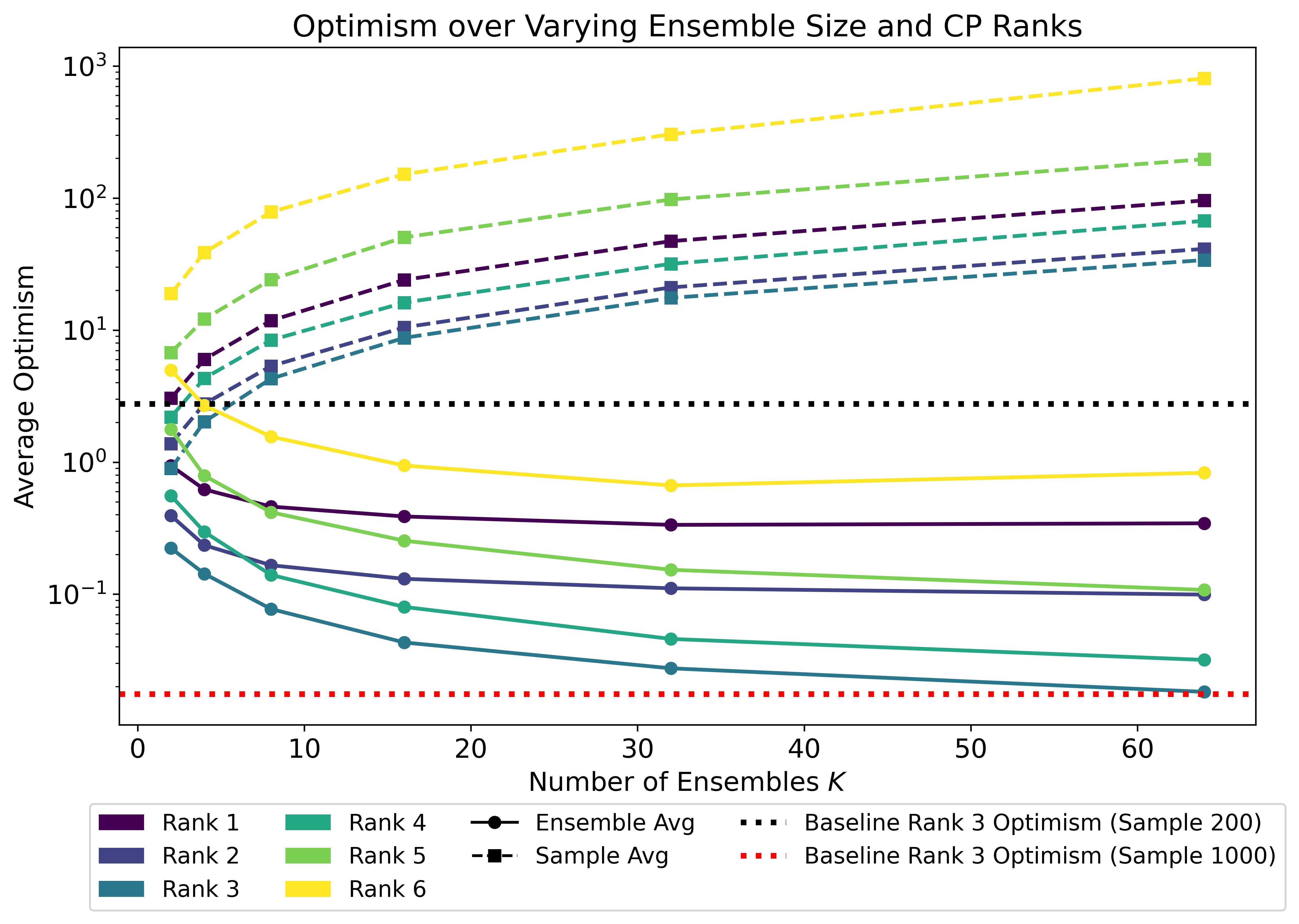}
\caption{\label{fig:CP_ensemble_vary_size}Average optimism of CP ensemble regression over 10000 MC runs with varying number of ensembles (x-axis) and ensemble CP ranks. The noise level is  $5\%$ of signal standard deviation and training sample size is $n_{\text{train}} = 1000$. Each ensemble member is trained on a random subsample of size $n_k=200$. The true model CP rank is $R = 3$. Two horizontal dash lines represent the baseline optimism for a single (non-ensembled) rank-$3$ model trained on $200$ (black) and $1000$ (red) samples, respectively.}
\end{figure}

Figure~\ref{fig:CP_ensemble_vary_size} presents the expected optimism for ensemble-averaged estimator $\bar{\T{B}}$ and  weighted average of the individual learners' expected optimisms (i.e., $\sum_{k=1}^{K} \frac{n_k}{n} \mathrm{OptR}_{\T{X}}^{\mathrm{(k)}}$) with various target CP ranks $R_t$ and numbers of weak learners $K$. As shown in the figure, the ensemble's optimism, $\mathrm{OptR}_{\T{X}}^{\mathrm{(ens)}}$, is consistently lower than the weighted-average optimism across all tested configurations. This result empirically supports the theoretical upper bound established in Theorem~\ref{thm:thm_CP_ensemble_bound}.

\begin{figure}[H]
\centering

\includegraphics[width=1\textwidth]{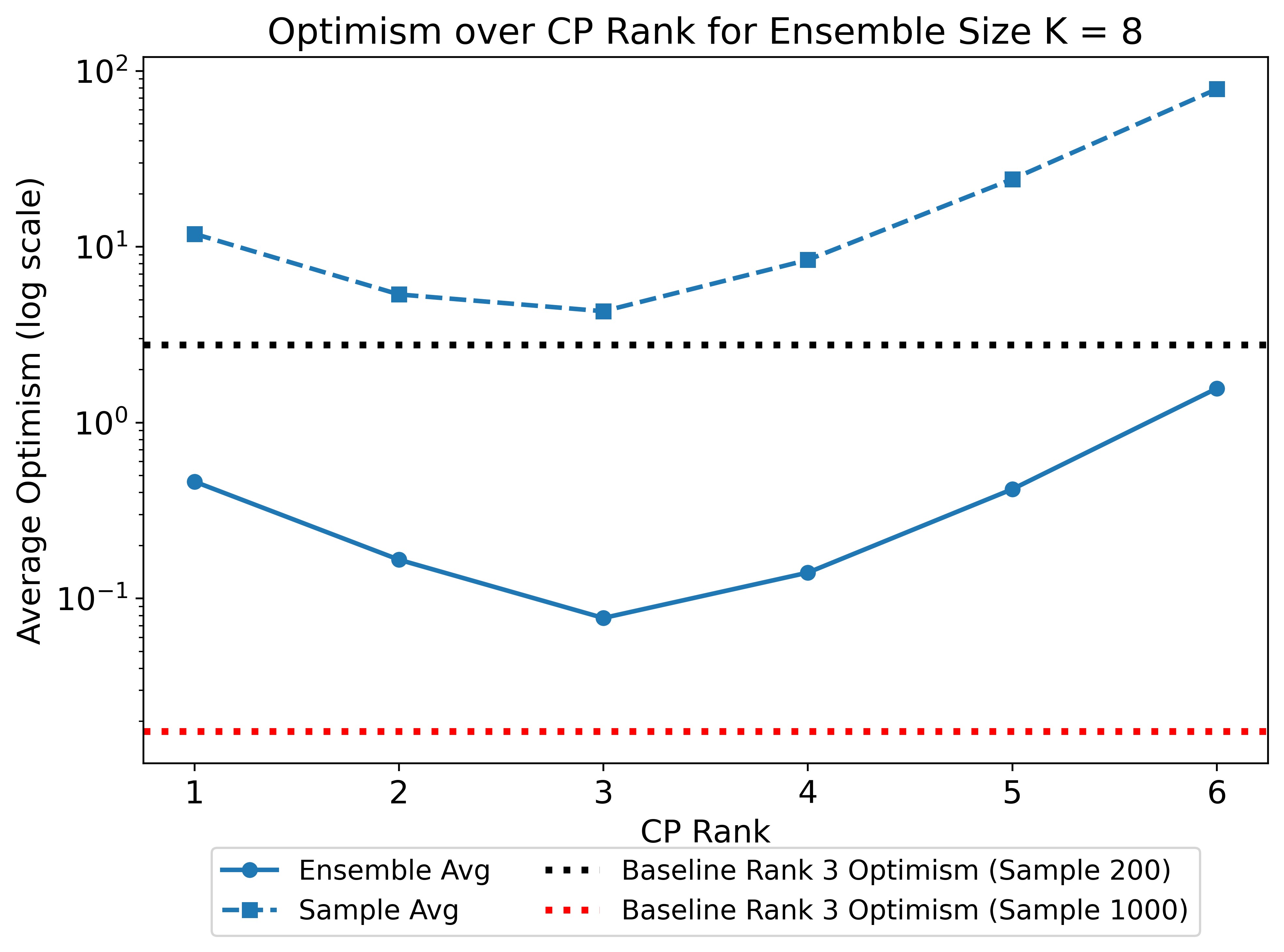}
\caption{\label{fig:CP_ensemble_vary_rank}Average optimism of CP ensemble regression over $10000$ MC runs with $K = 8$ ensemble members and varying ensemble CP ranks (x-axis). The noise level is  $5\%$ of signal standard deviation and training sample size is $n_{\text{train}} = 1000$. Each ensemble member is trained on a random subsample of size $n_k=200$. The true model CP rank is $R = 3$. Two horizontal dash lines represent the baseline optimism for a single (non-ensembled) rank-$3$ model trained on $200$ (black) and $1000$ (red) samples, respectively.}
\end{figure}

Furthermore, Figure~\ref{fig:CP_ensemble_vary_rank} illustrates the behavior of these two optimism quantities as a function of the target CP rank $R_t$, with the number of ensemble members fixed at $K=8$. Notably, both the ensemble optimism and the weighted-average optimism are minimized when the target rank is $R_t = 3$, which corresponds to the true rank of the underlying tensor coefficient $\T{B}$. This observation is consistent with Proposition~\ref{prop:CP_proposition}, which guarantees that the expected optimism for each individual learner is minimized at the true CP rank.

\subsection{Different Tensor Formulation for Compressed MLP}
As discussed in the main paper (Section~\ref{sec:neural_net}), the tensorization of the MLP weight matrix is artificial. To demonstrate the robustness of our findings to the specific reformulation chosen, we present results for an alternative tensorization.

Figure~\ref{fig:CP_MLP_opt_aic_bic_2} illustrates the behavior of the selection criteria in the same experimental settings, but with a different tensor formulation for the input layer weight $\M{W}$. Instead of the $5 \times 6 \times 12 \times 12$ reformulation used previously, we reshape it to a $\T{W} \in \mathbb{R}^{6 \times 5 \times 8 \times 18}$ tensor and apply a CP decomposition. The results are consistent with our main findings: all selection measures again favor a simple rank-1 model. Furthermore, the test MSE for this compressed model remains lower than that of the full, uncompressed MLP, reinforcing our observation on the benefits of low-rank structures in this context.

\begin{figure}[H]
    \begin{minipage}[c]{0.7\textwidth}
        \includegraphics[width=\linewidth]{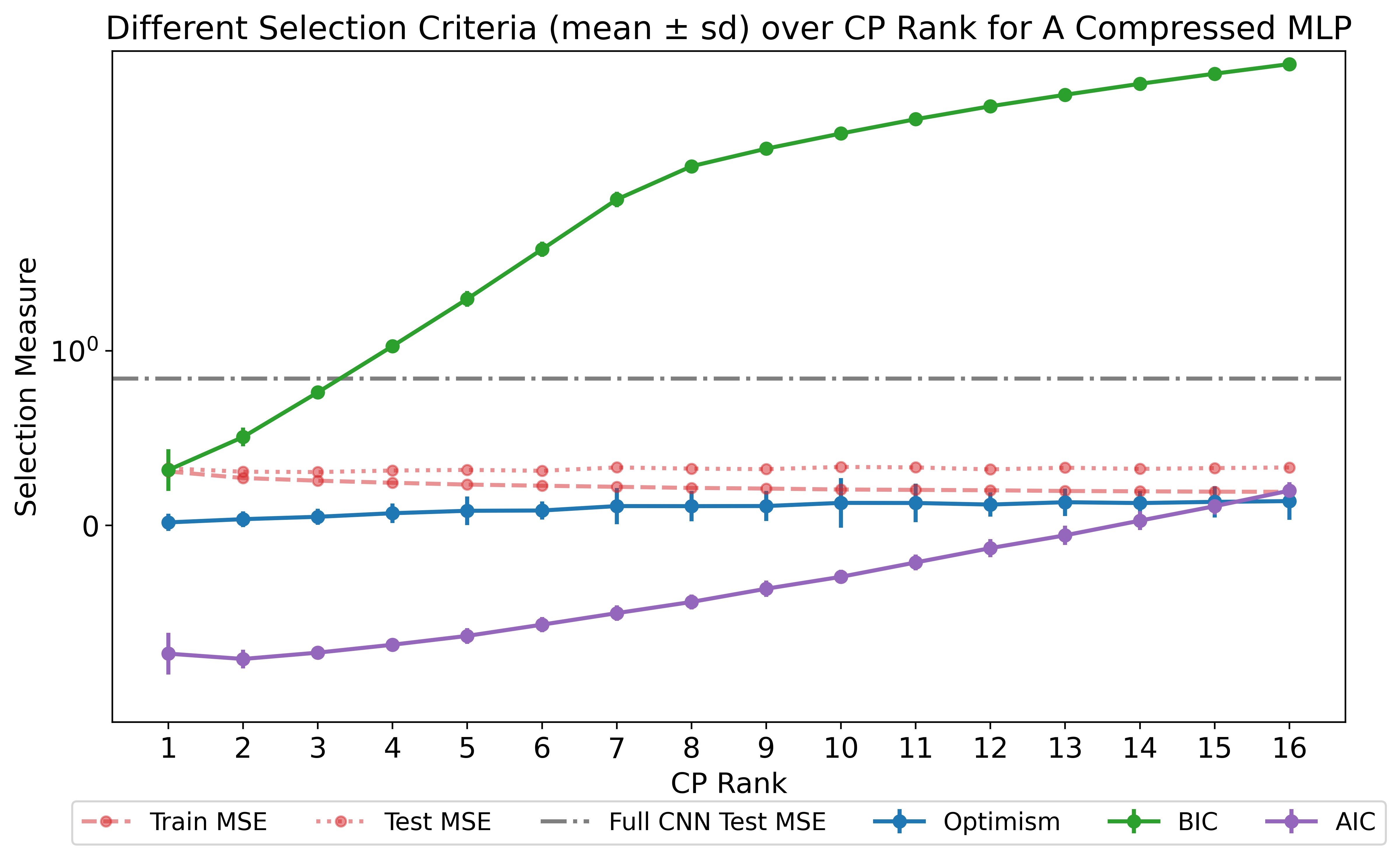}
    \end{minipage}%
     \hspace{0.02\textwidth} 
    \begin{minipage}[c]{0.25\textwidth}
        \centering
        \captionof{table}{Selected Ranks}
        \label{tab:selected_ranks_MLP}
        
        \resizebox{\linewidth}{!}{%
            \begin{tabular}{lcc}
                \toprule
                \textbf{Criterion} & \textbf{Rank} & \textbf{Test MSE} \\
                \midrule
                Optimism & 1 & 0.3 \\
                AIC       & 2  & 0.3 \\
                BIC       & 1 & 0.3 \\
                \bottomrule
            \end{tabular}%
        } 
    \end{minipage}

    \caption{\label{fig:CP_MLP_opt_aic_bic_2}Different Selection Criteria for a two-layer MLP with a CP-decomposed input layer, fitted on the Infrared Thermography Temperature dataset \citep{wang2021infrared}. Compared to Figure~\ref{fig:CP_MLP_opt_aic_bic} in the main manuscript, here the input layer is reformulated as $\T{W} \in \Real^{6 \times 5 \times 8 \times 18}$ and compressed using CP decompositions with varying ranks (x-axis). The model is trained on an 80\% training and 20\% testing split using Adam optimizer (fixed learning rate 0.001) with MSE loss and 200 epochs. Optimism is calculated via the hold-out algorithm ($100$ MC replicates) in \citet{luo2025optimism}. Error bars show one standard deviation. The table on the right summarizes the optimal rank selected by each criterion (the minimum value on its curve) and the corresponding Test MSE at that rank.}
\end{figure}

\end{document}